%% file: final.tex
\documentclass{article}
\input{arxiv/style}

\title{Non-Stationary Bandit Learning via Predictive Sampling}
\author[1]{Yueyang Liu}
\author[2]{Xu Kuang}
\author[3]{Benjamin Van Roy}
\affil[1]{Jones Graduate School of Business, Rice University}
\affil[2]{Stanford Graduate School of Business}
\affil[3]{Department of Management Science and Engineering, Department of Electrical Engineering, Stanford University}
\date{}
\begin{document}

\maketitle

\begin{abstract}
Thompson sampling has proven effective across a wide range of stationary bandit environments.  However, as we demonstrate in this paper, it can perform poorly when applied to non-stationary environments.  We show that such failures are attributed to the fact that, when exploring, the algorithm does not differentiate actions based on how quickly the information acquired loses its usefulness due to non-stationarity. Building upon this insight, we propose predictive sampling, an algorithm that deprioritizes acquiring information that quickly loses usefulness. Theoretical guarantee on the performance of predictive sampling is established through a Bayesian regret bound. We provide versions of predictive sampling for which computations tractably scale to complex bandit environments of practical interest.  Through numerical simulations, we demonstrate that predictive sampling outperforms Thompson sampling in all non-stationary environments examined.
\end{abstract}

\input{arxiv/introduction}

\input{arxiv/related_work}

\input{arxiv/motivation}

\input{arxiv/definition}

\input{arxiv/algo_PS}

\input{arxiv/prop_PS}

\input{arxiv/regret_analysis}

\input{arxiv/experiment}

\section{Concluding Remarks}
\label{section:conclusion}
This paper demonstrates that TS and its variants
\remove{that were proposed in the
literature}often do not perform well in non-stationary bandits, because they fail to intelligently account
for the durability of information when selecting actions. To address this, we propose PS, an algorithm that can be viewed as a version of TS that takes the sequence of future rewards as the learning target.
We develop efficient procedures to execute PS in AR(1) bandits\remove{ and a practical approximation of it in AR(1) logistic bandits}.
We demonstrate the efficacy of PS through coin-tossing examples, regret bounds, and numerical experiments.

{While PS intelligently accounts for the durability of information, it relies on prior knowledge of the underlying environment and its ability to update based on observed rewards. In some cases, this is relatively straightforward; in others, it requires careful tuning of hyperparameters, and in yet others, the process becomes more complex.  Certain non-stationary bandit learning algorithms, such as discounted UCB and sliding-window UCB, rely on hyperparameter estimation. A promising direction for future research is to refine the implementation of PS or its approximations, enabling the algorithm to consistently rely on simpler procedures, such as tuning a small number of hyperparameters.}

At a high level, our paper illustrates how we can improve an existing algorithm (TS) to arrive at a new algorithm (PS) that is better suited for non-stationary bandits, by changing its learning target.
As discussed in Sections~\ref{sec:related_work} and~\ref{section:algorithms}, a number
of other existing algorithms also do not account for the durability of information when selecting actions.
{An interesting direction for future work would be to explore whether other algorithms can be augmented with an awareness of information durability to help them better navigate non-stationary environments. For instance, one could conceivably design a UCB variant that assigns each arm its own confidence interval and padding function, and explicitly modulates both based on the arm’s information durability. As another example, we could potentially enhance information-directed sampling (IDS) to better handle non-stationarity, where the new agent minimizes \eqref{eq:information_ratio} at each timestep.}

\addnew{While the notion of information durability discussed in the Introduction helps to explain the superiority of PS over TS in many bandit instances, PS is not guaranteed to always react to information durability relative to TS. For instance, suppose an arm $a$ has a mean reward $\theta_{t,a}$ that is re-sampled from $\mathcal{N}(0,1)$ every $\tau$ steps, but is fixed in between, and the reward observations of this arm $R_{t,a}$ are noiseless and always equal to $\theta_{t,a}$. In most timesteps, PS and TS would sample arm $a$ with the same probability at each step independently of the value of $\tau$, despite it capturing the information durability of the arm.\footnote{We thank the referee who suggested this counterexample during the review process of this manuscript.} We suspect that PS's insensitivity to information durability in this example might be due to the particular assumption of noiseless rewards, but regardless, a more thorough understanding of what constitutes a rigorous definition of ``information durability'' beyond an explanatory device, and whether PS takes advantage of it in some form of near-optimal manner, could be an interesting direction for future research.}

While the sequence of future rewards has proven to be an effective learning target for our purpose,
it
would be interesting to investigate whether there are other potentially more powerful learning targets.
For example, by suitably defining the ``optimal action" at each timestep, we may alternatively take the sequence of future optimal actions as the learning target. It remains an open question how these alternative learning targets would perform in non-stationary environments, how to design algorithms with efficient implementations with respect to these learning targets, and how the current framework needs to adapt for their analyses. \add{Finally, while the current paper primarily focuses on understanding how PS helps improve performance relative to TS in non-stationary environments, it would be interesting to more thoroughly understand how PS performs relative to other bandit policies.}

\bibliographystyle{apalike}
\bibliography{arxiv/references}

\appendix
\input{arxiv/appendix}

\end{document}

%% file: arxiv/style.tex
\usepackage[in]{fullpage}
\usepackage{graphicx,psfrag,amsmath,amsthm,amsfonts,amssymb,verbatim,algpseudocode,comment,mathtools}
\usepackage{bbm}
\usepackage[linesnumbered,ruled,vlined]{algorithm2e}
\usepackage{natbib}
\usepackage{xcolor}
\usepackage{enumerate}
\usepackage{textcomp}
\usepackage{caption}
\usepackage{subcaption}
\usepackage{mathrsfs}
\usepackage[citecolor=blue]{hyperref}
\usepackage{thmtools,thm-restate}
\usepackage[makeroom]{cancel}
\usepackage{epigraph}
\usepackage{etoolbox}
\usepackage[mathscr]{euscript}
\usepackage{setspace}
\usepackage{authblk}
\usepackage{soul}

\graphicspath{{./}{../}}

\newtheorem{lemma}{Lemma}

\newtheorem{corollary}{Corollary}

\theoremstyle{definition}
\newtheorem{definition}{Definition}
\newtheorem{example}{Example}

\def\actions{\mathcal{A}}

\def\histories{\mathcal{H}}

\def\KL{\mathbf{d}_{\mathrm{KL}}}

\def\E{\mathbb{E}}
\def\V{\mathbb{V}}

\def\H{\mathbb{H}}
\def\I{\mathbb{I}}
\def\Pr{\mathbb{P}}
\def\R{\mathbb{R}}
\def\1{\mathbf{1}}
\def\0{\mathbf{0}}

\DeclareMathOperator*{\argmax}{arg\,max}

\newcommand\numberthis{\addtocounter{equation}{1}\tag{\theequation}}

\newcommand{\remove}[1]{}
\newcommand{\add}[1]{{#1}}
\newcommand{\addnew}[1]{{#1}}

\hypersetup{
    colorlinks=true,
    linkcolor=blue,
    filecolor=magenta,
    urlcolor=purple,
}

\newcount\Comments
\newcommand{\kibitz}[2]{\ifnum\Comments=1{\textcolor{#1}{\textsf{\footnotesize #2}}}\fi}
\definecolor{darkred}{rgb}{0.7,0,0}
\definecolor{darkgreen}{rgb}{0.0,0.5,0.0}
\definecolor{darkblue}{rgb}{0.0,0.0,0.5}
\definecolor{teal}{rgb}{0.0,0.5,0.5}
\definecolor{dogwoodrose}{rgb}{0.84, 0.09, 0.41}
\definecolor{electriccrimson}{rgb}{1.0, 0.0, 0.25}
\definecolor{folly}{rgb}{1.0, 0.0, 0.31}
\definecolor{frenchrose}{rgb}{0.96, 0.29, 0.54}
\definecolor{cadetgrey}{rgb}{0.57, 0.64, 0.69}

\allowdisplaybreaks[4]

%% file: arxiv/introduction.tex
\section{Introduction}
\label{sec:introduction}
{Thompson sampling (TS) \citep{Thompson1933} is a bandit learning algorithm that, at each timestep, samples statistically plausible mean rewards and selects an action with the largest sampled mean.}  {The algorithm's popularity derives from its conceptual elegance and practical performance.  In particular,} for a range of stationary bandit environments, or {\it stationary bandits}, for short, the efficacy of TS has been established through theoretical and empirical analyses \citep{pmlr-v23-agrawal12,  NIPS2011_e53a0a29, RussoMOR2014}. 
The algorithm has enjoyed a wide range of applications, including 
revenue management \citep{ferreira2018online}, website optimization \citep{hill2017efficient}, Monte Carlo tree search \citep{bai2013bayesian}, A/B testing \citep{graepel2010web}, advertising \citep{agarwal2014laser, graepel2010web, agarwal2013computational,  schwartz2017customer}, 
and recommender systems \citep{kawale2015efficient}.

Many of these applications, however, exhibit non-stationarity.\footnote{\addnew{Our usage of the term {\it non-stationary} is consistent with the bandit learning literature, which differs from the literature on stochastic processes.  In particular a bandit is considered stationary if the vector process $\{R_t\}_{t=1}^{+\infty}$, where each component $R_{t,a}$ is the reward earned if action $a$ is selected at time $t$, is exchangeable.}}  For example, recommender systems often serve user populations where trends and preferences fluctuate over time. Similarly, learning algorithms designed for dynamic pricing can be impacted by seasonal changes in supply and demand patterns. Unfortunately, as we demonstrate in this paper, TS is not suited for these non-stationary environments. We show that applying TS can lead to near worst-case performance in some examples of non-stationary bandits; this is further corroborated by the sub-optimal performance of TS in our numerical experiments. 

\add{What causes TS to perform poorly in the face of non-stationarity? One way to explain this is by using a \addnew{notion of} ``information durability'': the degree to which a reward observation at the current time step influences our belief about future rewards. Let us denote by $R_{t, a}$ the reward observed from pulling arm $a$ at time step $t$. Then, in a stationary bandit environment, because the arms' reward distributions do not change over time, observing the reward $R_{t, a}$ at time $t$ would influence our belief and predictions of all future rewards from this arm, $R_{t', a}$ for $t' \geq  t + 1$. In this sense, the information gleaned from $R_{t, a}$ is infinitely durable.  In contrast, consider a non-stationary environment where, say, $R_{t,a}$ is drawn from $\mathcal{N}(\theta_{t,a}, 1)$, and $\theta_{t,a}$ is drawn independently from $\mathcal{N}(0,1)$ every ten timesteps, remaining fixed between draws. In this case, information revealed by $R_{t, a}$ about the future rewards from arm $a$ is \addnew{less durable} as a result of the ``resets'' in $\theta_{t,a}$.  
}

\add{As a corollary, different actions can now have different levels of information durability in a non-stationary environment, and this lies at the root of TS's inefficiency. Bandit algorithms like TS typically sacrifice short-term rewards in exchange for learning about the reward distributions, with the tacit assumption that this knowledge will then help the decision-maker harness greater rewards in the future. But if information durability varies from arm to arm, then it would seem to follow that a well-performing algorithm should also take variations in information durability into account when conducting exploration. Unfortunately, TS, like most bandit algorithms designed for stationary environments, does not incorporate information durability in its decision-making by design.}   

\add{Motivated by the above observations, we propose predictive sampling (PS). At a high level, PS inherits much of the conceptual elegance of
the original TS algorithm, but is redesigned so that it naturally shapes exploration using information
durability.  At each time $t$, PS samples a statistically plausible sequence  of \emph{future} reward vectors, $\hat{R}_{t+1:\infty}$, from the joint posterior distribution of future rewards $R_{t+1:\infty}$ (each vector $R_k$ assigns a reward $R_{k,a}$ to each action $a$). It then computes the expected immediate rewards for the current time step,\footnote{We use the equality sign in $ R_{t+1:\infty} = \hat{R}_{t+1:\infty}$ for simplicity in this informal exposition. The actual operation here is a change of measure, which we will formally define in the later section. }
\begin{equation}
    \bar{R}_{t} = \mathbb{E}[R_{t} \mid H_{t-1}, R_{t+1:\infty} = \hat{R}_{t+1:\infty}],
    \label{eq:barR_intro}
\end{equation}
where $H_{t-1}$ stands for the history. Importantly, the algorithm ``pretends'' that the future rewards are given by the sampled vector $\hat{R}_{t+1:\infty}$. Finally, an action $A_t \in \arg\max_{a \in \mathcal{A}} \bar{R}_{t,a}$ is chosen to maximize over these conditional expectations.

Here is some high-level intuition for why PS, and in particular \eqref{eq:barR_intro}, adapts to information durability: Suppose rewards are independent across arms. If arm $a$'s rewards are relatively stationary, then one would expect that observing its future rewards ${R}_{t+1:\infty, a}$ is highly informative of the current-step reward $R_{t, a}$. Therefore, the mean reward of $R_{t, a}$ conditional on the sampled future rewards, $\bar{R}_{t, a}$, would be more sensitive to the input $\hat{R}_{t+1:\infty}$ and thus have a larger variance. This generally encourages sampling arm $a$. Conversely, if arm $a$ has highly non-stationary rewards, then seeing future rewards ${R}_{t+1:\infty, a}$ tells us little about the current reward, $R_{t, a}$. As a result, $\bar{R}_{t, a}$ would be less sensitive to the input $\hat{R}_{t+1:\infty}$ and thus have a smaller variance. This generally discourages sampling arm $a$.
In summary, via the calculations in \eqref{eq:barR_intro}, PS naturally endows arms with more durable information with more ``uncertainty,'' thus encouraging their exploration relative to less information-durable arms.\footnote{\addnew{There can be special cases of bandit environments where PS is not responsive to information durability; we discuss one such case in Section \ref{section:conclusion}.}}  We will provide more concrete examples in Section \ref{section:properties} to illustrate the intuition behind PS, once the proper notation is introduced.
}

Beyond proposing PS, a second major contribution of this paper lies in establishing a general information-theoretic framework that facilitates performance analysis of any agent in a non-stationary bandit.  Non-stationary introduces significant challenges to applying existing information-theoretic analyses due to changes in both the quantity and the quality of information. For example, while the total amount of information in the environment is typically bounded in a stationary bandit {(which we will formalize later in Section 7.3)}, in a non-stationary one, additional randomness can be injected in each time step, and thus the information can grow unbounded as the horizon $T$ increases. Furthermore, as discussed earlier, not all such information is relevant when it comes to predicting future rewards.  In particular, the analysis needs to account for information durability. To overcome these difficulties, we introduce the concept of \emph{predictive information}, which captures the type of information that is useful in predicting future rewards. Using it, along with a new notion of information ratio, we are able to extend the information-theoretic regret analysis originally developed for stationary bandits \citep{RussoMOR2014} to non-stationary bandit learning and obtain non-trivial regret upper bounds. 

Using this framework, we establish a general regret bound for any agent that is expressed in terms of the cumulative predictive information $\Delta$.
The bound grows linearly in $\sqrt{\Gamma \Delta}$, where $\Gamma$ denotes the sum of the information ratios. 
Applying this analysis to PS, we establish a regret bound that grows linearly in $\sqrt{T \Delta}$. In particular, when applied to a stationary environment, this bound reproduces some of the best known bounds for the regret of TS (c.f.,~\citep{neu2022lifting}), suggesting that this new  analysis is likely competitive against its stationary counterparts in terms of tightness. 

We further leverage these general regret bounds to derive easy-to-interpret  guarantees for specific classes of non-stationary bandit problems. We do so by developing new techniques that allow us to upper-bound the cumulative predictive information $\Delta$. For instance, we analyze a class of modulated Bernoulli bandits that
generalizes the constant rate per-arm abrupt switching model in \citep{pmlr-v31-mellor13a}, where the reward distributions evolve according to a Markov chain. In this case, we are able to derive explicit bounds on $\Delta$ and further bounds on the regret that exhibits a graceful dependence on the transition kernel of the modulating Markov chain. 
We also establish a lower bound on the regret incurred by any agent in these bandits. These regret bounds demonstrate the effectiveness of predictive sampling
across a range of such non-stationary Bernoulli bandits.

Finally, we develop computationally tractable implementations of PS. For a class of non-stationary Gaussian bandits, we demonstrate how to implement PS exactly in a computationally efficient manner. For a more complex class of non-stationary logistic bandits, where PS is too computationally expensive to be performed exactly, we develop an efficient procedure to approximate PS using Laplace approximation. Using these implementations, we conduct extensive numerical experiments across a range of such bandits with varying information durability. Our computational results suggest that PS, as well as its approximation, consistently outperform not only TS, but also other algorithms proposed for non-stationary bandit learning. 

In summary, the main contributions of this paper include: 
\begin{enumerate}[(1)]
\item We  elucidate how and why TS and variants of TS proposed in previous literature do not account for information durability when selecting actions. 
\item  We propose PS for non-stationary bandit learning, and layout qualitative insights on how and why PS can significantly outperform TS in non-stationary environments. We further support the claim with theoretical results and numerical evidence.  
\item We develop computationally tractable implementations of PS for a class of Gaussian bandits which we refer to as AR(1) bandits  and an approximation of PS for a class of logistic bandits. 
\end{enumerate}

\textbf{Structure of the paper} \quad The paper is organized as follows. Section~\ref{sec:motivating_discussion} presents an example illustrating the limitations of TS in certain non-stationary bandits. Section~\ref{sec:bandits} introduces a general formulation of bandits. Sections~\ref{sec:ps} and~\ref{section:properties} formally introduce PS and discuss its qualitative properties. Section~\ref{sec:regret} provides the regret analyses. Section~\ref{sec:experiments} presents tractable examples and approximations of PS, along with numerical experiments. Section~\ref{section:conclusion} summarizes the paper. The appendix provides the probabilistic framework, information-theoretic notations and concepts, and technical proofs.

%% file: arxiv/related_work.tex
\section{Related Work}
\label{sec:related_work}
\paragraph{Non-Stationary Bandit Learning}

A number of interesting algorithms for non-stationary bandit learning have been proposed based on modifying TS \citep{9194367, gupta2011thompson, pmlr-v31-mellor13a, raj2017taming, trovo2020sliding, viappiani2013thompson} or other stationary bandit algorithms \citep{bacchiocchi2022autoregressive, besbes2019optimal, besson2019generalized, cheung2019learning, garivier2008upper, hartland2006multi, kocsis2006discounted, mintz2020nonstationary, pmlr-v31-mellor13a, zhao2020simple}. 
{ The main distinction between this literature and our work is they tend to focus on obtaining better estimates of the current rewards in the face of uncertainty, before deploying these improved reward estimates within an otherwise conventional bandit algorithm. Some of the prominent methods to improve reward estimates include maintaining a sliding window
\citep{cheung2019learning, garivier2008upper, russac2020algorithms, trovo2020sliding}, 
discounting past rewards by recency
\citep{bogunovic2016time, garivier2008upper, kocsis2006discounted, russac2020algorithms},  
and restarting a base algorithm periodically or when a change point is detected 
\citep{abbasi2022new, pmlr-v99-auer19a, besbes2019optimal, besson2019generalized, cheung2019learning, chen2024non, ghatak2021kolmogorov, gupta2011thompson, hartland2006multi, luo2018efficient, mellor2013thompson, raj2017taming, wei2018abruptly, viappiani2013thompson, zhao2020simple}. 
{These approaches focus on improving reward estimation rather than modifying how an algorithm explores based on these estimates. In contrast, our approach explicitly considers how to adjust exploration to account for information durability in a non-stationary environment. }
\remove{However, because these approaches do not modify the underlying bandit algorithm, they still do not incorporate information durability into their exploration and therefore tend to suffer from the same limitations as their stationary counterparts. In contrast, our approach places a heavy emphasis on reasoning about how to change the way a bandit algorithm ought to explore in a non-stationary environment.}} 

Another line of literature considered bandit problems where reward processes are modeled by more general (non i.i.d.) stochastic processes \citep{kaspi1998multi, pmlr-v162-kim22j, levine2017rotting, mandelbaum1986discrete, mandelbaum1987continuous, varaiya1985extensions}. A major distinction is that these works generally assume that the reward from an action evolves only when the action is selected by the agent, and as a result, the reward from a particular action only depends on how many times the said action has been used up until this point. In contrast, the rewards in our model, as well as those in the above-mentioned bandit literature, generally evolve over time {exogenously}, regardless of whether an action has been selected or not. 

\paragraph{Information-Theoretic Analysis of Stationary Bandit Learning}
Our work builds on the body of literature on information-theoretic regret analyses 
for stationary bandits 
\citep{bubeck2015bandit, dong2018information, hao2022contextual, lattimore2019information, lu2021reinforcement, neu2022lifting, RussoMOR2014, JMLR:v17:14-087, russo2018learning}. 
This literature introduces the notion of an information ratio and bounds the regret of an agent in terms of its information ratio. Our work contributes to this literature by extending the information-theoretic framework to non-stationary bandit learning. As described in the Introduction, we overcome several non-trivially difficulties encountered in the process by leveraging a new notion of information ratio, originally proposed by \cite{RussoMOR2014}, that is better suited for non-stationary bandits, as well as by using a novel concept of predictive information that allows us to articulate the predictive value of information for future rewards. 

{After we posted an initial version of this paper, \citep{min2023information} was disseminated which also examines the performance of TS in a non-stationary environment through an information-theoretic lens. While \cite{min2023information} primarily focus on bounding the performance degradation of TS in a non-stationary environment, we are more interested in understanding how to design a new algorithm that would mitigate some of the shortcomings of TS in non-stationary environments. 
} 

\paragraph{Prediction Driven Decision Making} There are algorithms that explicitly or implicitly use predictions of future system inputs in decision making. One paradigm, often known as model predictive control,  involves a controller who repeatedly solves a planning problem into the future by substituting future inputs using predictions \citep{mesbah2018stochastic}; the problems studied in \citep{freund2019good, spencer2014queuing, wen2022predictions2decisions, xu2016using} fall under this general category. In other cases, thinking in terms of a hypothetical future input trajectory has proven valuable in obtaining improved performance characterization for Markov decision processes, a technique known as information relaxation \citep{brown2010information}. Notably,  \cite{min2019thompson} proposes an interesting family of information-relaxation-inspired sampling algorithms for stationary bandits. When the actions are independent, PS can be shown to be equivalent to an extreme point of a sequence of such algorithms. However, compared to our work, existing applications that make use of predictions are either not concerned with learning with bandit-type feedback, or do not address the issue of learning in a non-stationary environment.

%% file: arxiv/motivation.tex
\section{Motivation}
\label{sec:motivating_discussion}

This section demonstrates that TS, as typically applied, does not account for the durability of information and this can severely degrade its performance. 
Moreover, we demonstrate that in some non-stationary bandits, TS performs arbitrarily close to the worst possible agent. 
We also show that 
the same holds true for variants of TS that have been 
proposed in the literature. 

\subsection{Tossing Random Coins}
\label{sec:random_coins}
Suppose you engage in a sequence of decisions where, at each timestep, you choose between one of two biased coins to toss and subsequently receive a payoff of $\$1$ if the coin lands heads and $\$0$ otherwise.  This environment is illustrated in Figure~\ref{fig:two_coins_replace_prob}. The first coin is known to land heads with probability $p_1 = 0.99$. The second coin is drawn from a bag that holds an infinite number of extremely biased coins, half of which always land heads and the other half always land tails. At each timestep, 
there is a high probability, say $0.99$, that 
the second coin is replaced by a new one drawn from the bag. 
The bias of the second coin, which we denote by $p_{t,2}$, takes $0$ and $1$ with equal probabiliy, and changes over time.  

\begin{figure}[htb]
\begin{center}
\includegraphics[scale=0.3]{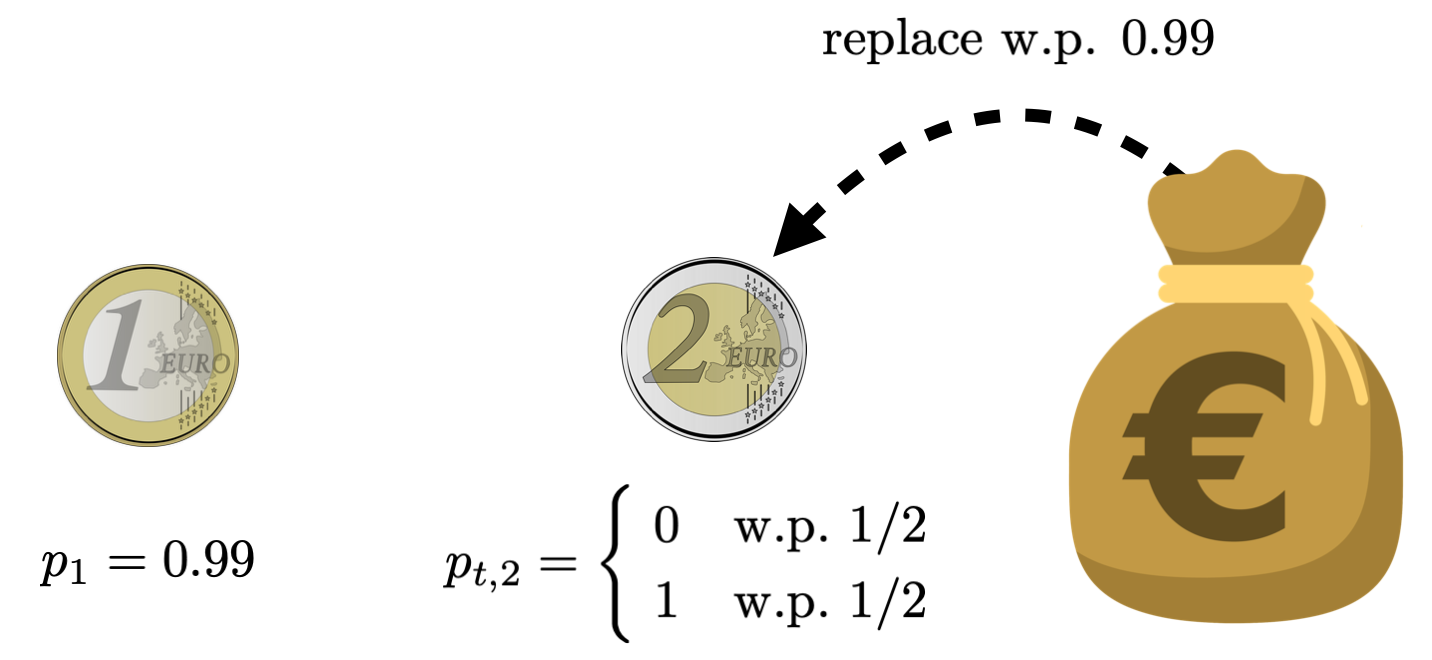}
\end{center}
\caption{
Choosing between two coins: one with bias $p_1 = 0.99$ and the other with bias $p_{t,2}$ of either $1$ or $0$. The second coin is replaced each time with probability $0.99$. 
}
\label{fig:two_coins_replace_prob}
\end{figure}

{
In this game, the bias of the first coin, $p_1 = 0.99$, is known. The replacement process for the second coin—specifically, the replacement rate of $0.99$ and the biases of the coins in the bag—is also known. However, you cannot directly observe when a coin is replaced or the bias of the newly introduced coin.  In this case, you can attempt to infer the bias  
$\{p_{t,2}\}_{t = 1}^{+\infty}$ of the second coin based on observed rewards.  
}

In this environment, selecting the first coin offers payoff of $99$\textcent\ per timestep. Selecting the second coin offers \$1 if the coin bias is 1 and \$0 otherwise, each with equal probability. 
After a single toss of the second coin, the bias of the second coin is revealed. 
However, there is a high probability that the second coin is replaced at the next timestep, and the learned bias becomes irrelevant. 
Therefore, an optimal agent would be one that only ever tosses the first coin, accumulating payoffs at an expected rate of $99$\textcent\ per timestep, instead of investing to learn the bias of the second coin.

Although TS offers an approach to making such sequential decisions, it turns out that it invests in learning the bias of the second coin, the durability of which is poor,
and is thus suboptimal in this environment. 
Observe that this environment is identified by the coin biases $p_1$ and $p_{t,2}$. 
At each timestep, TS samples from the posterior distribution of the coin biases and selects an action that would maximize 
the sample. 
{Formally, TS samples $\hat{p}_{t,1}$ from the posterior of $p_1$ and $\hat{p}_{t,2}$ from the posterior of $p_{t,2}$, and selects the first coin if $\hat{p}_{t,1} > \hat{p}_{t,2}$ and the second coin otherwise.} 
Because the first bias is known, at each timestep $t$, TS takes its sample to be $\hat{p}_{t,1} = p_1 = 0.99$. The second coin, on the other hand, is replaced with \remove{high }probability {$0.99$} at each time, and when it is replaced, the bias becomes $0$ or $1$ with equal probability. So $\hat{p}_{t,2} = 1$ with a probability that is at least {$0.99 \times 1/2 =$} $0.495$. 
Maximizing between $\hat{p}_{t,1}$ and $\hat{p}_{t,2}$, TS samples the second coin with a probability that is at least $0.495$.
A simple derivation {(Appendix~\ref{appendix:motivating_example_two})} shows that TS accumulates payoffs at an expected rate of at most $75$\textcent\ per timestep, which clearly falls far short of the $99$\textcent\ rate that would be earned by repeatedly tossing the first coin.

\subsection{Existing Variants of TS Do Not Fix the Issue}
A number of variants of TS have been proposed for the purpose of non-stationary bandit learning, including TS with change-detection \citep{9194367}, dynamic TS \citep{gupta2011thompson}, change-point TS \citep{ pmlr-v31-mellor13a}, discounted TS \citep{raj2017taming},  sliding-window TS \citep{trovo2020sliding} and reset-aware TS \citep{viappiani2013thompson}. 
Instead of maintaining an exact posterior distribution, which can be highly computationally demanding in a non-stationary environment, these algorithms strive to employ various heuristics to approximate the posterior distribution, and also to update this approximation in a tractable manner. 
Each algorithm, similar to what TS would do, samples from the approximate posterior distribution, and selects an action that optimizes the corresponding expected payoff. 

In the coin-tossing environment of Figure~\ref{fig:two_coins_replace_prob}, each of these algorithms maintains an approximate posterior distribution of the coin biases $p_1$ and $p_{t,2}$, 
samples from this distribution, and selects an action that maximizes the sample. 
Similar to TS, each agent would sample $\hat{p}_{t,1} = 0.99$. 
Recall that the second coin is replaced with probability $0.99$ at each timestep, and when it is replaced, the bias becomes $0$ or $1$ with equal probability. Consequently, if an agent intelligently uses this coin-replacement information in approximating the posterior distribution of the coin biases, the agent would sample $\hat{p}_{t,2} = 1$ with a probability that is at least $0.495$. Maximizing between $\hat{p}_{t,1}$ and $\hat{p}_{t,2}$, the agent would select the second coin with a probability that is at least $0.495$ and deviate from the optimal agent that only ever selects the first coin. 

Although not all agents would intelligently use the coin-replacement information in approximating the posterior distribution, readers can verify that each of the aforementioned agents would select the second coin with a positive probability. 
Therefore, similar to TS, these variants of TS invest in learning the bias of the second coin, and these variants are thus suboptimal in this environment. Since the variants of TS behave and perform similarly to TS, we will  focus our attention on TS as a benchmark in the rest of the paper. 

\subsection{TS Can Perform Very Poorly}
\label{sec:K_coins}
Now that we have demonstrated that TS deviates from the optimal strategy in some non-stationary environments, we next characterize how bad TS can perform. 
The main message here is that TS can perform almost as badly as a worst-performing agent in some non-stationary environments. 

Consider a variant of the environment of Figure~\ref{fig:two_coins_replace_prob}, where
the decision at each time is to choose which coin among $K$ coins to toss, where $K$ is greater than $2$. 
{The coins from the third to the $K$-th} are independent copies of the second coin. That is, each of these coins is drawn independently from the bag and is replaced independently at each time with probability $0.99$. In addition, with this variant, almost all coins in the bag have bias $0$: suppose that $99\%$ of the coins in the bag have bias $0$, and the rest have bias $1$. 
Figure~\ref{fig:many_coins_n} illustrates this.

\begin{figure}[htb]
\begin{center}
\includegraphics[scale=0.3]{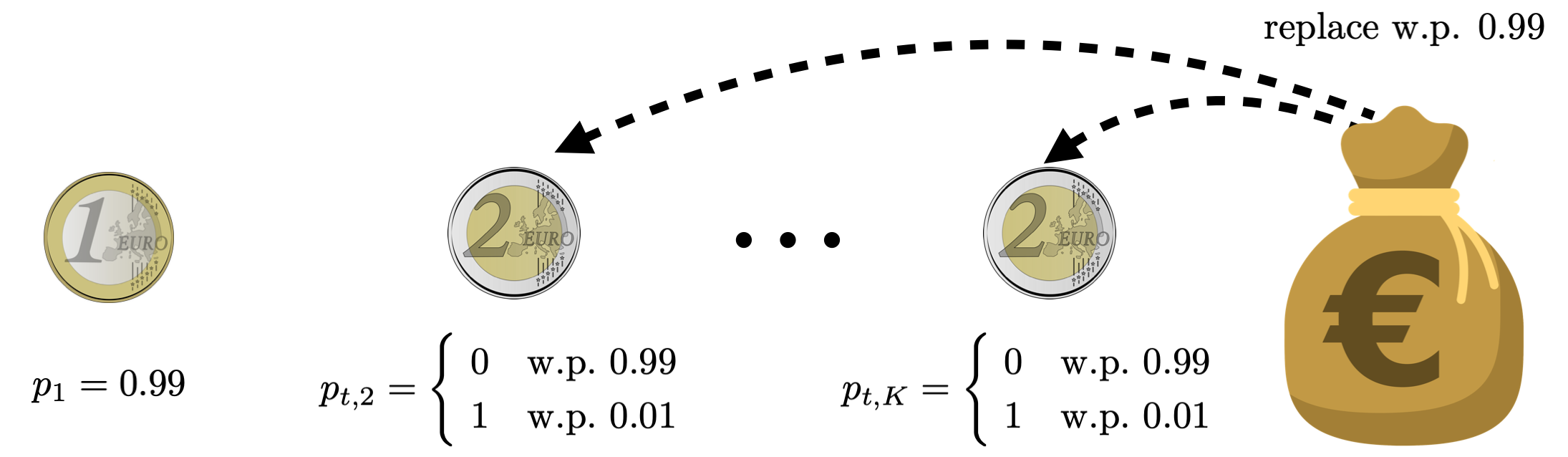}
\end{center}
\caption{\label{fig:many_coins_n}
Choosing among $K$ coins: the third through $K$-th coins are independent copies of the second coin.}
\end{figure}

In such an environment, TS performs almost as badly as the worst-performing agent. 
To see why, first {note that TS samples $\hat{p}_{t,1}$ from the posterior of $p_{1}$ and 
$\hat{p}_{t, k}$ from the posterior of $p_{t,k}$ for $k \in \{2, ... , K\}$, and selects the coin with the largest sampled value. }
\remove{observe that}TS takes $\hat{p}_{t,1} = 0.99$ as before. {At each timestep, each coin from the second to the $K$-th} is independently replaced with a positive probability $0.99$, and that the bag contains a positive proportion $1\%$ of coins with bias $1$. Therefore, $\hat{p}_{t,k} = 1$ with a \remove{positive probability}{probability that is at least $0.99 \times 0.01 = 0.0099$}, for $k \in \{2, ..., K\}$. 
When $K$ is sufficiently large, by maximizing among $\hat{p}_{t, 1}, \hat{p}_{t,2}, ... , \hat{p}_{t, K}$, 
 TS selects one of the second through the $K$-th coins with a sufficiently large probability; the probability converges to $1$ as $K \rightarrow +\infty$. 
However, {tossing any coin from the second to the $K$-th} yields an expected payoff at a rate that is close to $0$\textcent
\remove{; a simple derivation reveals that this rate is less than $2$\textcent\ per timestep}.
{Indeed, a simple derivation (Appendix~\ref{appendix:motivating_example_k}) reveals that} in an environment where $K$ is sufficiently large, 
TS collects an expected payoff that is at most $2$\textcent\ per timestep, much smaller than the rate of $99$\textcent\ accumulated by an agent that only ever selects the first coin. This gap of $97$\textcent\ in payoffs per timestep is large, considering that the expected payoff per timestep lies in the range $[\$0, \$1]$. 

In fact, following a similar argument, we can show that there exists an environment with multiple coins where the performance gap between TS and the optimal agent can be arbitrarily close to \$1. In other words, we are able to show that in such an environment, 
TS 
accumulates an expected payoff 
that is arbitrarily close to $\$0$ and that an optimal agent accumulates an expected payoff that is arbitrarily close to $\$1$. Since payoffs in such environments are binary-valued, this indicates that TS performs arbitrarily close to the worst-possible agent for this environment. 
The above observations will be formalized in Theorem~\ref{theorem:performance_difference_1} of Section~\ref{section:learning_target} using the language of Bernoulli bandits.

%% file: arxiv/definition.tex
\section{Models and Definitions}
\label{sec:bandits}
We formally describe in this section the system model and some key definitions.  All random quantities are defined with respect to a probability space $(\Omega, \mathcal{F}, \Pr)$. 

We first formalize the concept of a bandit. \remove{Let $\actions$ be a finite set, and
$\{R_t\}_{t = 1}^{+\infty}$
be a stochastic process taking values in 
$\R^{|\actions|}$.}{Let $\mathcal{A}$ be a finite set, and 
$\{R_t\}_{t = 1}^{+\infty}$ be a stochastic process, where each $R_t = (R_{t,a}: a \in \mathcal{A})$ takes values in $\mathbb{R}^{|\actions|}$.} A \emph{bandit} is defined by the tuple $(\{R_t\}_{t = 1}^{+\infty}, \actions)$, where the two elements correspond to the reward process and the action set, respectively. In particular, for every $t \in \mathbb{Z}_+$ and $a \in \actions$, $R_{t,a}$ represents the reward that will be realized if an agent executes action $a$ at timestep $t$. We use $\mathbb{Z}_+$ to denote positive integers. We use $R_{1:+\infty}$ as a shorthand for the full reward sequence, and $R_{i:j}$ as a shorthand for the reward sequence $\{R_t\}_{t = i}^j$. 

Let $\histories$ denote the set of all sequences of a finite number of action-reward pairs. We refer to the elements of $\histories$ as \emph{histories}. 
A \emph{policy} is a function that maps a history in $\mathcal{H}$ to a probability distribution over $\actions$. 
So a policy $\pi$ assigns, for each realization of history $ h \in \mathcal{H}$, a probability $\pi(a|h)$ of choosing an action $a$ for all $a \in \actions$. {For any policy $\pi$, we use $A_t^{\pi}$ to denote the action selected at time $t$ by an agent following policy $\pi$, and  
$H^{\pi}_{t}$ to denote the history generated upon selecting $A_t^{\pi}$ and observing the subsequent reward while executing $\pi$.}  
Specifically, we let $H^{\pi}_0$ be the empty history, and iteratively define $A_t^{\pi}$ and $H_{t}^{\pi}$ for all $t \in \mathbb{Z}_{+}$. We let $A_t^{\pi}$ be such that $\Pr(A^{\pi}_t \in \cdot |H^{\pi}_{t-1}) =  \pi(\cdot|H^{\pi}_{t-1})$ and that $A_t^{\pi}$ is independent of $R_{1:+\infty}$ 
conditioned on $H_{t-1}^{\pi}$, and let {$H^{\pi}_{t} = (A^{\pi}_1, R_{1,A^{\pi}_1}, \ldots, A^{\pi}_{t}, R_{t,A^{\pi}_{t}})$}. 

Much of the work presented in this paper studies an agent that executes a specific policy, i.e., PS. Note that when it is clear from the context, we suppress superscripts that indicate this. For example, we use 
$A_t$ for the action selected and 
$H_t$ for the history generated as an agent executes PS.

{It is worth noting that we adopt a standard Bayesian framework, and as such model all uncertain quantities as random variables. 
The information about the environment that the agent possesses at the beginning of time is represented as a joint distribution 
over the sequence of reward vectors $\Pr(\{R_t\}_{t = 1}^{+\infty} \in \cdot)$. 
The agent then iteratively refines their knowledge of the reward process by updating its posterior distribution using the observed rewards and the rules of conditional probabilities. In our model the actions do not have delayed consequences. 
} 

\paragraph{Nonstationary Bandits vs. Stationary Stochastic Processes}
Consistent with the bandit learning literature, we define a bandit as stationary if its reward process $\{R_t\}_{t=1}^{+\infty}$ is exchangeable. That is, its joint distribution is invariant under any finite permutation of time indices. By de Finetti’s theorem, this is equivalent to the existence of a distribution $P$ over $\mathbb{R}^{|\mathcal{A}|}$ such that, conditioned on $P$, the rewards are i.i.d. draws from $P$. We refer to this $P$ as \emph{the reward distribution}, which is only well-defined under stationarity. A bandit is said to be non-stationary if it does not satisfy this property. \add{Note that this definition of non-stationary bandit environments differs from the definition of stationarity in the theory of stochastic processes, where a process is considered stationary if it admits a joint distribution that is invariant under time shifts.}

\remove{In order to determine whether a bandit is non-stationary, 
We propose a definition of stationary bandits that is consistent with all stationary bandit models in the literature that we are aware of. We say that a bandit is stationary if the reward process, $\{R_t\}_{t = 1}^{+\infty}$, is exchangeable. That is, the joint distribution of the reward process is invariant under any finite permutation of the time indices. 
By de Finetti's theorem, we also obtain an equivalent, and more familiar definition: a bandit is stationary if there exists a distribution $P$ over $\R^{|\actions|}$ such that, conditioned on $P$, the rewards are independently and identically distributed according to $P$. We refer to this $P$ as the \emph{reward distribution}. This is a hidden state of the 
{It is worth mentioning that, here, the notion of a reward distribution is only defined for stationary bandits.} 
With the above definition in place, we say a bandit is \emph{non-stationary} if it is not stationary.}

%% file: arxiv/algo_PS.tex
\section{Predictive Sampling}
\label{sec:ps} 
This section introduces the predictive sampling (PS) algorithm.  

\subsection{Setting a Different Learning Target}
\label{section:learning_target}
We start by introducing a concept that is central to both the design and analysis of PS:  learning target. Learning naturally occurs as an agent acquires more information about the rewards process when interacting with a bandit environment. A learning target, $\chi$,  is a random variable that formalizes, and further crystallizes, \emph{about what} the agent aims 
to learn in this process. For instance, taking $\chi$ to be the mean rewards of different arms, one may cast many existing algorithms for stationary bandits as trying to strike a balance between gaining more information about $\chi$ and maximizing instantaneous rewards \citep{arumugam2021deciding, arumugam2021value, lu2021reinforcement, russo2022satisficing}. 

But we can take it a step further, by using the learning target not merely as a device to interpret existing algorithms, but also a design tool to actively shape agent's exploration behavior. This led us to a crucial insight: by choosing the {appropriate} learning target in a TS-like algorithm, we can overcome TS's failure to account for information durability. Specifically, we argue that a promising candidate for the learning target is the {sequence of all future {reward vectors}}, $R_{t+1:\infty}$. Intuitively, one would expect an agent that aims to learn about the entire future reward {vector} sequence would naturally have to take into account the durability of the information she gathers in each timestep. 

Note that we set the learning target as $R_{t+1:\infty}$ instead of $R_{t:\infty}$ when selecting action $A_t$. Since we observe $R_{t, A_t}$ immediately after selecting $A_t$, the rewards $R_{t+1:\infty}$ are precisely the ones that, if learned, can help us make better decisions for future actions $A_{t+1}, A_{t+2}$, and so on. The additional $R_t$ beyond $R_{t+1:\infty}$ does not contribute any further information for selecting improved future actions.

Building on the above insight, the PS algorithm follows naturally from the following two-step procedure: First, we provide a general formulation of TS so as to make the role of the learning target explicit. In particular, we frame TS as an agent who, at each timestep $t$, 
\begin{enumerate}
    \item \textbf{samples} a statistically plausible learning target $\hat{\chi}_t$ from its posterior $\Pr(\chi_t\in \cdot | H_{t-1}^{\pi_{\mathrm{TS}}})$, 
    \item \textbf{estimates} the conditional mean reward 
    $\hat{\theta}_t^{\pi_{\mathrm{TS}}} = 
    \E[R_{t} | H_{t-1}^{\pi_{\mathrm{TS}}}, \chi_t \leftarrow \hat{\chi}_t]$ 
    given the sampled learning target, 
    \item \textbf{selects} the action that has the highest mean reward estimate. 
\end{enumerate}
 In the second step, we simply replace the learning target in the above TS procedure with the sequence of all future {reward vectors}\remove{rewards} $R_{t+1:\infty}$, and doing so immediately leads to the PS algorithm. 
 
In other words, PS builds upon TS by changing not \emph{how} it samples, but \emph{what} it samples. The most commonly used learning target in TS and its variants is the \add{current-step mean reward or reward distribution (a random variable itself)}, irrespective of how ``quickly" \add{it} is about to change. This choice could incentivize an agent to expand valuable resources on learning a piece of information that will quickly lose
relevance in a non-stationary environment.

As a concrete piece of evidence, Theorem~\ref{theorem:performance_difference_1} establishes that a TS agent can suffer near worst-case performance in certain (non-stationary) Bernoulli bandits. Here, we use 
$\pi_{\text{TS}}$ to denote the policy executed by TS; for any policy $\pi$ and $T \in \mathbb{Z}_+$, denote by $\mathrm{Return}(T; \pi)$ the expected cumulative reward collected by an agent that executes $\pi$: 
$$
    \mathrm{Return}(T; \pi) = \sum_{t=1}^{T} \E\left[R_{t, A_t^{\pi}}\right]; 
$$
a Bernoulli bandit is a bandit where the rewards are $\{0,1\}$-valued.

\begin{restatable}{theorem}{maximumdifferencets}
\label{theorem:performance_difference_1}
For all $\epsilon \in (0, 1)$, 
there exists a Bernoulli bandit $\nu$ and a policy $\pi$ such that under $\nu$, 
\begin{align*}
     \mathrm{Return}(T; \pi_{\mathrm{TS}}) \leq \epsilon T, \text{ and }
     \mathrm{Return}(T; \pi) \geq (1 - \epsilon)T, \quad \mbox{for all $T \in \mathbb{Z}_{+}$.}
\end{align*}
\end{restatable}
A proof of this result is provided in Appendix~\ref{appendix:difference}.

\subsection{The Predictive Sampling Algorithm}
\label{sec:algorithm}

\begin{figure}[ht]
\centering
\begin{minipage}{0.55\textwidth}
\begin{algorithm}[H]
\caption{Predictive sampling (PS)}\label{alg:predictive-sampling}
\For{$t = 1, 2, \ldots, T$}{
\textbf{sample}: $\hat{R}^{(t)}_{t+1:\infty} \sim \Pr(R_{t+1:\infty} \in \cdot | H_{t-1})$ \\
\textbf{estimate}: $\hat{\theta}_{t} = \E[R_{t} | H_{t-1}, R_{t+1:\infty} \leftarrow \hat{R}^{(t)}_{t+1:\infty}]$ \\
\textbf{select}: $A_t \in \argmax_{a \in \actions} \hat{\theta}_{t, a}$ \\
\textbf{observe}: $R_{t,A_t}$ 
}
\end{algorithm}
\caption{The predictive sampling algorithm}
\label{algo:PS}
\end{minipage}
\end{figure}

We now provide a formal description of PS, summarized in Figure \ref{algo:PS}. At each timestep $t$, PS performs the following steps: 
\begin{enumerate}
    \item \textbf{samples} an infinite sequence of future {reward vectors}\remove{rewards} $\hat{R}_{t+1:\infty}^{(t)}$ from its posterior distribution, $\Pr(R_{t+1:\infty} \in \cdot | H_{t-1})$,  
    \item \textbf{estimates} expected mean rewards by deriving an estimate $\hat{\theta}_t$, 
    by ``pretending'' that $\hat{R}_{t+1:\infty}^{(t)}$ is the sequence of true future {reward vectors}\remove{rewards} $R_{t+1:\infty}$: 
    \begin{equation}
         \hat{\theta}_{t} = \E[R_{t} | H_{t-1}, R_{t+1:\infty} \leftarrow \hat{R}^{(t)}_{t+1:\infty}], 
    \end{equation}
\item \textbf{selects} the action that maximizes $\hat{\theta}_{t, a}$, by setting $A_t \in \argmax_{a \in \actions} \hat{\theta}_{t, a}$. 
\end{enumerate}
In Step 2, the notation $X \leftarrow Y$ denotes a change of measure from that of the random variable $X$ to that of the random variable $Y$:  if we let $f(x) = \E[R_{t} | H_{t-1}, R_{t+1:\infty} = x]$, then $\hat{\theta}_t = f(\hat{R}^{(t)}_{t+1:\infty})$.   This notation is  formally defined in Appendix~\ref{appendix:change-of-measure}. 

While feasibly sampling an entire infinite sequence of future rewards remains challenging, Section~\ref{sec:ps_ar1} presents an efficient implementation of PS for a class of non-stationary Gaussian bandits. This implementation leverages an observation introduced at the beginning of Section~\ref{sec:ps_information}, which shows that sampling $\hat{\theta}_t$ from the posterior distribution of a certain parameter is equivalent to sampling future rewards and computing $\hat{\theta}_t$ from them. \add{ Note also that while some non-stationary bandit models we examine later admit a natural interpretation where the reward process can be seen as generated by a hidden Markov model (e.g., the modulated Bernoulli bandit), the predictive sampling algorithm itself does not rely on any hidden Markovian interpretation of the reward process.} \addnew{In such a bandit, $t_s$ typically uses the hidden state as the learning target, while $p_s$ uses $r_{t+1:\infty}$ as the learning target.} \add{This generality can be especially useful, for instance, when PS is realized by training a black-box predictive model, such as a neural network, to directly sample trajectories of future rewards, without specifying or inferring hidden states.}

%% file: arxiv/prop_PS.tex
\section{Qualitative Properties of PS}
\label{section:properties}
Next, we use some examples to illustrate two salient qualitative properties of PS. First, it reacts to information durability, and benefits from doing so. Second, in stationary environments, it coincides with the behavior of TS, and is therefore expected to perform just as well as TS.

\subsection{PS Reacts to Information Durability}
\label{sec:ps_information}
\remove{Let us use the type of coin-tossing environments in Figures \ref{fig:two_coins_replace_prob} and \ref{fig:many_coins_n} to illustrate the mechanism through which PS takes information durability into account. }
First, observe that PS samples the mean reward estimate $\hat{\theta}_t$ according to: 
\begin{align*}
\Pr\left(\hat{\theta}_t \in \cdot | H_{t-1}\right) = \Pr\left(\E[R_{t} | H_{t-1}, R_{t+1:\infty} \leftarrow \hat{R}^{(t)}_{t+1:\infty}] \in \cdot | H_{t-1}\right)
= \Pr\left(\E[R_{t} | H_{t-1}, R_{t+1:\infty}] \in \cdot | H_{t-1}\right), 
\end{align*}
where we have used the fact that the trajectory of future 
{reward vectors}\remove{rewards} are sampled with respect to the posterior distribution: $\Pr(\hat{R}^{(t)}_{t+1:\infty} \in \cdot | H_{t-1}) = \Pr(R_{t+1:\infty} \in \cdot | H_{t-1})$. In other words,  PS samples the mean reward estimate $\hat{\theta}_t$ from the posterior distribution {$\Pr(\E[R_{t} | H_{t-1}, R_{t+1:\infty}] \in \cdot | H_{t-1})$.}\remove{of $\E[R_{t+1} | H_t, R_{t+2:\infty}]$ conditioned on the history $H_t$.} Let us examine how this distribution changes as a function of the information durability of an action. 

{We consider} a more general version of the example given in Figure \ref{fig:two_coins_replace_prob}. Here, in every timestep the second coin can be replaced with a new coin drawn from the bag with a probability
{$q$ that is not necessarily equal to $0.99$.} 
{Below we present a formal description of the example. 
\begin{example} [\bf{Coin-Tossing Example Parameterized by $q$}]
\label{ex:coin_math} Consider two coins. 
The bias of the first coin $p_1 = 0.99$. The sequence $\{p_{t,2}\}_{t = 1}^{+\infty}$ represents the bias of the second coin, and transitions according to 
\begin{align}
\label{eq:bias_transition}
p_{1, 2} = \begin{cases}
0\ \ & \mathrm{w.p.}\ 0.5\\
1 & \mathrm{w.p.} \ 0.5
\end{cases},
\ \ 
p_{t+1, 2} = \begin{cases}
p_{t,2} & \mathrm{w.p.}\ 1 - 0.5 q\\
1 - p_{t,2} & \mathrm{w.p.}\ 0.5 q
\end{cases}.
\end{align}
The outcome of the first coin $R_{t,1} \overset{\mathrm{i.i.d}}{\sim} \mathrm{Bernoulli}(p_1)$, independent of the bias of the other coin $\{p_{t,2}\}_{t = 1}^{+\infty}$ or the outcomes of the other coin $\{R_{t,2}\}_{t = 1}^{+\infty}$. 
The outcome of the second coin $R_{t,2} \sim \mathrm{Bernoulli}(p_{t,2})$, independent of the bias or outcomes at other times or of the other coin. Select one coin at each time, observe its outcome, and collect the corresponding reward. 
\end{example}
}

In Example 1, the bias of the first coin, $p_1 = 0.99$, is known. The transition equation for the bias of the second coin, 
$\{p_{t,2}\}_{t=1}^{+\infty}$, i.e., Equation~\eqref{eq:bias_transition}, is also known. Although the timing of the replacement of the second coin and its bias are not directly observable, one can attempt to estimate the bias $\{p_{t,2}\}_{t = 1}^{+\infty}$ using observed rewards and Equation~\eqref{eq:bias_transition}.

{This example describes a bandit with actions $\actions = \{1,2\}$ and reward process $\{R_t\}_{t = 1}^{+\infty}$. 
In this example, intuitively,}\remove{Intuitively,} the 
\remove{smaller}
{larger}
$q$ is, the 
\remove{less}{more}
likely that the current second coin would be replaced, and therefore the 
\remove{\emph{more durable}}{\emph{less durable}}
any information about its distribution. {More specifically,}\remove{In this example,} due to the independence between the biases of the two coins, \remove{we see that} the mean reward estimate for second coin, $\hat{\theta}_{t,2}$, is drawn from 
the posterior distribution 
{$\Pr(\E[R_{t,2} | H_{t-1}, R_{t+1:\infty,2}] \in \cdot | H_{t-1})$, where $H_{t-1}$ denotes the actions taken in the past and the rewards observed in the past.}\remove{ of $\E[R_{t+1,2} | H_t, R_{t+2:\infty,2}]$ conditional on the history $H_t$.} When the redraw probability $q$ is very large, \addnew{there would have likely been a ``redraw'' shortly after $t$, and therefore the values of future rewards beyond the ``redraw" would have little influence on our knowledge about $R_{t,2}$. Therefore, $\E[R_{t,2} | H_{t-1}, R_{t+1:\infty,2}]$ is relatively close to $\E[R_{t,2} | H_{t-1}]$. In other words, $\E[R_{t,2} | H_{t-1}, R_{t+1:\infty,2}]$ is relatively \emph{insensitive} to \remove{the realized value of}{the sequence of potential outcomes}.} This further implies that this would induce a relatively \emph{small variance} in the sampling distribution 
{$\Pr(\hat{\theta}_{t, 2} \in \cdot | H_{t-1}) 
= \Pr(\E[R_{t,2} | H_{t-1}, R_{t+1:\infty,2}] \in \cdot | H_{t-1})$}
of $\hat{\theta}_{t,2}${.}
\remove{, since multiple values of $R_{t+2:\infty,2}$ can all lead to similar final outcome for $\hat{\theta}_{t,2}$.} 

On the other hand, when the redraw probability $q$ is small, i.e., when the information associated with the second coin is durable, we see the opposite effect { that the sampling variance for $\hat{{\theta}}_{t,2}$ is larger}. In this case, 
{$\Pr(\E[R_{t,2} | H_{t-1}, R_{t+1:\infty,2}] \in \cdot | H_{t-1})$} 
\remove{$\E[R_{t+1,2} | H_t, R_{t+2:\infty,2}]$}
is much more \emph{sensitive} to \remove{the realization of}$R_{t+1:\infty,2}$, beyond just the first few entries, because the agent can more confidently leverage many entries of the future rewards to infer the value of \remove{the next reward}{the immediate reward associated with action $2$}, $R_{t,2}$, knowing that there's hardly any redrawing occurring between now and then. Consequently, we would expect the sampling distribution for $\hat{\theta}_{t,2}$ to have a larger variance. 

In general, all else being equal, increasing the variance of the sampling distribution 
{$\Pr(\hat{\theta}_{t,a} \in \cdot | H_{t-1})$}
of the mean reward estimate of an action {$a$} tends to encourage the exploration of that action. The above analysis therefore suggests that PS naturally tends to favor exploring actions that have higher information durability, a desirable feature. 

To illustrate the performance impact from PS's use of information durability, let us revisit the example in Theorem \ref{theorem:performance_difference_1}. The next theorem shows that PS excels, and in fact achieves near-optimal performance, in a bandit environment where TS failed. A complete proof is provided in Appendix~\ref{appendix:difference}. 
\begin{restatable}{theorem}{maximumdifference}
\label{theorem:performance_difference_2}
For all $\epsilon \in (0, 1)$, 
under the Bernoulli bandit $\nu$ specified in Theorem~\ref{theorem:performance_difference_1}, we have that
\begin{align*}
    \mathrm{Return}(T; \pi_{\mathrm{PS}})  \geq (1-\epsilon) T,  \quad \mbox{for all $T \in \mathbb{Z}_{+}$.}
\end{align*}
\end{restatable}

Finally, we can apply PS to the coin-tossing environments of Figures \ref{fig:two_coins_replace_prob} and \ref{fig:many_coins_n}. It turns out PS executes the optimal policy in both instances. In the two-coin environment of Figure~\ref{fig:two_coins_replace_prob}, each reward {vector} $R_t$ corresponds to the payoff {of selecting different actions}, and 
PS takes the sequence of \remove{future payoffs} 
{all future payoff vectors} $R_{t+1:\infty}$ 
to be the learning target. 
{Recall that Example~\ref{ex:coin_math} with $q = 0.99$ corresponds to this two-coin environment.}
Since the first coin has a known bias of 
{$p_1 = 0.99$}\remove{$0.99$}, PS takes 
{$\hat{\theta}_{t,1} = \mathbb{E}[R_{t,1} | H_{t-1}, R_{t+1:\infty, 1}] = p_1 = 0.99$.}\remove{$\hat{\theta}_{t,1} = 0.99$.} 
Recall that the second coin is replaced with probability 
{$q = 0.99$}\remove{$0.99$}
at each timestep; {more specifically, the bias 
$\{p_{t,2}\}_{t = 1}^{+\infty}$
of the second coin 
transitions according to \eqref{eq:bias_transition}.
}\remove{half of the coins in the bag have bias $0$, and the other half bias $x1$.} Therefore, the sample 
{$\hat{\theta}_{t,2} \sim \Pr(\mathbb{E}[R_{t, 2} | H_{t-1}, R_{t+1:\infty, 2}] \in \cdot | H_{t-1})$}\remove{$\hat{\theta}_{t,2}$}
is close to $0.5$. 
{More specifically, we can easily derive some really loose bounds
$
\mathbb{E}[R_{t, 2} | H_{t-1}, R_{t+1:\infty, 2}]
= 
\Pr(R_{t, 2} = 1 | H_{t-1}, R_{t+1:+\infty, 2}) = \Pr(p_{t, 2} = 1 | H_{t-1}, R_{t+1:+\infty})
\geq q^2/4 \approx 0.245$ 
and 
$
\mathbb{E}[R_{t, 2} | H_{t-1}, R_{t+1:\infty, 2}] 
\leq 1 - q^2/4 = 0.755$.  
Hence, since 
$\hat{\theta}_{t,2} \sim \Pr(\mathbb{E}[R_{t, 2} | H_{t-1}, R_{t+1:\infty}] \in \cdot | H_{t-1})$, 
$\hat{\theta}_{{t-1},2}$ takes values in $[0.245, 0.755]$. } 
Consequently, by maximizing between 
{$\hat{\theta}_{t,1}$ and $\hat{\theta}_{t,2}$}
\remove{$\hat{p}_{t,1}$ and $\hat{p}_{t,2}$}, 
PS only ever selects the first coin and executes the optimal policy in this environment. 

For the $K$-coin environment of Figure~\ref{fig:many_coins_n}, 
{we first present a formal description of it below. 
\begin{example} [\bf{Coin-Tossing Example with $K$ Coins}]
\label{ex:coin_math_k}
Consider $K$ coins. 
The bias of the first coin $p_1 = 0.99$. The sequence $\{p_{t,i}\}_{t = 1}^{+\infty}$ represents the bias of the $i$-th coin, for $i \in \{2, ... , K\}$, and transitions according to 
\begin{align}
\label{eq:bias_transition_K}
p_{1, i} = \begin{cases}
0\ \ & \mathrm{w.p.}\ 0.99\\
1 & \mathrm{w.p.} \ 0.01
\end{cases},
\ \ 
p_{t+1, i} = \begin{cases}
0 & \mathrm{w.p.}\ 0.9901 \mathbf{1}_{\{p_{t,i = 0}\}} + 0.9801 \mathbf{1}_{\{p_{t,i} = 1\}}\\
1 & \mathrm{w.p.}\ 0.0099 \mathbf{1}_{\{p_{t,i} = 0\}} + 0.0199 \mathbf{1}_{\{p_{t,i} = 1\}}
\end{cases},
\end{align}
where $\mathbf{1}$ denotes the indicator function. 
The sequences $\{p_{t,i}\}_{t = 1}^{+\infty}$ are independent across $i \in \{2, ... , K\}$. 
The outcome of the first coin $R_{t,1} \overset{\mathrm{i.i.d}}{\sim} \mathrm{Bernoulli}(p_1)$, independent of the bias of the other coins or the outcomes of the other coins. 
The outcome of the $i$-th coin $R_{t,i} \sim \mathrm{Bernoulli}(p_{t,i})$, for $i \in \{2, ... , K\}$, independent of the bias or outcomes at other times or of the other coin. 
\end{example}
}
{
Example~\ref{ex:coin_math_k} describes a bandit with actions $\actions = \{1, 2, ... , K\}$ and reward process $\{R_t\}_{t = 1}^{+\infty}$. 
In this example}, PS takes $\hat{\theta}_{t,1} = 0.99$ as before. 
\remove{Recall that in this environment, the second coin is again replaced with probability $0.99$ at each timestep, 
but 
$99\%$ of the coins in the bag have bias $0$ and $1 \%$ have bias $1$. Therefore, }
{According to \eqref{eq:bias_transition_K},} 
the sample $\hat{\theta}_{t,2}$ is close to $0.01$. More specifically, 
{by \eqref{eq:bias_transition_K}, $\mathbb{E}[R_{t, 2} | H_{t-1}, R_{t+1:+\infty}]  = \Pr(p_{t, 2} = 1 | H_{t-1}, R_{t+1:\infty}) \in (0, 0.04)$, 
which together with 
$\hat{\theta}_{t,2} \sim 
\Pr(\mathbb{E}[R_{t, 2} | H_{t-1}, R_{t+1:+\infty}] \in \cdot | H_{t-1})$ implies that $\hat{\theta}_{t,2}$ takes values in $(0, 0.04)$. }
\remove{$\hat{\theta}_{t,2} \in (0, 0.02)$.} 

{Each coin from the third to the $K$-th} is an independent copy of the second coin. Therefore, $\hat{\theta}_{t,i} \in (0, 0.04)$ for all $i \in \{2, ... , K\}$. 
By maximizing among 
$\hat{\theta}_{t, 1}, \hat{\theta}_{t, 2}, ... , \hat{\theta}_{t, K}$, 
PS only ever selects the first coin and executes the optimal policy in this environment. Thus, a PS agent accumulates expected payoffs at a rate of $99$\textcent\ per timestep. This is much higher than the rate of at most $2$\textcent\ of the TS agent described in Section~\ref{sec:motivating_discussion}. 

\subsection{PS Coincides with TS in Stationary Bandits}
\label{sec:ps_ts_equiv}
We show here that in a stationary bandit, PS executes the same policy as TS, if the latter uses the reward distribution as the learning target. \add{Recall that, in stationary bandits, there exists an invariant reward distribution—formally a random variable $P$—that characterizes the long-run behavior of rewards. Conditioned on $P$, the rewards are independent and identically distributed according to $P$.} \addnew{For example, in a Gaussian bandit with independent actions, $P$ can be a normal distribution with mean $\theta$ and co-variance matrix $I$, where $\theta$ is a random variable. Here, $P$ is a distribution-valued random variable and the learning target of $t_s$.} The notion of a policy is defined in Section~\ref{sec:bandits}. 
This is a very useful property because we can thus be assured that PS is guaranteed to succeed also in the type of stationary bandits where TS thrives.

{
Before we proceed, recall that the TS agent who targets to learn the reward distribution $P$ of a stationary bandit is presented in Section~\ref{section:learning_target} with $\hat{\chi}_t = P$ for all $t \in \mathbb{Z}_+$. \addnew{.}
The PS agent is presented in Algorithm~\ref{alg:predictive-sampling}. Both agents are provided with identical information and begin with the same belief about the environment. 
}

\begin{restatable}{proposition}{pstsequivalence}
\label{prop:equivalence}
In any stationary bandit where the reward distribution is $P$, a PS agent and a TS agent 
that takes $\chi_t = P$ for all $t \in \mathbb{Z}_+$ 
execute the same policy. 
\end{restatable} 

The proof is given in Appendix~\ref{appendix:equivalence} and leverages
 the fact that in a stationary bandit, the reward distribution is equivalent to PS's learning target; more precisely, the reward distribution $P$ and PS's learning target of the sequence of future 
 {reward vectors}\remove{rewards} 
 $R_{t+1:\infty}$ are equally informative in predicting the immediate reward $R_{t}$. {For example, if $P$ is a normal distribution with mean $\theta$ and variance 1, then knowing $P$ is equivalent to knowing $\theta$, and observing $R_{t+1:\infty}$, an infinite sequence of i.i.d. samples from $\mathcal{N}(\theta, 1)$, provides the same predictive information about $R_{t}$.}

%% file: arxiv/regret_analysis.tex
\section{Regret Analyses}
\label{sec:regret}

The previous sections have presented several desirable properties of PS and showcased their performance implications in some examples. The goal of this section is to provide general theoretical guarantees on the performance of PS for a much wider range of environments. 

To do so, we will first introduce a notion of regret for non-stationary bandits.  
We then establish a theoretical framework for information-theoretic regret analyses, which generalizes that developed by \cite{RussoMOR2014} for stationary bandits. 
Critical to the framework is a new information ratio and the concept of predictive information. 
We establish a regret bound that applies to any agent. 
We specialize the bound to PS and then further to non-stationary Bernoulli bandits. These bounds suggest that PS performs well across a range of such bandits.

\subsection{Performance and Regret}
\label{sec:regret_def} 
Regret is a widely used metric in the bandit learning literature that measures the difference between the rewards collected by an oracle and that by an agent. While traditionally the oracle is one who knows and chooses the optimal action, the notion of an optimal action does not easily extend to a non-stationary environment {because the ``expected reward" at each timestep cannot be unambiguously defined. This issue is discussed extensively in \citep{liu2023definition}. We provide a simple example in Appendix~\ref{appendix:regret_non-stationary}, showing that multiple valid and yet contradicting ``expected rewards" can exist.}

{With the above context in mind,} the first definition of regret that we will consider involves an oracle that sees the entirety of all realizations of past rewards from all actions, $R_{1: {t-1}}$, and selects the action at each timestep to maximize the expected reward $\E[R_{t, a} | R_{1: {t-1}}]$. 

\begin{definition}
[\bf{Hindsight Regret}]
\label{def:regret}
For all policies $\pi$ and $T \in \mathbb{Z}_+$, 
the hindsight regret associated with a policy $\pi$ over $T$ timesteps is
\begin{align}
\mathrm{Regret}_{\mathrm{H}}(T; \pi) = \sum_{t=1}^{T} \E \left[R_{t, *} - R_{t, A_t^{\pi}}\right],
\label{eq:regret}
\end{align}
where   
$R_{t, *} =  
\max_{a \in \actions}\E[R_{t, a} | R_{1:t-1}]$.  
\end{definition}
Because a real agent cannot observe the rewards of the actions that she did not choose in a given timestep, the aforementioned oracle is  ensured to outperform any agent. This fact is formalized in the following result, which indicates that the hindsight regret is always non-negative; the proof is given in Appendix~\ref{appendix:regret_nonnegative}.

\begin{restatable}{proposition}{regretnonnegative}
For all policies $\pi$ and $T \in \mathbb{Z}_{+}$,
\begin{align*}
 \sum_{t=1}^{T} \E \left[R_{t, *}\right] \geq \sum_{t = 1}^{T} \E \left[R_{t, A_t^{\pi}}\right].
\end{align*}
\label{proposition:baseline}
\end{restatable}

Because it can be difficult to directly analyze the hindsight regret in general bandit environments, we will also introduce a second regret definition, which will serve as a more tractable proxy for hindsight regret within the family of non-stationary bandits that we will focus on. Instead of an oracle who sees all past rewards, we consider an oracle that instead has access to 
{all \emph{future} reward vectors,}\remove{the realizations of all \emph{future} reward realizations,} and subsequently chooses the action that maximizes the conditional mean reward $\E[R_{t,a}|R_{t+1:\infty}]$. 
\begin{definition}
[\bf{Foresight Regret}]
\label{def:strong_regret} 
For all policies $\pi$ and $T \in \mathbb{Z}_+$, 
the foresight regret associated with a policy $\pi$ over $T$ timesteps is
\begin{align}
\mathrm{Regret}_{\mathrm{F}}(T; \pi) = \sum_{t=1}^{T} \E \left[R_{t, \overline{*}} - R_{t, A_t^{\pi}}\right],
\label{eq:foresight_regret}
\end{align}
where 
$R_{t, \overline{*}} =  
\max_{a \in \actions}\E[R_{t, a} | R_{t+1:\infty}]$.  
\end{definition}

The foresight regret has an advantage in tractability. Because our main aim in this paper is to characterize the performance of PS, which takes simulated future trajectories as a crucial input, foresight regret is more amenable to analysis because the oracle therein also uses future reward trajectories to drive actions. Furthermore, we will show in the subsequent section that, within a family of non-stationary bandits with a reversibility structure, foresight regret is always no smaller than hindsight regret, thus making it useful as a tool to upper-bound the latter. 

Another benefit of foresight regret is that it generalizes the conventional notion of regret used in stationary bandit learning literature \citep{lai1985asymptotically, neu2022lifting}. To see why, observe that in a stationary bandit {with reward distribution $P$}, an oracle that sees all future {reward vectors $R_{t+1:\infty}$} \remove{rewards}{can recover $P$ from these rewards and} 
will simply pick the action with the largest mean 
{$\mathbb{E}[R_{t,a}|P]$}
at all timesteps. 
{For instance, consider a bandit where $R_{t,a} = 0$ for all $t$ or $R_{t,a} = 1$ for all $t$ with equal probability. In this case, 
$R_{t, \overline{\star}} = \max_{a \in \actions}\E[R_{t,a}|R_{t+1:\infty}] = \max_{a \in \actions}R_{t+1,a} = \max_{a \in \actions} R_{1,a}$ and we recover the conventional benchmark.} As a result,
the upper bounds we derive with foresight regret can then \addnew{be} compared to known bounds on conventional regret when the bandit is stationary.

\subsection{Reversible Bandits}
\label{sec:reversible}
{Many}\remove{The majority} of our theoretical results will focus on the performance of an agent in a class of non-stationary bandits which we will refer to as reversible bandits. 
{We explicitly state the reversibility assumption in results that require it. The reversible bandits}\remove{They}
are defined as follows: 
\begin{definition} [\bf{Reversible Bandit}]
A bandit is reversible if the reward process $\{R_t\}_{t = 1}^{+\infty}$ is reversible in the following sense: for all $T \in \mathbb{Z}_{+}$, $\{R_t\}_{t=1}^T$ and $\{R_{T+1-t}\}_{t=1}^T$ have the same joint distribution. 
\end{definition}

Reversible bandits encompass a wide range of bandit environments. All stationary bandits are reversible, so are all non-stationary bandits discussed in this paper, including the modulated Bernoulli bandits,  AR(1) bandits, and AR(1) logistic bandits. There are certainly non-stationary bandits that are not reversible, whose analyses will be outside the scope of this paper but we believe can be an interesting direction for future work. 

The following result shows that in a reversible bandit, the hindsight regret (Definition \ref{def:regret}) is always bounded from above by the foresight regret (Definition \ref{def:strong_regret}); the proof is provided in Appendix~\ref{appendix:strong_regret}. As a consequence, we will be focusing on characterizing the foresight regret of an agent in the remainder of the paper, with the understanding that any upper bound will also apply to the hindsight regret when the bandit is reversible. 
\begin{restatable}{proposition}{strongregret}
Suppose the bandit is reversible. For all policies $\pi$ and $T \in \mathbb{Z}_+$, 
\begin{align*}
\mathrm{Regret}_{\mathrm{H}}(T; \pi) \leq {\mathrm{Regret}}_{\mathrm{F}}(T; \pi). 
\end{align*}
\label{proposition:strong_regret}
\end{restatable}
\subsection{Predictive Information and a New Information Ratio}
\label{sec:entropy_analysis}

As our primary analytical framework, we will be using a variant of the information-theoretic regret analysis developed in the stationary bandit learning literature. 
{Before we proceed, let us introduce some information-theoretic concepts and notation, specifically the entropy and mutual information. We introduce these concepts as defined in \citep{cover1999elements} for discrete random variables. The continuous generalization is presented in \citep{gray2011entropy}. 

Let $X$, $Y$, and $Z$ be discrete random variables taking values in $\mathcal{X}$, $\mathcal{Y}$, and $\mathcal{Z}$, respectively. Then the entropy of $X$ is defined as: 
\begin{align*}
\H(X) = -\sum_{x \in \mathcal{X}}\Pr(X = x) \log \Pr(X = x).
\end{align*}
Here, the $\log$ is base $2$, and entropy is expressed in bits. Additionally, we use the convention that $0\log 0 = 0$. Entropy measures the uncertainty of a random variable, quantifying information required to learn about it. Entropy is always non-negative. 

The conditional entropy $\H(Y|X)$ is defined as: 
\begin{align*}
\H(Y|X) = \sum_{x \in \mathcal{X}}\Pr(X = x) \H(Y|X = x) = -\sum_{x \in \mathcal{X}} \sum_{y \in \mathcal{Y}}\Pr(X = x, Y = y) \log \Pr(Y = y | X = x). 
\end{align*}
It measures the uncertainty of $Y$ given full knowledge of $X$, or the additional information required to learn about $Y$ given full knowledge of $X$. 

The mutual information between $X$ and $Y$ is defined as: 
\begin{align*}
\I(X; Y) = \sum_{x \in \mathcal{X}} \sum_{y \in \mathcal{Y}}
\Pr(X = x, Y = y) \log \frac{\Pr(X = x, Y = y)}{\Pr(X = x) \Pr(Y = y)}.
\end{align*}
It is noteworthy that mutual information is symmetric around its arguments $\I(X; Y) = \I(Y; X)$. 
It measures the dependency between two random variables. 
The following equality establishes a connection between entropy and mutual information:
$
\I(X; Y) = \H(X) - \H(X | Y).
$
Intuitively, this equality shows that mutual information quantifies the reduction of uncertainty of one random variable due to another or the amount of information acquired about one random variable by acquiring full knowledge about another. Similarly, conditional mutual information can be defined as $
\I(X; Y | Z) = \H(X | Z) - \H(X | Y, Z). 
$
}

{The information-theoretical regret analysis}\remove{This analysis} 
was first used by \citep{ JMLR:v17:14-087} for analyzing TS and subsequently extended to many other settings \citep{bubeck2015bandit, lattimore2019information, lu2021reinforcement, neu2022lifting}. The key idea is to bound the regret in terms of a certain information ratio, $\Gamma$, defined to be the ratio between a function of the immediate regret and the information gain about a learning target. The logic is that a ``good'' agent who optimally balances between information acquisition and reward maximization ought to have a relatively small information ratio, and \emph{vice versa}. While proven extremely versatile in analyzing various stationary bandit learning agents, the aforementioned framework does not apply to non-stationary environments. One of the key obstacles is that the information ratio originally defined for stationary bandits do not generalize easily. 

To overcome this challenge, we introduce in this paper a new information ratio that is better suited to our regret analysis for non-stationary bandits. The first change is that we now measure immediate regret with respect to the benchmark $R_{t,\overline{*}}$ as per Definition \ref{def:strong_regret}, instead of the reward associated with the optimal action as in stationary bandits. The second, and more important, change is that we now measure information gain with respect to a new learning target, the sequence of future
{reward vectors}\remove{rewards}, $R_{t+1:\infty}$, in contrast to using the current reward distribution as the learning target in stationary bandits.  Formally, we define 
the information ratio associated with policy $\pi$ at timestep $t$ be 
\begin{align}
    \Gamma_t^{\pi} =  \frac{\E\left[R_{t, \overline{*}} - R_{t, A_t^{\pi}} \right]^2}{\I\left(R_{t+1:\infty}; A_t^{\pi}, R_{t, A_t^{\pi}} | H_{t-1}\right)}. 
\label{eq:information_ratio}
\end{align} This information ratio measures how an agent trades off between a single-timestep regret and information about $R_{t+1:\infty}$. Since much of the paper studies PS, we simplify the notation in this case and use $\Gamma_t$ to denote the information ratio associated with PS.

As mentioned earlier, the foresight regret coincides with the conventional regret definition in stationary bandits. Similarly, it is easy to show that the information gain defined here also coincides with one for stationary bandits (cf.~\cite{neu2022lifting}).

Next, we introduce the notion of predictive information.
It represents the new information that is being injected to the environment during each timestep, due to non-stationarity. Specifically, for all  $t$, we define the \emph{incremental predictive information} at time $t$ as the mutual information between the next reward $R_{t+1}$ and the sequence of all future 
{reward vectors}\remove{rewards} 
$R_{t+2:\infty}$, conditional on $R_{1:t}$. The definition is formalized below.

\begin{definition} [\bf{Incremental Predictive Information}]
For all $t \in \mathbb{Z}_+$, the incremental predictive information at timestep $t$ is
\begin{align*}
\Delta_t = 
\I(R_{t+1:\infty}; R_{t} | R_{1:{t-1}}). 
\end{align*}
\label{def:predictive_information}
\end{definition}
Intuitively, $\Delta_t$ measures the \emph{incremental} information one can learn about future rewards
$R_{t+1:\infty}$ by observing current reward $R_{t}$, \emph{after} having already observed all past rewards $R_{1:t-1}$. If the environment is highly stationary, then as $t$ increases, there is little additional information gained from observing the current reward, because most of the information about future rewards are already contained in the past rewards $R_{1:t-1}$. On the other hand, when the environment is highly non-stationary, historical data plays a less important role, and observing the immediate reward $R_{t}$ can be extremely helpful in predicting future rewards no matter how large $t$ is. In summary, one would expect that the predictive information $\Delta_t$ to decay as $t$ increases in a stationary environment, but to remain relatively large in a non-stationary one. This behavior is  consistent with what we would expect from a good metric to quantify information. 

Finally, we refer to the cumulative sum of incremental predictive information, $\sum_{t'=1}^{t}\Delta_{t'}$,  as the \emph{cumulative predictive information} at time $t$, representing the total new uncertainty that has been injected into the environment thus far. 
{The cumulative predictive information is bounded for stationary environments. The following proposition establishes this. 
\begin{restatable}{proposition}{predictiveinformations}
Consider a stationary bandit with reward distribution $P$. The cumulative predictive information satisfies 
$
\sum_{t = 1}^{+\infty} \Delta_t = \H(P).$
\label{proposition:predictive_information_stationary}
\end{restatable}
The proof is presented in Appendix~\ref{appendix:predictive_information_stationary}. 
When the reward distribution $P$ is parameterized by a mean reward vector $\theta$, from Proposition~\ref{proposition:predictive_information_stationary} we recover $\sum_{t = 1}^{+\infty} \Delta_t = \H(\theta)$. 

The cumulative predictive information is typically linear in most non-stationary environments we care about. For example, in the coin-tossing game of Figure~\ref{fig:two_coins_replace_prob}, which was formally introduced in Example~\ref{ex:coin_math} with $q = 0.99$, the incremental predictive information at each time $t \in \mathbb{Z}_+$ and $t > 1$ satisfies  
\begin{align*}
\Delta_t = \I(R_{t+1:\infty}; R_{t} | R_{1:t-1})
= \I(R_{t+1:\infty}; R_{t} | R_{t-1})
= \I(R_{t+1}; R_{t} | R_{t-1})
= \I(R_{3}; R_{2} | R_{1}) > 0, 
\end{align*}
where the second equality follows from the fact that $R_{1:t-2} \perp R_{t:\infty}| R_{t-1}$
and the third equality follows from the fact that $R_{t} \perp R_{t+2:\infty} | R_{t-1}, R_{t+1}$. Thus, the cumulative predictive information $\sum_{t^{\prime} = 1}^t \Delta_{t^{\prime}} = (t-1) \I(R_3; R_2|R_1) + \I(R_2; R_1)> 0$ is linear in time $t$.  

}

The cumulative predictive information will play a crucial role in our analysis, and deriving an upper bound on it often allows us to further upper-bounding the regret itself.

\subsection{Regret Bounds for General Agents}
\label{sec:general_regret}
We are now ready to present our regret analysis. We begin with Theorem~\ref{theorem:general_regret}, which establishes a general upper bound on foresight regret that applies to any agent in any bandit.  
In particular, the bound is expressed in terms of the sum of the information ratios $\sum_{t = 1}^{T-1}\Gamma_t^{\pi}$ and the cumulative predictive information $\sum_{t = 1}^{T}\Delta_t$.

\begin{restatable}{theorem}{generalregret}
For all policies $\pi$ and $T \in \mathbb{Z}_+$, 
\begin{align*}
    {\mathrm{Regret}}_{\mathrm{F}}(T; \pi) \leq \sqrt{\sum_{t = 1}^{T} \Gamma_t^{\pi} \sum_{t = 1}^{T}\Delta_t}. 
\end{align*}
\label{theorem:general_regret}
\end{restatable}
The proof is provided in Appendix~\ref{appendix:general_regret}. 
The proof leverages the fact that when an agent incurs regret, it must acquire some amount of information that is relevant for predicting 
future 
{reward vectors}\remove{rewards} 
$R_{t+1:\infty}$. Therefore, we can upper-bound the cumulative regret by how the agent trades off between immediate regret and such information, which is measured by the sum of information ratios $\sum_{t = 1}^{T}\Gamma_t^{\pi}$, and the cumulative predictive information $\sum_{t = 1}^{T} \Delta_t$, which upper-bounds the total information that an agent acquires to predict future rewards. 

We will subsequently use the general regret bound of Theorem~\ref{theorem:general_regret} by separately bounding the two constituent terms, $\sum_{t = 1}^{T} \Gamma_t^{\pi}$ and $\sum_{t = 1}^{T}\Delta_t$, respectively. The former depends on the choice of the agent or algorithm, while the latter is a function only of the underlying bandit environment. Because the cumulative predictive information is agent-independent, we will begin with it and first introduce Lemma~\ref{lemma:markov_info}. The lemma provides an elegant approach towards bounding the cumulative predictive information in any bandit; the proof is provided in Appendix~\ref{appendix:predictive_information}. 

\begin{restatable}{lemma}{markovinfo} 
{Let $\{S_t\}_{t = 1}^{+\infty}$ 
be a Markov process 
such that, \remove{for all $t \in \mathbb{Z}_+$, 
$S_{t+1:\infty} \perp R_{1:t}| S_t$
and $R_{t+2:\infty} \perp R_{t+1} | S_{t+2}, R_{1:t}$.}{for all $t \in \mathbb{Z}_+$, 
conditioned on $S_t$, $R_t$ is independent of the other elements in $\{S_t\}_{t = 1}^{+\infty}$ and $\{R_t\}_{t = 1}^{+\infty}$ associated with different times, i.e.,
\[
R_{t} \perp (S_{1:t-1}, S_{t+1:\infty}, R_{1:t-1}, R_{t+1:\infty}) | S_t.
\]
}} 
{For} all $T \in \mathbb{Z}_+$, the cumulative predictive information satisfies 
\begin{align*}
\sum_{t = 1}^{T} \Delta_t  \leq 
\I(S_2; S_1) + 
\sum_{t = 1}^{T-1} \I(S_{t+2}; S_{t+1} | S_{t}).  
\end{align*}
\label{lemma:markov_info}
\end{restatable}

\remove{The random variable $S_t$ in the above lemma can be thought of as the hidden state of the bandit at timestep $t$. }

The random variable $S_t$ in the above lemma can be thought of as the hidden state of the bandit at timestep $t$. By suitably constructing $\{S_t\}_{t=1}^{+\infty}$, we can apply this lemma to bound predictive information in any bandit. In particular, this can be achieved by letting $S_t = R_{1:t}$ for all $t \in \mathbb{Z}_+$. For some bandits, there is a better choice of $\{S_t\}_{t=1}^{+\infty}$ with which we can derive sharper bounds or more interpretable bounds when applying the lemma. For instance, in Example 1 which models the coin-tossing game, we can define hidden states to be the coin biases by setting $S_t = (p_1, p_{t,2})$ for all $t \in \mathbb{Z}_+$. In addition, in Example 3 of Section~\ref{sec:modulated_bernoulli}, which we will introduce shortly, we can let $S_t = \theta_t$ for all $t \in \mathbb{Z}_+$.

It is worth noting that, by combining Lemma~\ref{lemma:markov_info} and Theorem~\ref{theorem:general_regret}, and specializing the resulting bound to stationary bandits, we can recover existing regret bounds for stationary bandits. See Appendix~\ref{appendix:recover_stationary} for this result. 

\subsection{Regret Bounds for PS}
We now turn our attention to using the general regret bound of Theorem \ref{theorem:general_regret} to obtain sharper regret bounds for PS in various non-stationary bandits.  

We begin by obtaining an upper bound on the information ratio $\Gamma_t = \Gamma_t^{\pi_\mathrm{PS}}$ associated with PS, expressed in terms of the number of actions and the variance of the rewards. The proof of the following result is provided in Appendix~\ref{appendix:info_ratio}. 
\begin{restatable}{lemma}{ir}
\label{lemma:ir}
If for all $t \in \mathbb{Z}_+$ and $a \in \actions$, $\Pr(R_{t, a} \in \cdot|H_{t-1})$ is almost surely $\sigma_{\text{SG}}$-sub-Gaussian, then for all $t \in \mathbb{Z}_+$, the information ratio associated with PS satisfies 
\begin{align*}
\Gamma_t \leq  2|\actions|\sigma_{\mathrm{SG}}^2. 
\end{align*}
\end{restatable}

The following theorem follows directly from combining Theorem~\ref{theorem:general_regret} and Lemma~\ref{lemma:ir}. It establishes a regret bound for PS in terms of the cumulative predictive information. 
\begin{restatable}{theorem}{psregret}
\label{th:ps_regret}
If for all $t \in \mathbb{Z}_+$ and $a \in \actions$, $\Pr(R_{t, a} \in \cdot|H_{t-1})$ is almost surely $\sigma_{\text{SG}}$-sub-Gaussian, then for all $T \in \mathbb{Z}_+$, the regret of PS satisfies 
\begin{align*} 
{\mathrm{Regret}}_{\mathrm{F}}(T) \leq  
    \sqrt{2 |\actions| \sigma_{\mathrm{SG}}^2 T \sum_{t = 1}^{T}\Delta_t}. 
\end{align*}
\end{restatable}
Finally, using Lemma~\ref{lemma:markov_info} to better characterize the cumulative predictive information, we obtain the following refinement. 

\begin{corollary}
\label{cor:regret}
{Let $\{S_t\}_{t = 1}^{+\infty}$ 
be a Markov process 
such that,} {for all $t \in \mathbb{Z}_+$, 
$R_{t} \perp R_{1:t-1, t+1:\infty}, S_{1:t-1, t+1:\infty}| S_t$.}
{for all $t \in \mathbb{Z}_+$, 
conditioned on $S_t$, $R_t$ is independent of the other elements in $\{S_t\}_{t = 1}^{+\infty}$ and $\{R_t\}_{t = 1}^{+\infty}$ associated with different times, i.e., $R_{t} \perp (S_{1:t-1}, S_{t+1:\infty}, R_{1:t-1}, R_{t+1:\infty}) | S_t$.} 
\remove{If the states $\{S_t\}_{t = 1}^{+\infty}$ is a Markov process, and} 
{If for} all $t \in \mathbb{Z}_+$ and $a \in \actions$, $\Pr(R_{t, a} \in \cdot|H_{t-1})$ is almost surely $\sigma_{\text{SG}}$-sub-Gaussian, then for all $T \in \mathbb{Z}_+$, the regret of PS satisfies 
\begin{align*}
{\mathrm{Regret}}_{\mathrm{F}}(T) \leq \sqrt{2|\actions| \sigma_{\mathrm{SG}}^2 T \left[\I(S_2;S_1) + \sum_{t = 1}^{T-1} \I(S_{t+2}; S_{t+1} | S_{t})\right]}.
\end{align*}
\end{corollary}
Note that specializing the bound of PS in Corollary~\ref{cor:regret} to stationary bandits, we can recover existing regret bounds for TS in stationary bandits. See Appendix~\ref{appendix:recover_stationary} for this result. We will also use this corollary  to derive regret bounds for PS in more specialized bandit environments in the next subsection.

\subsection{Regret Bounds in Modulated Bernoulli Bandits}
\label{sec:modulated_bernoulli}

Building on the results from the previous section, we now focus on a specific model of non-stationary reversible bandits that will enable us to derive a sharper and more interpretable regret bound. 

A majority of the existing literature on non-stationary bandit learning focuses on  Bernoulli bandits, and for this reason we will focus our analysis on this family as well. Below, we consider a family of non-stationary Bernoulli bandits that generalizes an abrupt switching model introduced by \cite{pmlr-v31-mellor13a}. 
It is easy to verify that bandits in  family are reversible, and also  encompass both coin-tossing environments introduced in Section~\ref{sec:motivating_discussion}.

\begin{example} [\bf{Modulated Bernoulli Bandit}]
\label{ex:modulated_bernoulli}
\noindent Consider a Bernoulli bandit with independent actions. 
For all $a \in \actions$, let $\{\theta_{t,a}\}_{t = 1}^{+\infty}$ be a sequence of random variables.  
We refer to $\theta_{t,a}$ as the mean reward associated with action $a$ at timestep $t$; 
conditioning on $\theta_{t,a}$, the reward $R_{t, a} \sim$ Bernoulli($\theta_{t,a}$), independent of the rewards and mean rewards associated with other timesteps or actions. 
Each mean reward sequence $\{\theta_{t,a}\}_{t = 1}^{+\infty}$
transitions according to 
\begin{align*}
    \theta_{t+1, a} = 
    \begin{cases}
    \sim \Pr(\theta_{1,a} \in \cdot),  &\ \text{with probability $q_a$}, \\
    \theta_{t,a}, &\ \text{otherwise}, 
    \end{cases}
\end{align*}
where $q_a \in [0, 1]$ is deterministic and known. 
At each timestep, the mean reward 
$\theta_{t,a}$
can be thought of as ``redrawn" from its initial distribution independently with probability $q_a$.  
\end{example}

{Modulated Bernoulli bandits serve as a stylized model for real applications where non-stationarity manifests as abrupt changes. In a modulated Bernoulli bandit, each $q_a$ determines the frequency of abrupt changes associated with action $a$. 
As we have discussed, modulated Bernoulli bandits generalize a model introduced in \citep{mellor2013thompson}; various other models incorporating abrupt changes have been explored in \citep{abbasi2022new, pmlr-v99-auer19a, besbes2019optimal, besson2019generalized, cheung2019learning, chen2024non, ghatak2021kolmogorov, gupta2011thompson, hartland2006multi, luo2018efficient, mellor2013thompson, raj2017taming, wei2018abruptly, viappiani2013thompson, zhao2020simple}. }

We are now ready to present the main result of this section: an upper bound on the regret of PS in a modulated Bernoulli bandit. The proof, presented in Appendix~\ref{appendix:modulated_bernoulli}, leverages Corollary~\ref{cor:regret} by carefully constructing a sequence $\{S_t\}_{t = 1}^{+\infty}$ where we set $S_t = \theta_{t}$ for all $t \in \mathbb{Z}_+$. 

\remove{\begin{restatable}{theorem}{regretbernoulli}
\label{theorem:regret_bernoulli}
In a modulated Bernoulli bandit, for all $T \in \mathbb{Z}_+$, the regret of PS satisfies 
 $${\mathrm{Regret}}_{\mathrm{F}}(T) \leq \sqrt{\frac{1}{2} |\actions| T\left\{ \sum_{a \in \actions} (1 - q_a) \H(\theta_{0,a})
+ (T-1) \sum_{a \in \actions} \left[2 \H(q_a) + q_a(1 - q_a) \H(\theta_{0,a})\right]\right\}}.$$
\end{restatable}
}
\begin{restatable}{theorem}{regretbernoulli}
\label{theorem:regret_bernoulli}
In a modulated Bernoulli bandit, for all $T \in \mathbb{Z}_+$, the regret of PS satisfies 
 \begin{align*}
 {\mathrm{Regret}}_{\mathrm{F}}(T) \leq &\ 
\sqrt{\frac{1}{2} |\actions| T\left\{ \I(\theta_2; \theta_1) + (T - 1) \I(\theta_3; \theta_2 | \theta_1) \right\}} \\
\leq &\ \sqrt{\frac{1}{2} |\actions| T\left\{ \sum_{a \in \actions} (1 - q_a) \H(\theta_{1,a})
+ (T-1) \sum_{a \in \actions} \left[2 \H(q_a) + q_a(1 - q_a) \H(\theta_{1,a})\right]\right\}}.
\end{align*}
\end{restatable}
\remove{A first observation is that the above regret upper bound scales linearly in horizon, $T$. Interestingly, this scaling matches the next regret lower bound that we present, suggesting that the linear scaling in $T$ in our upper bound is tight. To the best of our knowledge, this is the first known linear regret lower bound for modulated Bernoulli bandits; a proof is provided in Appendix~\ref{appendix:lower_bound}.}
{We now present Theorem~\ref{theorem:lower_bound}, which establishes a regret lower bound in modulated Bernoulli bandits.}

\begin{restatable}{theorem}{regretlowerbound}
There exists a modulated Bernoulli bandit and a constant $C \in \mathbb{R}_{+}$ 
such that, for all policies $\pi$ and $T \in \mathbb{Z}_{+}$, the \addnew{foresight} regret satisfies 
\begin{align*}
   {\mathrm{Regret}}_{\mathrm{F}}(T; \pi) \geq CT.
\end{align*}
\label{theorem:lower_bound}
\end{restatable}

{A key observation is that both the regret upper bound for PS and the regret lower bound scale linearly with horizon $T$. In a typical non-stationary environment, such as a modulated Bernoulli bandit, regret grows linearly because the learning target continuously evolves. In this setting, the leading constant represents the average regret of the bandit policy in non-stationary environments, providing a useful measure for distinguishing weak algorithms from strong ones.

The regret upper bound for PS established in Theorem~\ref{theorem:regret_bernoulli} is meaningful for two main reasons. First, its leading term depends on simple and interpretable primitives in a natural way—namely, predictive information, represented as the sum of mutual information between parameters. Second, in certain cases, the regret upper bound for PS is provably and significantly better than what can be achieved by TS and certain other algorithms or bounds in the literature. Below, we illustrate the second point by analyzing how the regret upper bound suggests that PS performs competitively in extreme cases where $q_a = 0$ or $q_a = 1$, while some existing algorithms and bounds fail to do so. Additionally, through an example, we demonstrate that the regret upper bound provides meaningful performance guarantees and qualitative insights into PS across a wider range of parameter values beyond these extremes.}

\remove{Together, Theorems \ref{theorem:regret_bernoulli} and \ref{theorem:lower_bound} serve as encouraging evidence that PS performs competitively in modulated Bernoulli bandits.}
{The regret upper bound of PS suggests that PS performs competitive in extreme cases where $q_a$ is close to $0$ or $1$. }\remove{Furthermore, we can use Theorem~\ref{theorem:regret_bernoulli} to investigate how the performance of PS depends on various key parameters of the environment. 
Crucially, the bound exhibits a graceful dependence on $q_a$, the probability that a coin is redrawn in each timestep. }
On one hand, when $q_a = 0$ for all $a$, i.e., when the environment is stationary, the bound becomes ${\mathrm{Regret}}_{\mathrm{F}}(T) = \sqrt{\frac{1}{2}|\actions|T \H(\theta_1)}$, which recovers a current best known regret bound for TS \citep{neu2022lifting}. On the other hand, as the $q_a$'s approach $1$, the regret bound approaches $0$, suggesting that PS performs well at the other extreme. This corresponds to the cases where coins are replaced very frequently. When $q_a = 0$ for some actions $a \in \actions_1$, and $q_a = 1$ for the remaining actions $a \in \actions \backslash \actions_1$, the regret bound becomes ${\mathrm{Regret}}_{\mathrm{F}}(T) = \sqrt{\frac{1}{2}|\actions|T \sum_{a \in \actions_1}\H(\theta_{1,a})}$, a value that is small and further upper-bounded by the regret bound established for the case where $q_a = 0$ for all $a \in \actions$. \remove{In summary, our regret bounds suggest that PS performs competitively at the extremes where each $q_a$ is close to either $0$ or $1$.}

{It is noteworthy that existing results, when applied to the modulated Bernoulli bandits of Example~\ref{ex:modulated_bernoulli}, often result in {less competitive performance or} significantly larger regret bounds when the environment exhibits high levels of non-stationarity (i.e., $q_a$ approaching $1$). For example, \cite{besbes2019optimal} establish a regret upper bound for an algorithm, Rexp3, under a different regret metric, in a frequentist setting. When applied to modulated Bernoulli bandits, this bound is linear in $T$ and increases as each $q_a$ increases from $0$ to $1$.  Further details are available in Appendix~\ref{appendix:existing_bounds}.} {Additionally, recall that TS does not perform well in the coin-tossing game introduced in Section~\ref{sec:random_coins}, which can be formulated as a modulated Bernoulli bandit with $q_1 = 0$ and $q_2 = 0.99$. TS also struggles when $q_1 = 0$ and $q_2$ is close to 1, though not necessarily $0.99$. 
Thus, TS serves as another example of an algorithm that fails to perform well in a modulated Bernoulli bandit under extreme conditions. 
}

{We will use a simple family of environments to demonstrate how to leverage the regret upper bound to derive meaningful performance guarantees and qualitative insights on PS in more general parameter ranges beyond the extremes.

\begin{restatable}{example}{modulatedBernoulliBandit} {\bf{(A Class of Modulated Bernoulli Bandit with Two Actions)}}
Consider a modulated Bernoulli bandit with two actions $\mathcal{A} = \{1, 2\}$. Let $\theta_{t, 1} = 0.9$ for all $t \in \mathbb{Z}_+$. Let $\theta_{1, 2} \sim \mathrm{unif}(\{0, 1\})$. 
\label{ex:modulated_bernoulli_example}
\end{restatable}

This example represents a class of modulated Bernoulli bandits, each characterized by the parameter $q_2$. This also describes a set of coin-tossing games, where the first coin has a known bias of $0.9$, and the second coin is replaced at each timestep with probability $q_2$. 

If PS outperforms both a random policy with uniform action selection and TS, it indicates that PS exhibits intelligent action selection.  
\remove{We proceed to plot the first regret upper bound of PS established by Theorem~\ref{theorem:regret_bernoulli}, alongside with regret lower bound of TS when $q_2 > 0.2$, and a regret lower bound of a random (uniform) policy. (For detailed information on the regret lower bounds, please refer to Appendix~\ref{appendix:example_random_TS}.)} {We plot a regret upper bound of PS, alongside a regret lower bound of TS when $q_2 > 0.2$ and a regret lower bound of a random (uniform) policy. 
The regret bounds mentioned above pertain to the average regret, which at time $T$ is defined as the difference between the expected rewards accumulated by an agent and those accumulated by an optimal agent, divided by $T$. 
To derive an upper bound on the regret for PS, we divide the first inequality in Theorem~\ref{theorem:regret_bernoulli} by $T$, as forward regret divided by $T$ upper-bounds the average regret. For details on the regret lower bounds, see Appendix~\ref{appendix:example_random_TS}.} 

Specifically, Figure~\ref{fig:avg_regret} plots these regret bounds in Example~\ref{ex:modulated_bernoulli_example} across different values of $q_2$.  
The plot indicates that when $q_2 > 0.78$, Theorem~\ref{theorem:regret_bernoulli} establishes that PS provably outperforms a random policy, and when $q_2 > 0.85$, PS also outperforms TS. In summary, these experiments indicate that our regret bounds establish that PS exhibits intelligent action selection in a range of environments as $q_2$ moves away from $1$. 
\begin{figure} [!ht]
\centering
\begin{subfigure}[b]{0.45\textwidth}
\centering
\includegraphics[width=\textwidth]{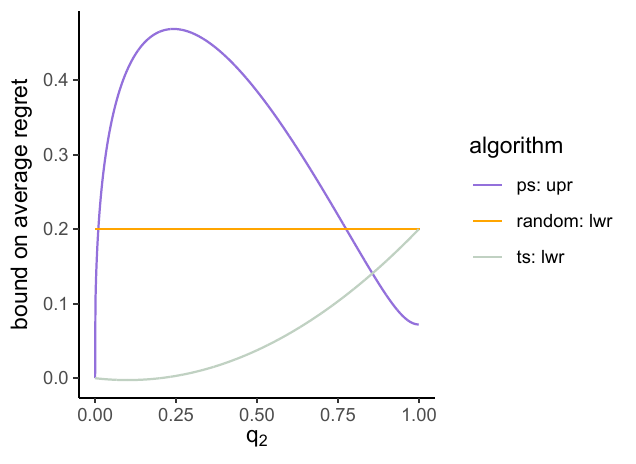}
\caption{When $q_2 \in [0, 1]$}
\label{fig:avg_regret_wide}
\end{subfigure}
\hfill
\begin{subfigure}[b]{0.45\textwidth}
\centering
\includegraphics[width=\textwidth]{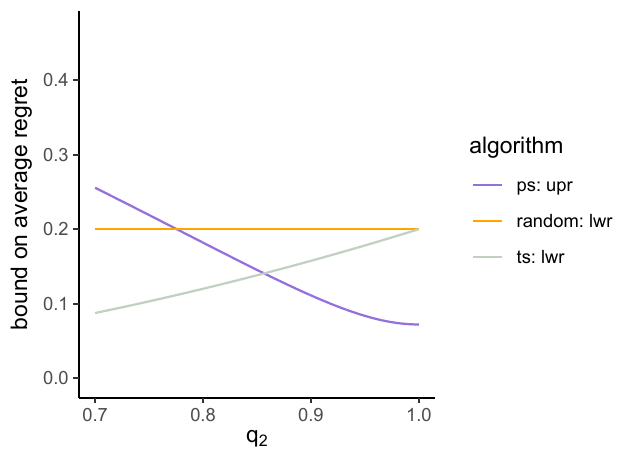}
\caption{When $q_2 \in [0.7, 1]$}
\label{fig:avg_regret_zoom}
\end{subfigure}
\caption{The regret upper bound of PS {('ps:upr')} and regret lower bounds of TS {('ts:lwr')} and a random (uniform) policy {('random:lwr')} in Example~\ref{ex:modulated_bernoulli_example} with various values of $q_2$}
\label{fig:avg_regret}
\end{figure}
}

%% file: arxiv/experiment.tex
\section{Efficient Implementations and Experiments}
\label{sec:experiments}
While the PS procedure is very easy to state, deriving the conditional probability distribution and subsequently sampling from the distribution can be computationally challenging. To address this, this section introduces efficient procedures to execute PS 
\remove{in some bandits 
and useful techniques to execute an approximation of the algorithm in others. 
In particular, we develop an efficient implementation of PS} in a class of Gaussian bandits. 
{Specifically, the execution of PS circumvents explicitly sampling of future reward sequences and direct computation of $\hat{\theta}_t$. Instead, $\hat{\theta}_t$ is sampled from a Gaussian distribution with a mean and variance that are updated iteratively. }
We also design efficient procedures to execute an approximation of PS in a class of logistic bandits, in Appendix~\ref{appendix:ar_logistic}. To examine the advantage of PS over TS, we conduct experiments\remove{in the aforementioned class of Gaussian bandits and logistic bandits}. The results suggest that PS consistently outperforms TS over time in environments with varying information durability. Additionally, we compare PS with other representative non-stationary bandit learning algorithms and show that PS consistently outperforms them.

\subsection{PS in AR(1) Bandits}
\label{sec:ps_ar1}
We consider a class of Gaussian bandits, which we refer to as AR(1) bandits, where each bandit is determined by a sequence $\{\alpha_{t, a}\}_{t = 1}^{+\infty}$ that independently transitions according to a first-order autoregressive (AR(1)) process. 
We study AR(1) bandits because AR(1) processes have been commonly used to model \add{time-varying} processes in fields such as nature and economics. {Similar formulations of non-stationary bandits have been examined by} 
\cite{bacchiocchi2022autoregressive, chen2024non, gupta2011thompson, gaussianar1, kuhn2015wireless,
slivkins2008adapting}. 

\begin{example} [\bf{AR(1) Bandit}]
\label{ex:autoregressive} 
In an AR(1) bandit, each reward $R_{t,a}$ is  distributed according to a Gaussian distribution with a random mean $\theta_{t,a} = \alpha_{t,a}$ and a deterministic variance $\sigma_a^2$. Each realized reward can be interpreted as a sum $R_{t,a} = \theta_{t,a} + Z_{t,a}$, where $Z_{t,a}$ is independent zero-mean noise with deterministic variance $\sigma_a^2$. The variable $\alpha_{t, a}$ \remove{changes}{evolves} over time\remove{, evolving} according to
$$\alpha_{t+1, a} = (1 - \gamma_a) c_a + \gamma_a \alpha_{t,a} + W_{t+1,a},$$ 
for each action $a \in \actions$. The coefficients $c_a$ and 
$\gamma_a$ are deterministic, and each takes value in $\mathbb{R}$ and 
$[0, 1]$, respectively; $W_{t+1,a}$ is independent zero-mean Gaussian noise with deterministic variance $\delta_a^2$, with $\delta_a \in \mathbb{R}_{+}$. When $\gamma_a = 1$, we require $\delta_a = 0$ and $\theta_{1,a}$ to be Gaussian. We assume that the sequence  $\{\alpha_{t,a}\}_{t=1}^{+\infty}$  is in steady-state: when $\gamma_a < 1$, this steady-state distribution is
$\mathcal{N}(c_a,\delta_a^2 / (1-\gamma_a^2))$. 
\end{example}

Note that the formulation of AR(1) bandits accommodates stationary Gaussian bandits with independent actions as a special case. Specifically, if we let $\gamma_a = 1$ and $\delta_a = 0$ for all $a \in \actions$, then $\alpha_{t+1, a} = \alpha_{1,a}$ for all $t \in \mathbb{Z}_{+}$ and $a \in \actions$.
Since $\alpha_{1,a}$ is a Gaussian random variable for all $a \in \actions$, 
we recover a stationary Gaussian bandit with independent actions. We can model any stationary Gaussian bandit with independent actions using an AR(1) bandit with 
$\gamma_a = 1$ and $\delta_a = 0$ for $a \in \actions$ and 
suitably-chosen $\alpha_{1,a}$ for each $a \in \actions$.

\remove{When interacting with an AR(1) bandit, we assume that an agent knows a priori 
$c_a$, $\gamma_a$, $\delta_a$, $q_a$, $\Pr(\theta_{0,a} \in \cdot)$, and $\sigma_a$ for all $a \in \actions$. 
In fact, this assumption is not necessary: an agent can learn these parameters. In other words, an agent is able to estimate these parameters quite accurately after a finite number of timesteps that is large enough. Since we are interested in the performance of an agent in the regime $T \rightarrow +\infty$, the performance of the agent in the learning phase of the first finite number of timesteps is irrelevant. Therefore, we can focus on investigating the performance of an agent after it learns these parameters. Without loss of generality, we can assume that the parameters are known to the agent a priori. Similar assumptions appear in \citep{slivkins2008adapting} and \citep{pmlr-v31-mellor13a} and techniques to estimate such parameters based on historical data have been discussed in \citep{wilson2010bayesian, turner2009adaptive}. }

{When interacting with an AR(1) bandit, we assume that the agent knows {\emph{a priori}} $c_a$, $\gamma_a$, $\delta_a$, $\Pr(\theta_{1,a} \in \cdot)$, and $\sigma_a$ for all $a \in \actions$. However, this assumption is not strictly necessary if these parameters can be accurately estimated from past data. Indeed, parameter estimation remains an important topic, and prior work \citep{wilson2010bayesian, turner2009adaptive} have developed techniques for learning parameters from historical data, which could be applied in this setting. Since we are primarily interested in the agent’s performance as $T \rightarrow +\infty$, we can focus on the agent’s performance after it learns these parameters. For the purpose of our analysis, we assume that these parameters are known a priori, while leaving further exploration of hyperparameter learning for future work. Similar assumptions regarding known parameters have also appeared in \citep{slivkins2008adapting} and \citep{pmlr-v31-mellor13a}. }

\subsubsection{Efficient Implementation of PS}
\label{sec:ar1_implementation}
While we are interested in developing efficient procedures to execute PS in an AR(1) bandit, as a stepping stone, we first focus on implementing TS that takes $\chi_t = \theta_t$ as its learning target. 
In an AR(1) bandit, TS  
samples at each timestep $\hat{\theta}_{t}^{\pi_{\mathrm{TS}}}$ from $\Pr(\theta_{t} \in \cdot | H_{t-1}^{\pi_{\mathrm{TS}}})$, 
and selects an action that maximizes $\hat{\theta}_{t,a}^{\pi_{\mathrm{TS}}}$. 
Recall that in an AR(1) bandit,  $\Pr(\theta_{1} \in \cdot)$ is Gaussian. When action $a$ is selected at timestep $t$, the agent observes $R_{t, a} \sim \mathcal{N}(\theta_{t,a}, \sigma_a^2)$, where $\sigma_a$ is deterministic and known. Therefore, $\Pr(\theta_t \in \cdot | H_{t-1}^{\pi_{\mathrm{TS}}})$ is Gaussian. We use $\mu_{t}^{\pi_{\mathrm{TS}}}$ and $\Sigma_{t}^{\pi_{\mathrm{TS}}}$ to denote the mean and covariance of it; they can be derived using Kalman filter. 
Algorithm~\ref{alg:Thompson-sampling-ar1} provides an implementation of TS in an AR(1) bandit; for the sake of simplicity, we drop the superscript $\pi_{\mathrm{TS}}$.

{We develop an efficient implementation of PS in Algorithm~2, which circumvents explicitly sampling future rewards $R_{t+1:\infty}$. This implementation is based on the following key observation from Section~\ref{sec:ps_information}: Steps~2 and~3 of Algorithm~\ref{alg:predictive-sampling}—sampling from the distribution of future rewards $R_{t+1:\infty}$ and deriving $\hat{\theta}_{t}$ based on the sample—are equivalent to sampling $\hat{\theta}_{t}$ from 
\begin{equation}
    \Pr(\E[R_{t}|H_{t-1}, R_{t+1:\infty}] \in \cdot | H_{t-1})= \Pr(\E[\theta_{t}|H_{t-1}, R_{t+1:\infty}] \in \cdot | H_{t-1}). 
\label{eq:sample_ps}
\end{equation}
Since rewards are jointly Gaussian in AR(1) bandits, the above distribution is Gaussian. We use $\tilde{\mu}_{t}$ and $\tilde{\Sigma}_{t}$ to denote its mean and covariance. Algorithm~\ref{alg:predictive-sampling-ar1} presents an implementation of PS in an AR(1) bandit setting, in which we sample $\hat{\theta}_t$ from $\mathcal{N}(\tilde{\mu}_{t}, \tilde{\Sigma}_{t})$ instead of explicitly sampling future rewards $R_{t+1:\infty}$ and computing $\hat{\theta}_{t}$ from them.} 

\begin{figure}[!ht]
\centering
\begin{minipage}{0.48\textwidth}
\begin{algorithm}[H]
\caption{PS in an AR(1) bandit}
\label{alg:predictive-sampling-ar1}
\For{$t = 1, 2, \ldots, T$}{
\textbf{sample}: $\hat{\theta}_t \sim \mathcal{N}(\tilde{\mu}_{t}, \tilde{\Sigma}_{t})$ \\
\textbf{execute}: $A_t \in \argmax_{a \in \actions} \hat{\theta}_{t, a}$ \\
\textbf{observe}: $R_{t, A_t}$ \\
\textbf{update}:
$\tilde{\mu}_{t+1} \leftarrow \E[\theta_{t+1}|H_{t}]$\\
\qquad \ \ \ \ \ \ $\tilde{\Sigma}_{t+1} \leftarrow \V(\E[\theta_{t+1}|H_t, R_{t+2:\infty}] | H_{t})$
}
\end{algorithm}
\end{minipage}
\begin{minipage}{0.48\textwidth}
\begin{algorithm}[H]
\caption{TS in an AR(1) bandit}
\label{alg:Thompson-sampling-ar1}
\For{$t = 1, 2, \ldots, T$}{
\textbf{sample}: $\hat{\theta}_{t}^{} \sim \mathcal{N}(\mu_{t}^{}, \Sigma_{t}^{})$ \\
\textbf{execute}: $A_t \in \argmax_{a \in \actions} \hat{\theta}_{t, a}$ \\
\textbf{observe}: $R_{t,A_t}$ \\
\textbf{update}: 
$\mu_{t+1}\! \leftarrow\! \E[\theta_{t+1} | H_{t}]$ \\ \qquad \qquad $\Sigma_{t+1}\! \leftarrow\! \V(\theta_{t+1} | H_{t})$
}
\end{algorithm}
\end{minipage}
\label{figure:implementations}
\end{figure}

{Deriving $\tilde{\mu}_{t+1}$ and $\tilde{\Sigma}_{t+1}$ in Step~5 of Algorithm~\ref{alg:predictive-sampling-ar1} is the most technically involved part of the implementation. Below, we present a straightforward and efficient procedure for computing them. }Observe that the distribution $\Pr(\theta_t \in \cdot | H_{t-1})$ is Gaussian. We use $\mu_{t}$ and $\Sigma_{t}$ to denote the mean and covariance of it, and  they can be derived analytically using Kalman filter. The values of $\tilde{\mu}_{t}$ and $\tilde{\Sigma}_{t}$ can be derived using $\mu_t$ and $\Sigma_t$. Specifically, because both $\Sigma_{t}$ and $\tilde{\Sigma}_t$ are diagonal, we use $\sigma_{t,a}^2$ and $\tilde{\sigma}_{t,a}^2$ to denote each entry along their diagonals, respectively. The following result establishes the exact analytical forms for $\tilde{\mu}_{t,a}$ and $\tilde{\sigma}^2_{t,a}$, as function of $\mu_{t,a}$ and $\sigma^2_{t,a}$. The proof is presented in Appendix~\ref{appendix:ar1}. 

\begin{restatable}{proposition}{arderivation}
\label{proposition:ar1}
In an AR(1) bandit, for all $t \in \mathbb{Z}_+$ and $a \in \actions$, conditioned on $H_{t-1}$, $\hat{\theta}_{t, a}$ is Gaussian with mean and variance 
\begin{align*}
&\ \tilde{\mu}_{t,a} = \mu_{t, a} \text{ and } 
\tilde{\sigma}_{t,a}^2 = \frac{\gamma_a^2 \sigma_{t, a}^4}{\gamma_a^2 \sigma_{t, a}^2 + x_a^*},
\end{align*}
where 
$x_a^* = \frac{1}{2} \left(\delta_a^2 + \sigma_a^2 - \gamma_a^2 \sigma_a^2 + \sqrt{(\delta_a^2 + \sigma_a^2 - \gamma_a^2 \sigma_a^2)^2 + 4 \gamma_a^2 \delta_a^2 \sigma_a^2}\right)$.
\end{restatable}

{Both PS and TS are sampling algorithms, 
each begins by sampling an estimate, $\hat{\theta}_t$ for PS and $\hat{\theta}_t^{\mathrm{TS}}$ for TS, followed by the selection of an action that maximizes the sample.

The key distinction between PS and TS lies in their respective sampling distributions, which of PS is defined in \eqref{eq:sample_ps}. According to Proposition~\ref{proposition:ar1}, for all $t \in \mathbb{Z}_+$ and $a \in \actions$, it holds that $\tilde{\mu}_{t,a} = \mu_{t,a}$. This indicates that when presented with the same history $H_{t-1}$, the sampling distributions of PS and TS share the same mean. In this scenario, the ratio of the variances of the sampling distributions is given by  $\frac{\tilde{\sigma}^2_{t,a}}{\sigma^2_{t,a}}
= \frac{\gamma_a^2 \sigma_{t,a}^2}{\gamma_a^2 \sigma_{t,a}^2 + x_a^*} \in [0, 1]
$. The ratio increases from $0$ to $1$ as $\gamma_a$ increases from $0$ to $1$. This implies that as information becomes more durable, 
$\gamma_a$ increases, so does the variance of the sampling distribution of PS, and thus PS explores more. This suggests that PS typically explores less compared to TS, and how much it explores varies according to the durability of information. 
}

\add{Different from the coin-tossing games and its generalization of modulated Bernoulli bandits, the AR(1) bandit exhibit continuous dynamics. While durability is less explicit than in the earlier examples, it is captured by the parameter $\gamma$. As we show, each $\gamma_a$ influences the variance of the sampling distribution and thereby the algorithm’s behavior.}

\subsubsection{Experiments}
\label{sec:ar1_experiments}
We next conduct experiments in a sequence of AR(1) bandits where the actions are associated with varying degrees of information durability and compare the performance of PS against that of TS. Note that in an AR(1) bandit, $\gamma_a$ captures the durability of information associated with action $a \in \actions$, with $\gamma_a = 1$ indicating that the information durability is infinite and $\gamma_a = 0$ indicating that the information durability is zero. 
Then we conduct experiments in bandits with varying $\gamma_a$ for $a \in \actions$. 

In particular, we let $\actions = \{1, 2\}$, the stationary distribution of each arm's mean reward be $\mathcal{N}(0, 1)$, and $\gamma_1, \gamma_2 \in \{0.1, 0.3, 0.5, 0.7, 0.9\}$. This defines a range of two-armed AR(1) bandits where each mean reward distribution is standard normal, and $\gamma_1$ and $\gamma_2$ each varies in a discrete grid in $(0, 1)$. By symmetry in actions, we examine bandits where $\gamma_1 \in \{0.1, 0.3, 0.5\}$
and $\gamma_2 \in \{0.1, 0.3, 0.5, 0.7, 0.9\}$\remove{, respectively}.

\paragraph{PS Outperforms TS Across Bandits}
Figure~\ref{fig:ar1_scatter} plots the {average reward} collected by PS and that collected by TS over a long duration of $T = 1000$ timesteps, which serve as good approximations to the long-run average rewards collected by the two agents. 
The plot shows that PS consistently outperforms TS, regardless of information durability associated with the actions. 
This provides additional evidence, complementing our theoretical analyses and examples, to support that PS has advantage over TS in non-stationary bandits. 

The plot also suggests that both PS and TS perform better in bandits with a larger value of $\gamma_1$ or $\gamma_2$. 
This is consistent with our intuition that any given agent tends to perform better in bandits with better information durability because they can put the information learned so far to better use for longer periods of time. 

\paragraph{PS Outperforms TS Across Time}  
Now that we have examined the long-run performance of PS, we next investigate if PS sacrifices its short-term performance for long-term benefits. In particular, we focus on one example of the aforementioned AR(1) bandits, with $\gamma_1 = 0.1$ and $\gamma_2 = 0.9$ and plot the mean rewards across a number of timesteps. 
Figure~\ref{fig:ar1_example} plots the average reward collected by PS and that collected by TS over $t \in \{1, 2, ... , 200\}$ timesteps, with the error bars representing 95\% confidence intervals. We observe that PS consistently outperforms TS across time.  

The phenomenon that PS outperforms TS across time can be observed beyond this example. 
We present additional plots in Appendix~\ref{appendix:ar1_experiments} that illustrate this. The plots suggest that the phenomenon is persistent across a diverse set of AR(1) bandit instances.

\begin{figure} [!ht]
\centering
\begin{subfigure}[b]{0.45\textwidth}
\centering
\includegraphics[width=\textwidth]{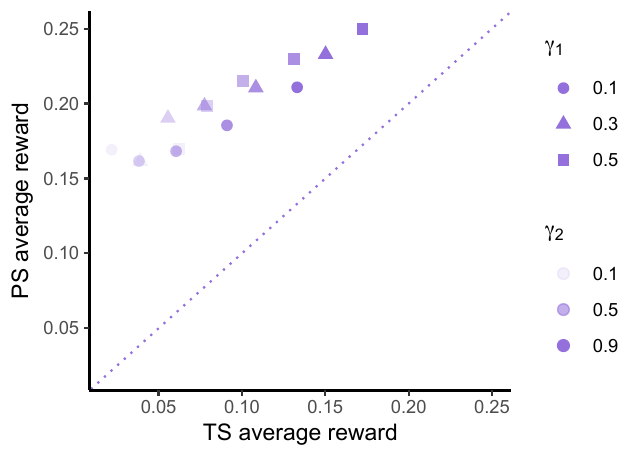}
\caption{Average rewards collected over $T = 1000$ timesteps where $\gamma_1 \in \{0.1, 0.3, 0.5\}$, $\gamma_2 \in \{0.1, 0.3, 0.5, 0.7, 0.9\}$}
\label{fig:ar1_scatter}
\end{subfigure}
\hfill
\begin{subfigure}[b]{0.45\textwidth}
\centering
\includegraphics[width=\textwidth]{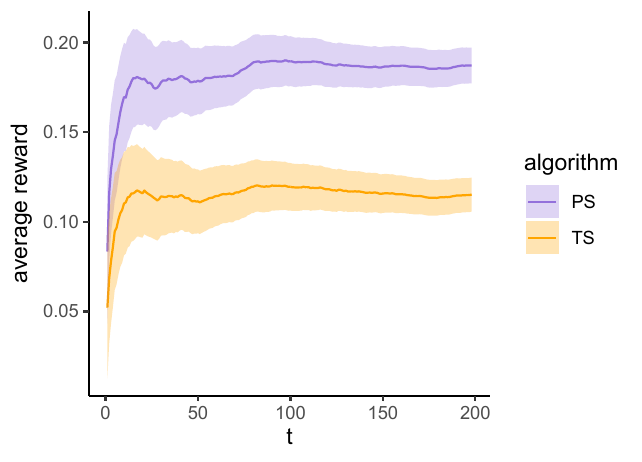}
\caption{An example: the average rewards collected over $t \in \{1, ... , 200\}$ timesteps in an environment where $\gamma_1 = 0.1$, $\gamma_2 = 0.9$}
\label{fig:ar1_example}
\end{subfigure}
\caption{The average rewards collected by PS and that collected by TS in AR(1) bandits}
\label{figure:ar1_plot}
\end{figure}

\subsection{Comparison with Other Algorithms}
\label{section:algorithms}
We conduct experiments to compare PS with non-stationary bandit algorithms 
beyond TS. 
\subsubsection{Algorithms and Environments}
A large set of algorithms 
\citep{besbes2019optimal, besson2019generalized, cheung2019learning, garivier2008upper, 9194367, gupta2011thompson, hartland2006multi, kocsis2006discounted, pmlr-v31-mellor13a, raj2017taming, trovo2020sliding, viappiani2013thompson} designed for non-stationary bandit learning  
focuses on making better inference about current mean reward; given what is inferred, these algorithms usually apply action-selection schemes that are designed for stationary bandits, such as TS, upper-confidence-bound methods \citep{lai1985asymptotically}, and the exponential-weight algorithms \citep{auer2002nonstochastic, freund1997decision}. In this sense, one can think of these algorithms as taking the current mean reward as the learning target; in particular, to make it more concrete, when applied to the coin-tossing environments of Section~\ref{sec:motivating_discussion}, this learning target corresponds to the coin biases. Therefore, these algorithms, with TS and its existing variants as special cases, do not intelligently account for the information durability when selecting actions. 
\remove{, and we believe that PS has advantages over these algorithms in non-stationary bandit learning.}

\paragraph{Algorithms} 
We conduct experiments with a representative subset of these non-stationary bandit learning algorithms. 
Since there are three popular heuristics on making inference on current mean reward, i.e., using a fixed-length sliding-window, weighing data by recency, and periodic restarts, we choose one algorithm focusing on each of the three heuristics.  
In addition, we take a naive TS agent, who executes TS while pretending that the bandit is stationary, as a baseline. Below we briefly describe each of the four representative algorithms with which we conduct experiments. 
\begin{itemize}
\item  {Rexp3} \citep{besbes2019optimal} uses Exp3 as action-selection subroutine and restarts it periodically. 
\item  {Discounted UCB} \citep{garivier2008upper} uses UCB1 as subroutine and discounts the effect of past rewards on estimating current reward by weighing past data according to recency. 
\item  {Sliding-window UCB} \citep{garivier2008upper} maintains a sliding-window of fixed size and uses UCB1 as a subroutine. 
\item  {Naive TS} pretends that the bandit is stationary and proceeds with inference. 
\end{itemize}
{In the plots in this section, the legends use `dUCB' for discounted UCB, `sUCB' for sliding-window UCB, and `nTS' for naive TS.}

\paragraph{Environments} We next introduce the class of bandits in which we conduct the experiments and additional details in implementing the algorithms. 
Because past work that introduced Rexp3, discounted UCB, and sliding-window UCB conduct experiments in bandits \remove{where rewards are bounded}{with bounded rewards}, we restrict our attention to such bandits. 
In addition, to analyze how the information durability affect the performance of the agents, we also would like to conduct experiments in bandits where the information durabilities are determined by particular environment parameters, the adjustments of which adjust the information durabilities. 
To accommodate both needs, we design a set of bandits that differ from the AR(1) bandits only in that the rewards are truncated to $[0, 1]$. 

In implementating the algorithms, the parameters of Rexp3 are chosen according to Theorem~2 of \citep{besbes2019optimal}, where the ``variation budget'' is assumed to be known in advance for each simulation; the parameters in discounted UCB and sliding-window UCB are chosen according to Remark 3 and Remark 9 of \citep{garivier2008upper}, respectively, where ``the number of breakpoints'' is assumed to be known in advance; {PS} pretends that the rewards are not truncated.

\subsubsection{Experiments} 
\label{sec:ar1_allalgorithms_experiments}
We conduct two sets of experiments. 
{The first set of experiments show that PS explores less in a non-stationary environment, and outperforms other algorithms.} \remove{The first one in}{We consider} a bandit where $\actions = \{1, 2\}$, $c_1 = c_2 = 0.5$, $\gamma_1 = \gamma_2 = 0.85$, $\delta_a = 0.15(1 - \gamma_a^2)$ for $a \in \actions$, and $\sigma_1 = \sigma_2 = 0.1$. 
This describes a bandit with two symmetric actions.

Figure~\ref{fig:comparison_1_reward} plots the average reward collected by the agents over $t \in \{1, 2, ... , 2000\}$ timesteps, with the error bars representing 95\% confidence intervals. The results 
suggest that a majority of the non-stationary bandit learning algorithms outperforms naive TS, and PS outperforms all others. 
This is consistent with our theoretical results that PS
\remove{selects actions intelligently accounting for information durability, and as such} has advantages in non-stationary bandits. 

Figure~\ref{fig:comparison_1_action} plots the average frequency of selecting action $2$ over $t 
\in \{1, 2, ... , 2000\}$ timesteps. 
As expected, since the actions are symmetric, all agents select each action half of the time in the long run. 
{
To demonstrate when PS selects actions differently, we visualize the amount of exploration over $t \in \{1, 2, ... , 2000\}$ timesteps in Figure~\ref{fig:amount_exploration_1}. Exploration occurs when the selected action does not maximize $\mathbb{E}[R_{t,a}|H_t]$, hence, we quantify exploration by the frequency of selecting such actions. Figure~\ref{fig:amount_exploration_1} illustrates that in this bandit, PS explores significantly less compared to other policies. This observation aligns with our intuition that PS intelligently reduces exploration in non-stationary environments.
}
\begin{figure} [!ht]
\centering
\begin{subfigure}[b]{0.32\textwidth}
\centering
\includegraphics[width=\textwidth]{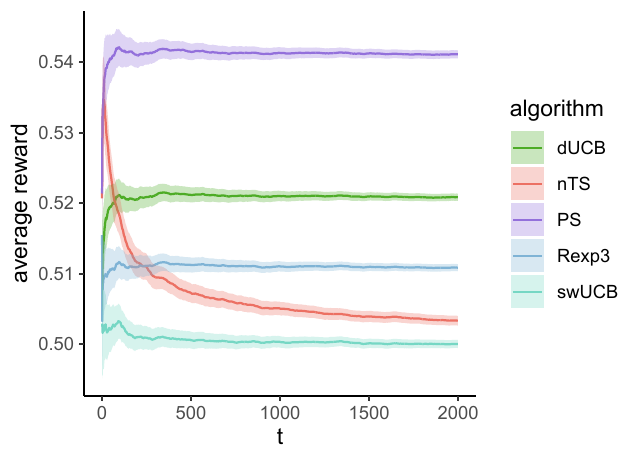}
\caption{Average rewards collected by each agent}
\label{fig:comparison_1_reward}
\end{subfigure}
\hfill
\begin{subfigure}[b]{0.32\textwidth}
\centering
\includegraphics[width=\textwidth]{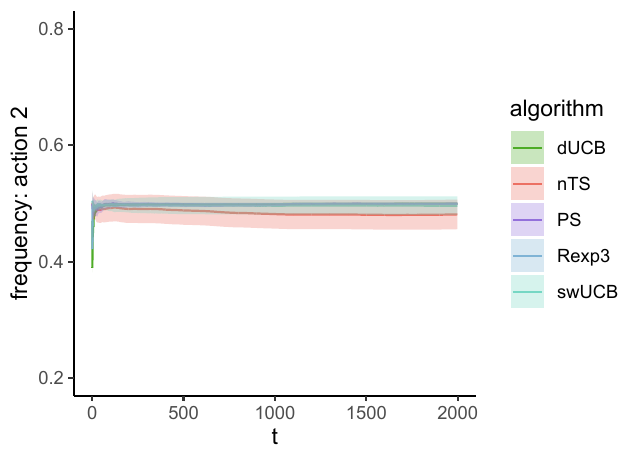}
\caption{Frequencies of selecting action $2$ }
\label{fig:comparison_1_action}
\end{subfigure}
\hfill
\begin{subfigure}[b]{0.31\textwidth}
\centering
\includegraphics[width=\textwidth]{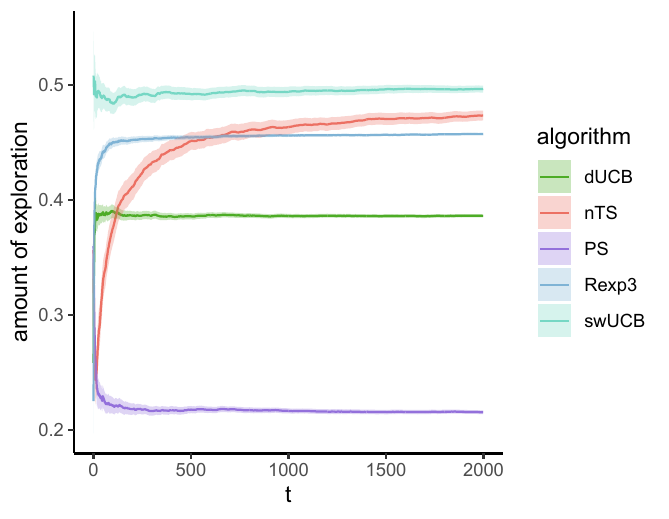}
\caption{Frequencies of engaging in {(amount of)} exploration 
}
\label{fig:amount_exploration_1}
\end{subfigure}
\caption{PS and representative algorithms{, i.e.,  discounted UCB (`dUCB'), sliding-window UCB (`sUCB'), Rexp3 (`Rexp3'), and naive TS (`nTS'),} in a bandit with two actions, where $c_1 = c_2 = 0.5$, $\gamma_1 = \gamma_2 = 0.85$, $\delta_a^2 = 0.15^2(1 - \gamma_a^2)$ for $a \in \actions$, and $\sigma_1 = \sigma_2 = 0.1$}
\label{fig:ar_comparison}
\end{figure}

We conduct our second set of experiments 
in a {slightly more complicated environment}\remove{bandit} where $\actions = \{1, 2\}$, $c_1 = 0.65, c_2 = 0.55$, $\gamma_1 = 0.1, \gamma_2 = 0.99$, $\delta_a = 0.15(1 - \gamma_a^2)$ for $a \in \actions$, and $\sigma_1 = \sigma_2 = 0.1$. This corresponds to a bandit with two asymmetric actions: action $1$ is associated with a mean reward that has a larger mean $c_1 = 0.65$ and changes very quickly, with $\gamma_1 = 0.1$; action $2$ is associated with a mean reward that has a smaller mean $c_2 = 0.55$ and changes very slowly, with$\gamma_2 = 0.99$. 

\remove{Returning to the averaged results,} Figure~\ref{fig:comparison_2_reward} plots the average rewards accumulated by the agents, with the error bars representing 95\% confidence intervals. 
The results reveal that PS accumulates more rewards compared with other agents. 
This {again} suggests that PS has advantages in non-stationary bandits\remove{because it intelligently account for the durability of information and as a result selects action $2$ less compared to other algorithms}.  {Figure~\ref{fig:amount_exploration_2} plots the amount of exploration. As expected, the plot indicates that naive TS explores the least, and PS engages in relatively less exploration compared to other non-stationary algorithms. }
Figure~\ref{fig:comparison_2_action} displays the average action selection frequency. 
{This time, with asymmetric actions, we observe different frequencies of selecting the same action.}
As expected, we observe that the average frequency of selecting action $2$ by naive TS converges to zero, because $c_2 < c_1$. 
 { 
The reduced exploration by PS compared to other non-stationary bandit learning algorithms implies that PS selects the action with lower mean reward estimate less frequently. In most cases, this action corresponds to action $2$. So PS selects action $2$ less frequently compared to other non-stationary algorithms, as corroborated by the consistent behavior observed in Figure~\ref{fig:comparison_2_action}. 
}

{To gain insight into when PS selects actions differently from other algorithms and why this is advantageous, let us examine a specific realization of the sample paths of $\theta_{t}$. Figure~\ref{fig:theta_500} displays $\theta_{1,t}$ and $\theta_{2,t}$ over $t \in \{1, 2, ... , 2000\}$ timesteps, along with the long-term average of $\theta_{1,t}$, which is 0.65 {(denoted as ‘benchmark’ in the legend)}. 
Figure~\ref{fig:action_freq_500} plots the frequencies of selecting action $2$. 
We observe that, in line with our earlier discussions, the naive TS agent progressively decreases its selection of action $2$ over time; the non-stationary bandit learning algorithms other than PS consistently select action $2$ with high frequency, irrespective of the value of $\theta_{t,2}$; PS selects action $2$ less frequently compared to other algorithms. Furthermore, Figure~\ref{fig:action_freq_500} indicates that PS opts for action $2$ more frequently when $\theta_{t,2}$ is relatively large, which explains the superior performance of PS.
}

\begin{figure} [!ht]
\centering
\begin{subfigure}[b]{0.32\textwidth}
\centering
\includegraphics[width=\textwidth]{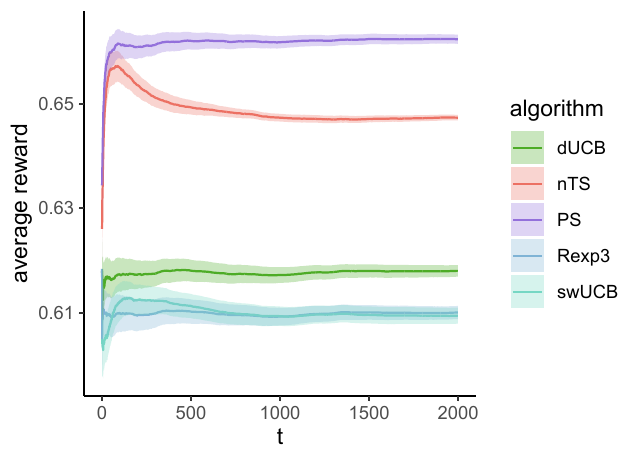}
\caption{Average rewards collected by each agent}
\label{fig:comparison_2_reward}
\end{subfigure}
\hfill
\begin{subfigure}[b]{0.31\textwidth}
\centering
\includegraphics[width=\textwidth]{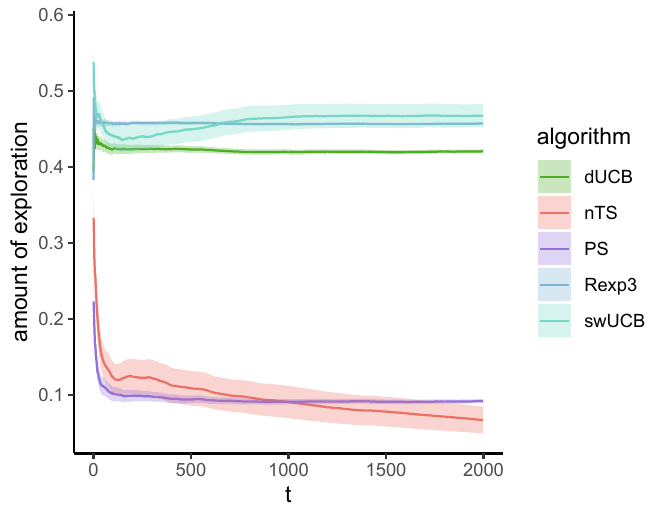}
\caption{Frequencies of engaging in exploration}
\label{fig:amount_exploration_2}
\end{subfigure}
\hfill
\begin{subfigure}[b]{0.31\textwidth}
\centering
\includegraphics[width=\textwidth]{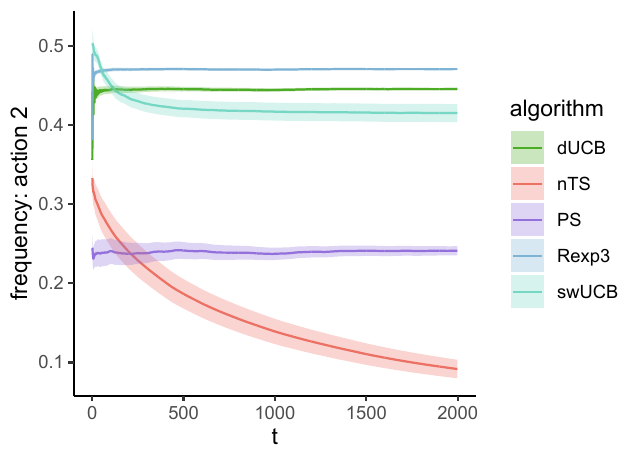}\caption{Frequencies of selecting action $2$}
\label{fig:comparison_2_action}
\end{subfigure}
\caption{PS and representative algorithms{, i.e.,  discounted UCB (`dUCB'), sliding-window UCB (`sUCB'), Rexp3 (`Rexp3'), and naive TS (`nTS'),} in a bandit with two actions, where $c_1 = 0.65, c_2 = 0.55$, $\gamma_1 = 0.1, \gamma_2 = 0.99$, $\delta_a^2 = 0.15^2(1 - \gamma_a^2)$ for $a \in \actions$, and $\sigma_1 = \sigma_2 = 0.1$}
\label{fig:ar_comparison_2}
\end{figure}
{  
\begin{figure} [!ht]
\centering
\begin{subfigure}[b]{0.45\textwidth}
\centering
\includegraphics[width=\textwidth]{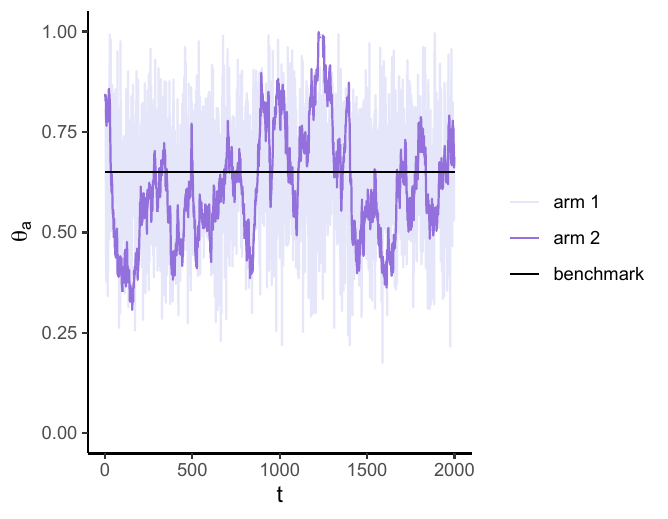}\caption{$\theta_a$ for $a \in \{1, 2\}$
}
\label{fig:theta_500}
\end{subfigure}
\hfill
\begin{subfigure}[b]{0.45\textwidth}
\centering
\includegraphics[width=\textwidth]{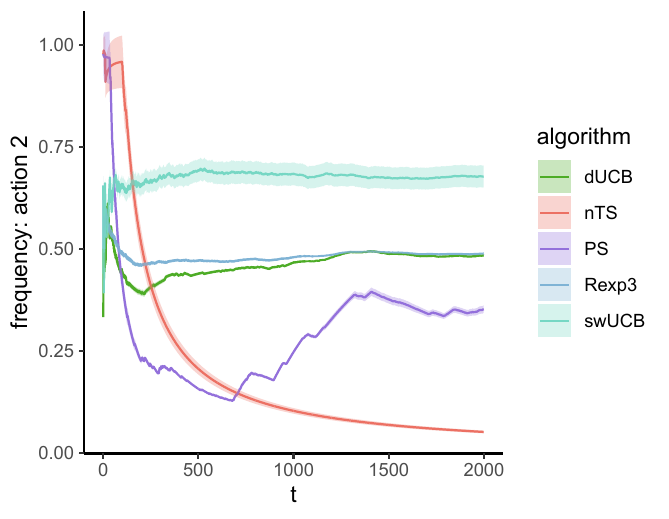}
\caption{Frequencies of selecting action $2$}
\label{fig:action_freq_500}
\end{subfigure}
\caption{PS and representative algorithms{, i.e.,  discounted UCB (`dUCB'), sliding-window UCB (`sUCB'), Rexp3 (`Rexp3'), and naive TS (`nTS'),} in a bandit with two actions, where $c_1 = 0.65, c_2 = 0.55$, $\gamma_1 = 0.1, \gamma_2 = 0.99$, $\delta_a^2 = 0.15^2(1 - \gamma_a^2)$ for $a \in \actions$, and $\sigma_1 = \sigma_2 = 0.1${---a specific realization}}
\label{fig:instance}
\end{figure}
}

%% file: arxiv/appendix.tex
\section{Probabilistic Framework}
\label{sec:probability}
Probability theory emerges from an intuitive set of axioms, and this paper builds on that foundation.  Statements and arguments we present have precise meaning within the framework of probability theory.  However, we often leave out measure-theoretic formalities for the sake of readability.  It should be easy for a mathematically-oriented reader to fill in these gaps.

We will define all random quantities with respect to a probability space $(\Omega, \mathcal{F}, \Pr)$.  The probability of an event $F \in \mathcal{F}$ is denoted by $\Pr(F)$.  For any events $F, G \in \mathcal{F}$ with $\Pr(G) > 0$, the probability of $F$ conditioned on $G$ is denoted by $\Pr(F | G)$.

A random variable is a function with the set of outcomes $\Omega$ as its domain.  For any random variable $Z$, $\Pr(Z \in \mathcal{Z})$ denotes the probability of the event that $Z$ lies within a set $\mathcal{Z}$.  The probability $\Pr(F | Z = z)$ is of the event $F$ conditioned on the event $Z = z$.  When $Z$ takes values in $\R$ and has a density $p_Z$, though $\Pr(Z=z)=0$ for all $z$, conditional probabilities $\Pr(F| Z=z)$ are well-defined and denoted by $\Pr(F | Z = z)$.
For fixed $F$, this is a function of $z$.  We denote the value, evaluated at $z=Z$, by $\Pr(F | Z)$, which is itself a random variable.  Even when $\Pr(F | Z = z)$ is ill-defined for some $z$, $\Pr(F | Z)$ is well-defined because problematic events occur with zero probability.  

For each possible realization $z$, the probability $\Pr(Z=z)$ that $Z = z$ is a function of $z$.  We denote the value of this function evaluated at $Z$ by $\Pr(Z)$.  Note that $\Pr(Z)$ is itself a random variable because it depends on $Z$.  For random variables $Y$ and $Z$ and possible realizations $y$ and $z$, the probability $\Pr(Y=y|Z=z)$ that $Y=y$ conditioned on $Z=z$ is a function of $(y, z)$.  Evaluating this function at $(Y,Z)$ yields a random variable, which we denote by $\Pr(Y|Z)$.

We denote independence of random variables $X$ and $Y$ by $X \perp Y$ and conditional independence, conditioned on another random variable $Z$, by $X \perp Y | Z$.

\section{Concepts, Notations, and Relations}
\label{appendix:information}
We review some standard information-theoretic concepts and associated notations, and introduce a change-of-measure notation in this section. 
\subsection{Information-Theoretic Concepts, Notations, and Relations}
\remove{A central concept is the entropy $\H(X)$, which quantifies the information content or, equivalently, the uncertainty of a random variable $X$.  For a random variable $X$ that takes values in a countable set $\mathcal{X}$, we will define the entropy to be $\H(X) = -\E[\ln \Pr(X)]$, with a convention that $0 \ln 0 = 0$.  Note that we are defining entropy here using the natural rather than binary logarithm.  As such, our notion of entropy can be interpreted as the expected number of nats---as opposed to bits---required to identify $X$.  The realized conditional entropy $\H(X | Y = y)$ quantifies the uncertainty remaining after observing $Y=y$.  If $Y$ takes on values in a countable set $\mathcal{Y}$ then $\H(X | Y = y) = - \E[\ln \Pr(X | Y) | Y = y]$.
This can be viewed as a function $f(y)$ of $y$, and we write the random variable $f(Y)$ as $\H(X|Y = Y)$.  The conditional entropy $\H(X | Y)$ is its expectation
$\H(X | Y) = \E[\H(X|Y=Y)]$.

The mutual information $\I(X; Y) = \H(X) - \H(X|Y)$ quantifies information common to random variables $X$ and $Y$, or equivalently, the information about $Y$ required to identify $X$.  If $Z$ is a random variable taking on values in a countable set $\mathcal{Z}$ then the realized conditional mutual information $\I(X; Y| Z=z)$ quantifies remaining common information after observing $Z = z$, defined by $\I(X; Y| Z=z) = \H(X|Z=z) - \H(X|Y, Z=z)$.
The conditional mutual information $\I(X; Y | Z)$ is its expectation $\I(X; Y | Z) = \E[\I(X; Y | Z = Z)]$.

For random variables $X$ and $Y$ taking on values in (possibly uncountable) sets $\mathcal{X}$ and $\mathcal{Y}$, mutual information is defined by
$\I(X; Y) = \sup_{f \in \mathcal{F}_{\text{finite}}, g \in \mathcal{G}_{\text{finite}}} \I(f(X); g(Y))$, 
where $\mathcal{F}_{\text{finite}}$ and $\mathcal{G}_{\text{finite}}$ are the sets of functions mapping $\mathcal{X}$ and $\mathcal{Y}$ to finite ranges.  Specializing to the case where $\mathcal{X}$ and $\mathcal{Y}$ are countable recovers the previous definition.  The generalized notion of entropy is then given by $\H(X) = \I(X; X)$.  Conditional counterparts to mutual information and entropy can be defined in a manner similar to the countable case.}

We will make use of
KL-divergence as measures of difference between distributions.
We denote KL-divergence by
$$\KL(P \| P') = \int P(dx) \ln\frac{dP}{dP'}(x).$$
Gibbs' inequality asserts that $\KL(P\|P') \geq 0$, with equality if and only if $P$ and $P'$ agree almost everywhere with respect to $P$.

The following result is established by Theorem 5.2.1 of \citep{gray2011entropy}.

\begin{lemma}[Variantal Form of the KL-Divergence]
\label{lemma:var_KL}
For any probability distribution $P$ and real-valued random variable $X$, both defined with respect to a measureable space $(\Omega', \mathbb{F}')$, let $\E_P[X] = \int_{x \in \Re} x P(dx)$.  For probability distributions $P$ and $P'$ on a measureable space $(\Omega', \mathbb{F}')$ such that $P$ is absolutely continuous with respect to $P'$, 
\begin{equation}
\label{eq:KL-variantal}
\KL(P\|P') = \sup_X (\E_P[X] - \ln \E_{P'}[\exp(X)]),
\end{equation}
where the supremum is taken over real-valued random variables on $(\Omega', \mathbb{F}')$ for which $\E_Q[\exp(X)] < \infty$.  
\end{lemma}

Mutual information and KL-divergence are intimately related.  For any probability measure $P(\cdot) = \Pr((X,Y) \in \cdot)$ over a product space $\mathcal{X} \times \mathcal{Y}$ and probability measure $P'$ generated via a product of marginals $P'(dx \times dy) = P(dx) P(dy)$, mutual information can be written in terms of KL-divergence:
\begin{equation}
\label{eq:mutual-information-marginal-distribution}
\I(X; Y) = \KL(P \| P').
\end{equation}
Further, the following lemma presents an alternative representation of mutual information. 
\begin{lemma} [KL-Divergence Representation of Mutual Information]
\label{lemma:KL_MI}
For any random variables $X$ and $Y$, 
\begin{equation}
\label{eq:mutual-information-conditional-distribution}
\I(X; Y) = \E[\KL(\Pr(Y \in \cdot|X) \| \Pr(Y \in \cdot))].
\end{equation}
\end{lemma}
In other words, the mutual information between $X$ and $Y$ is the KL-divergence between the distribution of $Y$ with and without conditioning on $X$.

Mutual information satisfies the chain rule and the data processing inequality. 
\begin{lemma} [Chain Rule for Mutual Information]
\label{lemma:chain_MI}
\begin{align*}
\I(X_1, X_2, ... , X_n; Y) = \sum_{i = 1}^n \I(X_i; Y| X_1, X_2, ... , X_{i-1}).
\end{align*}
\end{lemma}

\begin{lemma} [Data Processing Inequality for Mutual Information]
\label{lemma:dp_MI}
If $X$ and $Z$ are independent conditioned on $Y$, then 
\begin{align*}
\I(X; Y) \geq \I(X; Z). 
\end{align*}
\end{lemma}

The following lemma presents one useful property of mutual information. 
\begin{lemma}
\label{lemma:conditional_mi}
Let $A$, $B$, and $C$ be three random variables. If $A \perp C|B$ then 
\begin{align*}
 \I(A; B | C) \leq \I(A; B). 
\end{align*}
\end{lemma}

\begin{proof}
We prove for the case where $B$ has finite entropy. For the case where $B$ has infinite entropy, we use differential entropy instead of entropy in the analysis. 
\begin{align*}
  \I(A; B | C) = &\ \H(A | C) - \H(A | B, C)\\
  = &\ \H(A | C) - \H(A | B)\\
  \leq &\ \H(A) - \H(A | B) \\
  = &\ \I(A; B).  
\end{align*}
\end{proof}

\subsection{Change-of-Measure Notation}
\label{appendix:change-of-measure}
To characterize PS, let us define a change-of-measure notation. Consider random variables $X$ and $Y$ and a conditional probability $\Pr(Y \in \cdot | X = x)$ for all $x$ in the image of $X$. Let $f(x) \equiv \Pr(Y \in \cdot |X = x)$ and $g(x) \equiv \E[Y|X=x]$, for all $x$ in the image of $X$.  Given a random variable $Z$ with the same image as $X$, $f(Z)$ and $g(Z)$ are random variables. We use the notation $\Pr(Y \in \cdot |X \leftarrow Z)$ for $f(Z)$ and $\E[Y|X \leftarrow Z]$ for $g(Z)$. Note that we use the symbol $\leftarrow$ to distinguish a change of measure from conditioning on an event: in general, 
\begin{align*}
 \Pr \left(Y \in \cdot | X \leftarrow Z \right)
 \neq 
 \Pr(Y \in \cdot |X = Z),
\end{align*}
because the former represents a change of measure from $X$ to $Z$, whereas the latter represents the conditional distribution conditioning on the event $\{\omega: X = Z\}$. Similarly, in general, $\E \left[Y | X \leftarrow Z \right] \neq \E[Y |X = Z]$, and we need the former to represent a change of measure. 

{
\section{Derivation for the Coin-Tossing Games in Section~\ref{sec:motivating_discussion}}
\label{appendix:motivating_example}
First, we introduce notation that will be useful throughout this section. Our notation closely follows that introduced in Section~\ref{sec:bandits}, except that we omit the subscript $\mathrm{TS}$ for certain random variables, as it is clear from context that we are analyzing TS.

A history consists of a finite number of action-reward pairs, and we use $\mathcal{H}$ to denote the set of all histories. For all $t \in \mathbb{Z}_+$, let $A_t$ denote the action selected by TS at timestep $t$, and let $H_{t}$ represent the history generated after selecting $A_t$ and observing the corresponding reward.

\subsection{Two Coins}
\label{appendix:motivating_example_two}
Below, we show that in the coin-tossing game with two coins, introduced in Section~\ref{sec:random_coins}, TS accumulates payoffs at an expected rate of at most 75\textcent\ per timestep.

First, observe that for all $t \in \mathbb{Z}_+$, 
\begin{align*}
\mathrm{expected\ payoff\ at\ time\ } t = \mathbb{P}(p_{t,2} = 1 | A_t = 2)\mathbb{P}(A_t = 2) + 0.99\mathbb{P}(A_t = 1). 
\end{align*}

For all $t \in \mathbb{Z}_+$, the conditional probability 
\begin{align*}
\mathbb{P}(p_{t,2} = 1 | A_t = 2) = &\ \sum_{h \in \mathcal{H}} \mathbb{P}(p_{t,2} = 1 | H_{t-1} = h, A_t = 2) \mathbb{P}(H_{t-1} = h | A_t = 2)\\
= &\ \sum_{h \in \mathcal{H}} \mathbb{P}(p_{t,2} = 1 | H_{t-1} = h) \mathbb{P}(H_{t-1} = h | A_t = 2)\\
\leq &\ \max_{h \in \mathcal{H}} \mathbb{P}(p_{t,2} = 1 | H_{t-1} = h) \\
\leq &\ 1 - 0.99 \times \frac{1}{2} = 0.505. 
\end{align*}
where the second equality follows from the independence of $A_t$ and $p_{t,2}$ given $H_{t-1}$, 
and the second inequality follows from the replacement rate of $0.99$ and the fact that half of the coins in the bag have a bias of $0$. 

Therefore, combining the above equation and inequality, and note that the TS agent selects the second coin with probability at least $0.495$ (i.e., $\mathbb{P}(A_t = 2) \geq 0.495$), 
we have for all $t \in \mathbb{Z}_+$, 
\begin{align*}
\mathrm{expected\ payoff\ at\ time\ } t = 
&\ \mathbb{P}(p_{t,2} = 1 | A_t = 2)\mathbb{P}(A_t = 2) + 0.99\mathbb{P}(A_t = 1)\\
\leq &\ 0.505 \mathbb{P}(A_t = 2) + 0.99\mathbb{P}(A_t = 1)\\
\leq &\ 0.505 \times 0.495 + 0.99 \times (1 - 0.495) \approx 0.7499 < 0.75. 
\end{align*}
Therefore, TS accumulates an expected payoff of less than 75¢ per timestep.

\subsection{\texorpdfstring{$K$}{K} Coins}
\label{appendix:motivating_example_k}
Below, we show that in the coin-tossing game with $K$ coins, introduced in Section~\ref{sec:K_coins}, TS accumulates payoffs at an expected rate of at most 2\textcent\ per timestep when $K$ is sufficiently large. 

The proof closely follows the argument for the game with two coins. However, we present the full proof for completeness. 

First, observe that for all $t \in \mathbb{Z}_+$, 
\begin{align*}
\mathrm{expected\ payoff\ at\ time\ } t = \sum_{k = 2}^K \mathbb{P}(p_{t,k} = 1 | A_t = k)\mathbb{P}(A_t = k) + 0.99\mathbb{P}(A_t = 1). 
\end{align*}

For all $t \in \mathbb{Z}_+$ and $k \in \{2, ... , K\}$, 
the conditional probability 
\begin{align*}
\mathbb{P}(p_{t,k} = 1 | A_t = k) = &\ \sum_{h \in \mathcal{H}} \mathbb{P}(p_{t,k} = 1 | H_{t-1} = h, A_t = k) \mathbb{P}(H_{t-1} = h | A_t = k)\\
= &\ \sum_{h \in \mathcal{H}} \mathbb{P}(p_{t,k} = 1 | H_{t-1} = h) \mathbb{P}(H_{t-1} = h | A_t = k)\\
\leq &\ \max_{h \in \mathcal{H}} \mathbb{P}(p_{t,k} = 1 | H_{t-1} = h) \\
\leq &\ 1 - 0.99 \times 0.99 = 0.0199, 
\end{align*}
where the second equality follows from the independence of $A_t$ and $p_{t,k}$ given $H_{t-1}$, 
and the second inequality follows from the replacement rate of $0.99$ and the fact that 99\% of the coins in the bag have a bias of $0$. 

Therefore, combining the above equation and inequality, 
we have for all $t \in \mathbb{Z}_+$, 
\begin{align*} 
\mathrm{expected\ payoff\ at\ time\ } t = &\ \sum_{k = 2}^K \mathbb{P}(p_{t,k} = 1 | A_t = k)\mathbb{P}(A_t = k) + 0.99\mathbb{P}(A_t = 1)\\ 
\leq &\ 0.0199 \sum_{k = 2}^K \mathbb{P}(A_t = k) + 0.99 \mathbb{P}(A_t = 1).
\end{align*}
Recall that $\lim_{K \rightarrow +\infty} \sum_{k = 2}^K \mathbb{P}(A_t = k) = 1$, because 
\begin{align*}
\sum_{k = 2}^K \mathbb{P}(A_t = k) =
&\ \sum_{h \in \mathcal{H}}\mathbb{P}(A_t \neq 1 | H_{t-1} = h) \mathbb{P}(H_{t-1} = h) \\
\geq &\ \min_{h \in \mathcal{H}} \mathbb{P}(A_t \neq 1 | H_{t-1} = h)\\
= &\ \min_{h \in \mathcal{H}} \mathbb{P}\left(\max_{k \in \{2, ... , K\}} \hat{p}_{t, k} = 1 \Bigg{|} H_{t-1} = h\right) \\
= &\ \min_{h \in \mathcal{H}} \left[ 1 - \prod_{k \in \{2, ... , K\}} \mathbb{P}(\hat{p}_{t, k} = 0 | H_{t-1} = h)\right]  \\
= &\ \min_{h \in \mathcal{H}} \left[ 1 - \prod_{k \in \{2, ... , K\}} \left(1 - \mathbb{P}(\hat{p}_{t, k} = 1 | H_{t-1} = h)\right)\right],  \\
\geq &\ 1 - (1 - 0.0099)^{K-1} = 1 - 0.9901^{K-1}, 
\end{align*}
where the third-to-last equality follows from the fact that $\hat{p}_{t,k}$ are conditionally independent given $H_{t-1}$, and that each $\hat{p}_{t,k}$ takes values in $\{0, 1\}$ since the bag contains only coins with biases 0 and 1, and the last inequality follows from $\mathbb{P}(\hat{p}_{t, k} = 1 | H_{t-1} = h) \geq 0.99 \times 0.01 = 0.0099$  given that the replacement rate is 
$0.99$ and 1\% of the coins in the bag have a bias of 1.

Therefore, TS accumulates an expected payoff of at most 2 \textcent per timestep when $K$ is sufficiently large. 
}

\section{How Much Can PS Improve Over TS: Proof of Theorems~\ref{theorem:performance_difference_1} and~\ref{theorem:performance_difference_2} \label{appendix:difference}}

\maximumdifferencets*
\maximumdifference*

\begin{proof} 
Consider a modulated Bernoulli bandit (see Example~\ref{ex:modulated_bernoulli}) with $K$ arms. 
Arm $1$ has a deterministic mean of 
$p_1 = 1 - \epsilon$,  
which does not change over time. Each of arm $2$ through $K$'s mean reward takes value $0$ with 
probability $\beta$ and $1$ with probability $1-\beta$. 
The probability of transition for each of arm $2$ through $K$ is $q$. 

A PS agent estimates 
$\hat{\theta}_{t}$
at each timestep $t \in \mathbb{Z}_+$. Note that 
$\hat{\theta}_{t,1} = p_1$. 
For all $a \in \{2, ... , K\}$, 
\begin{align*}
    \E\left[R_{t, a} | H_{t-1}, R_{t+1:\infty}\right]
    =  \Pr\left(\theta_{t, a} = 1 | H_{t-1}, R_{t+1:\infty}\right)
    \leq 
    1 - q^2\beta\ \text{a.s.}
\end{align*}
This implies that for all timestep $t \in \mathbb{Z}_+$ and 
$a \in \{2, ... , K\}$, 
$\hat{\theta}_{t,a} \leq 1 - q^2 \beta$ a.s.

When $q$ and $\beta$ are sufficiently large, for $a \in \{2, ... , K\}$, $\hat{\theta}_{t,a} \leq 1 - q^2 \beta \leq 1 - \epsilon = p_1 = \hat{\theta}_{t, 1}$, and a PS agent selects action $1$ with probability one and collects cumulative reward 
\begin{align*}
   \mathrm{Return}(T; \pi_{\mathrm{PS}})
   =
   p_1 T = (1 - \epsilon) T. 
\end{align*}

A TS agent estimates $\hat{\theta}_t^{\pi_{\mathrm{TS}}}$ at each timestep $t \in \mathbb{Z}_+$. Note that 
$\hat{\theta}_{t, 1}^{\pi_{\mathrm{TS}}} = p_1 \in (0, 1)$. 
So a TS agent selects action $1$ at each timestep $t \in \mathbb{Z}_+$ with probability
\begin{align*}
    \Pr\left(\max_{a \in \{2, ... , K\}} \hat{\theta}_{t, a}^{\pi_{\mathrm{TS}}} < \hat{\theta}^{{\pi}_{\mathrm{TS}}}_{t, 1} \bigg| H_{t-1}^{\pi_{\mathrm{TS}}}\right) = \prod_{a = 2}^K \Pr\left(\hat{\theta}_{t, a}^{\pi_{\mathrm{TS}}} = 0 \Big| H_{t-1}^{\pi_{\mathrm{TS}}}\right)
    \leq \left(1 - q(1 - \beta)\right)^{K - 1} \ \text{a.s.}
\end{align*}
This probability is arbitrarily close to one with a sufficiently large $K$. 
So a TS agent accumulates rewards 
\begin{align*}
\mathrm{Return}(T; \pi_{\mathrm{TS}})
\leq &\
\left\{\left(1 - q(1 - \beta)\right)^{K - 1} (1 - \epsilon) + \left[1 - \left(1 - q(1 - \beta)\right)^{K - 1}\right] \left(1 - q \beta\right)\right\} T\\
\leq &\ \left\{\left(1 - q(1 - \beta)\right)^{K - 1}  + (1 - q\beta)\right\} T. 
\end{align*}
This is upper bounded by $\epsilon T$ if $q$ and $\beta$ are sufficiently large and that $K$ is sufficiently large. 
\end{proof}

\section{Equivalence of PS to TS in Stationary Bandits: Proof of Proposition~\ref{prop:equivalence} \label{appendix:equivalence}
}
\pstsequivalence*
\begin{proof}
It is sufficient to show that for all $t \in \mathbb{Z}_+$, 
\begin{align}
    \Pr\left(\hat{\theta}_t \in \cdot | H_{t-1}\right) = \Pr\left(\hat{\theta}_t^{\pi_{\mathrm{TS}}} \in \cdot | H^{\pi_\mathrm{TS}}_{t-1} \leftarrow H_{t-1}\right). 
\label{eq:equiv}
\end{align}
We first show that if $H_{t-1}$ and $H^{\pi_{\mathrm{TS}}}_{t-1}$ have the same support, then the change of measure is well-defined and \eqref{eq:equiv} holds. 

First observe that for all $t \in \mathbb{Z}_+$, $\Pr(\hat{\chi}_t \in \cdot | H^{\pi_\mathrm{TS}}_{t-1} ) = \Pr(\chi_t \in \cdot | H^{\pi_\mathrm{TS}}_{t-1})$ and $\Pr(\hat{R}^{(t-1)}_{t+1:\infty} \in \cdot | H_{t-1}) = \Pr(R_{t+1:\infty} \in \cdot | H_{t-1})$. Therefore, for all $t \in \mathbb{Z}_+$, 
\begin{align*}
&\ \Pr\left(\hat{\theta}_t^{\pi_{\mathrm{TS}}} \in \cdot \Big| H^{\pi_\mathrm{TS}}_{t-1}\right) = \Pr\left(\E[R_{t} | H_{t-1}^{\pi_{\mathrm{TS}}}, \chi_t \leftarrow \hat{\chi}_t] \in \cdot | H^{\pi_\mathrm{TS}}_{t-1}\right) = 
\Pr\left(\E[R_{t}| H_{t-1}^{\pi_{\mathrm{TS}}}, \chi_t] \in \cdot | H^{\pi_\mathrm{TS}}_{t-1}\right)\\
\text{and} &\ 
\Pr\left(\hat{\theta}_t \in \cdot \Big| H_{t-1}\right) = \Pr\left(\E[R_{t} | H_{t-1}, R_{t+1:\infty} \leftarrow \hat{R}^{(t)}_{t+1:\infty}] \in \cdot \Big| H_{t-1}\right)
= \Pr\left(\E[R_{t} | H_{t-1}, R_{t+1:\infty}] \in \cdot | H_{t-1}\right).\numberthis
\label{eq:equiv_proof_0}
\end{align*}
In addition, for all $t \in \mathbb{Z}_+$, 
\begin{align}
\E[R_{t} | H_{t-1}, R_{t+1:\infty}] 
\overset{\rm a.s.}{=} \E[R_{t} | H_{t-1}, R_{t+1:\infty}, P]
= \E[R_{t} | H_{t-}, P].
\numberthis
\label{eq:equiv_proof_1}
\end{align}
These conditional expectations determine how actions are sampled by PS and TS, and the equivalence implies that the two implement the same policy; that is, for all $t \in \mathbb{Z}_+$, 
\begin{align*}
      \Pr\left(\hat{\theta}_t \in \cdot \Big| H_{t-1}\right)
      \stackrel{(a)}{=} &\ 
      \Pr\left(\E[R_{t} | H_{t-1}, R_{t+1:\infty}] \in \cdot | H_{t-1}\right) \\
      \stackrel{(b)}{=} &\ 
      \Pr\left(\E[R_{t} | H_{t-1}, P] \in \cdot | H_{t-1}\right) \\
      = &\ \Pr\left(\E[R_{t} | H^{\pi_{\mathrm{TS}}}_{t-1}, P] \in \cdot 
      \Big| H^{\pi_{\mathrm{TS}}}_{t-1} \leftarrow H_{t-1}\right) \\
      \stackrel{(c)}{=} &\ 
      \Pr\left(\hat{\theta}_t^{\pi_{\mathrm{TS}}} \in \cdot \Big| H^{\pi_{\mathrm{TS}}}_{t-1} \leftarrow H_{t-1}\right), 
\end{align*}
where $(a)$ follows from \eqref{eq:equiv_proof_0}, $(b)$ follows from \eqref{eq:equiv_proof_1}, and $(c)$ follows from \eqref{eq:equiv_proof_0} and the fact that the TS agent we consider takes $\chi_t = P$ for all $t \in \mathbb{Z}_+$. 
Note that $H_0 = H_0^{\pi_{\mathrm{TS}}}$, so it is clear that by induction, for all $t \in \mathbb{Z}_+$, $H_t$ and $H_t^{\pi_{\mathrm{TS}}}$ have the same support and \eqref{eq:equiv} holds. 
\end{proof}

{
\section{A Motivating Example for Notions of Regret in Non-Stationary Bandits: Discussion in Section~\ref{sec:regret_def}}
\label{appendix:regret_non-stationary}
Below we provide a simple example, to show that multiple valid ``expected rewards" can exist, and thus it is hard to directly extend the traditional notion of regret defined for stationary environments to non-stationary ones. 
\begin{example}
[\bf{A Gaussian Bandit with Two Independent Actions}] Let 
$\theta_{1} = 0.8$, 
and $\{\theta_{t,2}\}_{t = 1}^{+\infty}$, 
$\{W_{2t, 2}\}_{t = 1}^{+\infty}$, 
$\{Z_{t, 1}\}_{t = 1}^{+\infty}$, and $\{Z_{t, 2}\}_{t = 1}^{+\infty}$ be independent sequences, each 
independent and identically distributed according to $\mathcal{N}(0, 1)$. 
Let $W_{2t-1, 2} = W_{2t, 2}$ for $t \in \mathbb{Z}_{+}$. 
We define the rewards as $R_{t, 1} = \theta_{t, 1} + Z_{t,1}$ and $R_{t,2} = \theta_{t,2} + W_{t,2} + Z_{t,2}$ for all $t \in \mathbb{Z}_{+}$. 
\end{example}

This example describes a non-stationary Gaussian bandit with two independent actions $\actions = \{1, 2\}$ and reward process $\{R_t\}_{t = 1}^{+\infty}$. In this example, each of 
$\{\theta_{t,2}\}_{t = 1}^{+\infty}$, 
$\{\theta_{t,2} + W_{t,2}\}_{t = 1}^{+\infty}$, 
and $\{R_{t,2}\}_{t = 1}^{+\infty}$ itself qualify as a potential choice for the sequence of ``expected rewards" associated with action $2$. 
}

\section{Nonnegativity of Hindsight Regret: Proof of Proposition~\ref{proposition:baseline}}
\label{appendix:regret_nonnegative}
\regretnonnegative*

\begin{proof}
For all policies $\pi$, 
and 
$t \in \mathbb{Z}_+$, 
\begin{align*}
     \E\left[R_{t, A_t^{{\pi}}}\right]
    = &\ \E\left[\E\left[R_{t, A_t^{{\pi}}} | H_{t-1}^{\pi}\right] \right] \\
    \stackrel{(a)}{=} &\ \E\left[\sum_{a \in \actions} \E\left[R_{t, a} |H_{t-1}^{\pi}\right] \Pr(A_t^{\pi} = a | H_{t-1}^{\pi})\right] \\
     \leq &\ \E\left[\max_{a \in \actions} \E\left[R_{t, a} | H_{t-1}^{\pi} \right] \right]\\
     \stackrel{}{=} &\ \E\left[\max_{a \in \actions} \E\left[\E[R_{t, a}|R_{1:t-1}] | H_{t-1}^{\pi} \right] \right]\\
     \stackrel{(b)}{\leq} &\ \E[R_{t, *}], \numberthis
     \label{eq:lowerbound_eq1}
\end{align*}
where $(a)$ follows from the fact that  $A_t^{\pi}$ is independent of $R_{t}$ conditioned on $H_{t-1}^{\pi}$, and $(b)$ follows from Jensen's inequality.  
\end{proof}

\section{Hindsight Regret and Foresight Regret: Proof of Proposition~\ref{proposition:strong_regret}}
\label{appendix:strong_regret}
\strongregret*
\begin{proof}
It suffices to show that for all $t \in \mathbb{Z}_+$, 
\begin{align*}
\E[R_{t, *}] \leq \E[R_{t, \overline{*}}]. 
\end{align*}
For all $t \in \mathbb{Z}_+$, 
\begin{align*}
\E[R_{t, \overline{*}}] 
\stackrel{}{=} &\ \E\left[\E\left[R_{t, \overline{*}}|R_{t+1:2t-1}\right]\right] \\
= &\ \E\left[\E\left[\max_{a \in \actions}\E\left[R_{t,a} | R_{t+1:\infty}\right]\bigg|R_{t+1:2t-1}\right]\right]\\
\stackrel{(a)}{\geq} &\ \E\left[\max_{a \in \actions}\E\left[\E\left[R_{t,a} | R_{t+1:\infty}\right]\bigg|R_{t+1:2t-1}\right]\right] \\
\stackrel{}{=} &\ \E\left[\max_{a \in \actions}\E\left[R_{t,a}|R_{t+1:2t-1}\right]\right] \\
\stackrel{(b)}{=} &\ \E\left[\max_{a \in \actions}\E\left[R_{t,a}|R_{1:t-1}\right]\right] \\
= &\ \E[R_{t, *}], 
\end{align*}
where $(a)$ follows from Jensen's inequality, and $(b)$ from reversibility. 
\end{proof}

\section{Predictive Information in Stationary Bandits: Proof of Proposition~\ref{proposition:predictive_information_stationary}}
\label{appendix:predictive_information_stationary}

\predictiveinformations*
\begin{proof} 
Observe that conditioned on $R_{1:t-1}$ and $P$, $R_{t+1:\infty} \perp R_{t}$; conditioned on $R_{1:t-1}$ and $R_{t+1:\infty}$, $P \perp R_{t}$. 
Therefore, by data-processing inequality for mutual information, we have 
\begin{align}
\I(R_{t+1:\infty}; R_{t} | R_{1:t-1}) =
 \I(P; R_{t} | R_{1:t-1}).
 \label{eq:information_proof_1}
 \end{align}

Then the cumulative predictive information satisfies
\begin{align*}
\sum_{t = 1}^{+\infty} \Delta_t = &\ \sum_{t = 1}^{+\infty} \I(R_{t+1:\infty}; R_{t} | R_{1:t-1}) \\
= &\ \sum_{t = 1}^{+\infty} \I(P; R_{t} | R_{1:t-1}) \\
= &\ \I(P; R_{1:\infty}) = \H(P), 
\end{align*}
where the second equality follows from \eqref{eq:information_proof_1}, and the 
third equality follows from chain rule for mutual information. 
\end{proof}

\section{General Regret Analysis: Proof of Theorem~\ref{theorem:general_regret}}
\label{appendix:general_regret}

\generalregret*

\begin{proof}
For all policies $\pi$ and $T \in \mathbb{Z}_+$, 
\begin{align*}
{\mathrm{Regret}}_{\mathrm{F}}(T; \pi)
=& \sum_{t= 1 }^{T} \E\left[R_{t, \overline{*}} - R_{t, A_t^{\pi}}\right] \\
\overset{(a)}{\leq}& \sum_{t=1}^{T} 
\sqrt{\Gamma_{t}^{\pi} \I(R_{t+1:\infty}; A_t^{\pi}, R_{t, A_t^{\pi}}|H_{t-1}^{\pi})}\\
\overset{(b)}{\leq}& \sqrt{\sum_{t=1}^{T} \I\left(R_{t+1:\infty}; A_t^{\pi}, R_{t, A_t^{\pi}}|H_{t-1}^{\pi}\right)} \sqrt{\sum_{t = 1}^{T}{\Gamma}^{\pi}_t}, 
\numberthis
\label{eq:main_proof_eq_1}
\end{align*}
where $(a)$ follows from the definition of the information ratio, 
and $(b)$ follows from the Cauchy-Schwarz inequality.
Observe that 
 $R_{t+1:\infty} \perp H^{\pi}_{t} | R_{1:t}$. 
Hence, for all policies $\pi$ and $t \in \mathbb{Z}_+$, 
\begin{align*}
       \I\left(R_{t+1:\infty}; A_t^{\pi}, R_{t, A_t^{\pi}} | H_{t-1}^{\pi}\right) 
    = \I\left(R_{t+1:\infty} ; R_{1:t} | H_{t-1}^{\pi}\right) - \I\left(R_{t+1:\infty} ; R_{1:t}|H_{t}^{\pi}\right).
\end{align*}
Therefore, for all policies $\pi$ and $T \in \mathbb{Z}_+$, 
\begin{align*}
    & \sum_{t = 1}^{T} \I\left(R_{t+1:\infty}; A_t^{\pi}, R_{t, A_t^{\pi}} | H_{t-1}^{\pi}\right) \\
    \stackrel{}{=} &\ \sum_{t = 1}^{T}\left[\I\left(R_{t+1:\infty}; R_{1:t} | H_{t-1}^{\pi}\right) - \I\left(R_{t+1:\infty}; R_{1:t}|H_{t}^{\pi}\right) \right]\\
    \leq &\ \I(R_{2:\infty}; R_1) + 
    \sum_{t = 2}^{T} \left[\I\left(R_{t+1:\infty}; R_{1:t} | H_{t-1}^{\pi}\right) - \I\left(R_{t:\infty}; R_{1:t-1}|H_{t-1}^{\pi}\right)\right] \\ 
    \stackrel{(a)}{=} &\ \I(R_{2:\infty}; R_1) + 
    \sum_{t = 2}^{T} \left[\I\left(R_{t+1:\infty}; R_{1:t-1}, | H_{t-1}^{\pi}\right)
    + \I\left(R_{t+1:\infty}; R_{t} | R_{1:t-1}, H_{t-1}^{\pi}\right)
    - \I\left(R_{t:\infty}; R_{1:t-1}|H_{t-1}^{\pi}\right)\right] \\ 
    \leq &\ \I(R_{2:\infty}; R_1) + 
    \sum_{t = 2}^{T} \I\left(R_{t+1:\infty}; R_{t} | R_{1:t-1}, H_{t-1}^{\pi}\right) \\ 
    \stackrel{(b)}{=} &\  \I(R_{2:\infty}; R_1) + 
    \sum_{t = 2}^{T} 
    \I\left(R_{t+1:\infty}; R_{t} | R_{1:t-1}\right) \\
    \stackrel{}{=} &\ \sum_{t = 1}^{T} \Delta_t, 
    \numberthis
    \label{eq:main_proof_eq_2}
\end{align*}
where $(a)$ follows from the chain rule of mutual information, and $(b)$ from $R_{t:\infty} \perp H_{t-1}^{\pi} | R_{1:t-1}$. 
Incorporating \eqref{eq:main_proof_eq_1} and \eqref{eq:main_proof_eq_2}, we complete the proof. 
\end{proof}

\section{Bounding the Predictive Information: Proof of Lemma~\ref{lemma:markov_info}}
\label{appendix:predictive_information}
\markovinfo*

\begin{proof}
Observe that for all $t \in \mathbb{Z}_+$, the incremental predictive information satisfies 
\begin{align*}
\Delta_t = &\ \I(R_{t+1:\infty}; R_{t} | R_{1:t-1})\\
\stackrel{(a)}{\leq} &\ \I(S_{t+1}; R_{t} | R_{1:t-1}) \\
= &\ \H(S_{t+1} | R_{1:t-1}) - \H(S_{t+1} | R_{1:t})\\
= &\ \H(S_{t+1} | R_{1:t-1}, S_{t}) 
+ \I(S_{t+1}; S_{t} | R_{1:t-1}) 
- \H(S_{t+1} | R_{1:t}, S_{t})
- \I(S_{t+1}; S_{t} | R_{1:t})\\
\stackrel{(b)}{=} &\ \H(S_{t+1} | S_{t}) 
+ \I(S_{t+1}; S_{t} | R_{1:t-1}) 
- \H(S_{t+1} | S_{t})
- \I(S_{t+1}; S_{t} | R_{1:t})\\
= &\ \I(S_{t+1}; S_{t} | R_{1:t-1}) 
- \I(S_{t+1}; S_{t} | R_{1:t}), 
\end{align*}
where $(a)$ follows from $R_{t+1:\infty} \perp R_{t} | S_{t+1}, R_{1:t-1}$ and the data processing inequality, and $(b)$ from $S_{t+1} \perp R_{1:t} | S_{t}$. 

Then the cumulative predictive information can be upper-bounded as follows:
\begin{align*}
\sum_{t = 1}^{T} \Delta_t = &\ 
\sum_{t = 1}^{T} \left[\I(S_{t+1}; S_{t} | R_{1:t-1}) 
- \I(S_{t+1}; S_{t} | R_{1:t})\right]\\
= &\ \I(S_2; S_1) + 
\sum_{t = 1}^{T-1} \left[\I(S_{t+2}; S_{t+1} | R_{1:t}) 
- \I(S_{t+1}; S_{t} | R_{1:t})\right] \\
\leq &\ \I(S_2; S_1) + 
\sum_{t = 1}^{T-1} \left[\I(S_{t+2}, S_{t}; S_{t+1} | R_{1:t}) 
- \I(S_{t+1}; S_{t} | R_{1:t})\right] \\
= &\ \I(S_2; S_1) + 
\sum_{t = 1}^{T-1} \left[\I(S_{t+2}; S_{t+1} | R_{1:t}, S_{t}) \right] \\
\stackrel{(a)}{=} &\ \I(S_2; S_1) + 
\sum_{t = 1}^{T-1} \left[\I(S_{t+2}; S_{t+1} | S_{t}) \right], 
\numberthis
\label{eq:markov_info_bound}
\end{align*}
where $(a)$ follows from $(S_{t+1}, S_{t+2}) \perp R_{1:t} | S_{t}$. 
\end{proof}

\section{Recovering Existing Bounds for Stationary Bandits}
\label{appendix:recover_stationary}
\subsection{Comparing Theorem~\ref{theorem:general_regret} to Existing Results}
It follows directly from  Lemma~\ref{lemma:markov_info} and Theorem~\ref{theorem:general_regret} that for all policies $\pi$ and $T \in \mathbb{Z}_+$, the regret satisfies 
\begin{align*}
{\mathrm{Regret}}_{\mathrm{F}}(T; \pi) \leq \sqrt{\left(\sum_{t = 1}^{T}\Gamma_t^{\pi}\right)[\I(S_2;S_1) + \sum_{t = 1}^{T-1} \I(S_{t+2}; S_{t+1} | S_{t})}].
\end{align*}

To relate our regret bound to existing regret bounds established in the literature of stationary bandits, we compare our bound to a bound established by \cite{neu2022lifting}. 
Observe that in a stationary bandit,
if the information ratio satisfies $\Gamma_t^{\pi} \leq \overline{\Gamma}^{\pi}$ for all $t \in \mathbb{Z}_+$ 
for some $\overline{\Gamma}^{\pi}$, 
then our result establishes that 
${\mathrm{Regret}}_{\mathrm{F}}(T; \pi) \leq \sqrt{\overline{\Gamma}^{\pi}T\H(P)}$, by letting $S_t$ be the reward distribution $P$. This is equivalent to an information-theoretic regret bound established by \cite{neu2022lifting}.  

\subsection{Comparing Corollary~\ref{cor:regret} to Existing Results for TS}
In a stationary bandit, by letting $S_t$ be the reward distribution, which we denote by $P$, Corollary~\ref{cor:regret} implies that 
\begin{align*}{\mathrm{Regret}}_{\mathrm{F}}(T) \leq \sqrt{2|\actions|\sigma_{\mathrm{SG}}^2 T \H(P)}.
\end{align*}
This regret bound for PS is identical to an information-theoretic regret bound for TS established by \cite{neu2022lifting} in stationary bandits.

\section{PS Regret Analysis: Proof of Lemma~\ref{lemma:ir}}
\label{appendix:info_ratio}
\ir*

\begin{proof}
For all $t \in \mathbb{Z}_+$, let 
\begin{align*}
\overline{\theta}_{t}^{H} = \E\left[R_{t} \big| H_{t-1},  R_{t+1:\infty}\right], 
\end{align*}
and $
A_{t, *}^{H} \in \arg \max_{a \in \actions}
\overline{\theta}^{H}_{t, a}$
satisfying $A_{t, *}^{H} \perp A_t | H_{t-1}$, 
and $R_{t, *}^{H} = R_{t, A_{t, *}^{H}}.
$ 
Then for all $t \in \mathbb{Z}_+$, we have
\begin{align*}
    \Pr\left(A_{t, *}^{H} \in \cdot | H_{t-1}\right) = \Pr(A_t \in \cdot | H_{t-1}) \text{ and } A_{t, *}^{H} \perp A_t | H_{t-1}. 
\end{align*}

We begin by establishing a relation using KL-divergence.  For all $a,a' \in \actions$, and $\lambda \in \R_+$, it follows from the variantal form of KL-divergence (Lemma~\ref{lemma:var_KL} of Appendix~\ref{appendix:information}) with $X = \lambda (R_{t,a} - \E[R_{t,a}|H_{t-1}])$ that
for all $t \in \mathbb{Z}_+$ and $h \in \mathcal{H}_t$, 
\begin{align*}
&\ \KL\left(\Pr ({R}_{t,a} \in \cdot | A_{t, *}^{H} = a', H_{t-1} = h ) \Big\| \  \Pr({R}_{t,a} \in \cdot | H_{t-1} = h)\right)\\
\geq &\ \E\left[X|H_{t-1} = h, A_{t, *}^{H} = a'\right] - \ln \E[\exp(X) |H_{t-1} = h] \\
\geq &\ \lambda \E\left[R_{t,a} - \E[R_{t,a}|H_{t-1} = h]|H_{t-1} = h, A_{t, *}^{H} = a'\right] - \frac{1}{2} \lambda^2 \sigma_{\text{SG}}^2.
\end{align*}
By maximizing over $\lambda$, we obtain
\begin{align*}
\label{eq:KL-regret-lemma}
&\ \left(\E\left[{R}_{t,a} | A_{t, *}^{H} = a' , H_{t-1} = h \right] - \E\left[ {R}_{t,a} | H_{t-1} = h \right]\right)^2 \\ \leq &\ 2 \sigma_{\text{SG}}^2 \KL \left(\Pr\left({R}_{t,a} \in \cdot | A_{t, *}^{H} = a^{\prime}, H_{t-1} = h \right) \Big\|\   \Pr\left({R}_{t,a} \in \cdot | H_{t-1} = h \right) \right).
\numberthis
\end{align*} 
We next establish a relation between this KL-divergence and mutual information. In particular,
\begin{align}
&\ \I\left(A_{t, *}^{H}; A_t, {R}_{t, A_t} | {H}_{t-1}  = h \right) 
\nonumber
\\
= &\ \I\left(A_{t, *}^{H}; A_t  | {H}_{t-1}  = h \right) + \I\left(A_{t, *}^{H}; {R}_{t, A_t} | A_t, {H}_{t-1}  = h \right)
\nonumber
\\
 \overset{(a)}{=} &\ \I\left(A_{t, *}^{H}; {R}_{t, A_t } | A_t , {H}_{t-1}  = h \right)
\nonumber
\\
\overset{}{=} &\ \sum_{a \in \actions} \Pr(A_t = a | H_{t-1} = h ) \I\left(A_{t, *}^{H}; {R}_{t, A_t } | A_t = a, {H}_{t-1} = h \right)
\nonumber
\\
\overset{}{=} &\ \sum_{a \in \actions} \Pr(A_t = a | {H}_{t-1} = h) \I\left(A_{t, *}^{H}; {R}_{t, a } | A_t = a, {H}_{t-1} = h \right)
\nonumber
\\
\overset{(b)}{=} &\ \sum_{a \in \actions} \Pr(A_t = a | {H}_{t-1} = h ) \I\left(A_{t, *}^{H}; {R}_{t, a } | {H}_{t-1} = h \right)
\nonumber
\\
\overset{(c)}{=} &\ \sum_{a \in \actions} \Pr(A_t = a | {H}_{t-1} = h) 
\left[ \sum_{a' \in \actions} \Pr\left(A_{t, *}^{H} = a' | {H}_{t-1} = h\right) \right.\\
&\left.\quad \cdot \KL \left( \Pr\left({R}_{t, a} \in \cdot | A_{t, *}^{H} = a', {H}_{t-1} = h\right) \big\|\ 
\Pr( {R}_{t, a} \in \cdot | {H}_{t-1} = h) \right) \right]
\nonumber
\\
 \overset{(d)}{=} &\ \sum_{a \in \actions} \sum_{a' \in \actions}
 \Pr\left(A_{t, *}^{H} = a | H_{t-1} = h\right)
 \Pr\left(A_{t, *}^{H} = a' | H_{t-1} = h\right)\\
&\quad \cdot \KL \left( \Pr\left({R}_{t, a} \in \cdot | A_{t, *}^{H} = a', H_{t-1} = h\right) \big\|\ 
\Pr( {R}_{t, a} \in \cdot | H_{t-1} = h) \right) 
\label{eq:KL-MI-lemma}
\end{align}
where $(a)$ follows from the fact that $A_t \perp A_{t, *}^{H} | H_{t-1} $, 
$(b)$ follows from $A_t\perp(A_{t, *}^{H}, R_{t,a})|H_{t-1}$, $(c)$ follows from the KL-divergence representation of mutual information (Lemma~\ref{lemma:KL_MI} of Appendix~\ref{appendix:information}), and $(d)$ follows from  $\Pr({A}_t \in \cdot | H_{t-1} = h) = \Pr(A_{t, *}^{H} \in \cdot | H_{t-1} = h)$ for all $t \in \mathbb{Z}_+$ and $h \in \mathcal{H}_t$.

Next, we bound the difference between $R_{t,A_{t, *}^{H}}$ and $R_{t, A_t}$. For all $t \in \mathbb{Z}_+$ and $h \in \mathcal{H}_t$, we have
\begin{align*}
&\ \E\left[R_{t,A_{t, *}^{H}} - R_{t, A_t}\big|H_{t-1} = h \right]^2 \\
\stackrel{(a)}{=} &\ \left[\sum_{a \in \actions} \Pr\left(A_{t, *}^{H} = a \Big| H_{t-1} = h \right) \right.\\
&\left.\quad \cdot \left(\E\left[R_{t, a}\Big|A_{t, *}^{H} = a,H_{t-1} = h\right] - \E[R_{t, a}|H_{t-1} = h]\right)\right]^2 \\
\stackrel{(b)}{\leq} &\ |\actions| \sum_{a \in \actions} \Pr\left(A_{t, *}^{H} = a | H_{t-1}=h\right)^2 \\
&\quad \cdot \left(\E\left[R_{t, a}|A_{t, *}^{H} = a,H_{t-1}=h\right] - \E\left[R_{t, a}|H_{t-1}=h\right]\right)^2 \\
\leq &\ |\actions| \sum_{a \in \actions} \sum_{a' \in \actions}
\Pr\left(A_{t, *}^{H} = a | H_{t-1}=h\right)
\Pr\left(A_{t, *}^{H} = a' | H_{t-1}=h\right) \\
&\quad \cdot \left(\E\left[R_{t, a}|A_{t, *}^{H} = a',H_{t-1}=h\right] - \E\left[R_{t, a}|H_{t-1}=h\right]\right)^2 \\
\stackrel{(c)}{\leq} &\ 2 |\actions| \sigma_{\text{SG}}^2
\sum_{a \in \actions} \sum_{a' \in \actions}
\Pr\left(A_{t, *}^{H} = a | H_{t-1}=h\right)
\Pr\left(A_{t, *}^{H} = a' | H_{t-1}=h\right) \\
&\quad \cdot \KL\bigl(\Pr\left(R_{t,a} \in \cdot | A_{t, *}^{H} = a',H_{t-1}=h\right) \\
&\qquad \big\|\ \Pr\left(R_{t,a} \in \cdot | H_{t-1}=h\right)\bigr) \\
\stackrel{(d)}{=} &\ 2 |\actions| \sigma_{\text{SG}}^2 \\
&\quad \cdot \I\left(A_{t, *}^{H}; A_t, R_{t,A_t} \right.\\
&\qquad \left. | H_{t-1}=h\right),
\numberthis
\label{eq:info_ratio_proof_2}
\end{align*}
where $(a)$ follows from 
$A_t\perp R_{t,a}|H_{t-1}$ and $\Pr({A}_t \in \cdot | H_{t-1} = h) = \Pr(A_{t, *}^{H} \in \cdot | H_{t-1} = h)$, $(b)$ follows from the Cauchy-Schwartz inequality, $(c)$ follows from Equation (\ref{eq:KL-regret-lemma}), and $(d)$ follows from Equation (\ref{eq:KL-MI-lemma}). 
Hence,
\begin{align*}
 {\E\left[R_{t, A_{t, *}^{H}} - R_{t, A_{t}}\right]^2}
 = &\ \E\left[\E \left[R_{t, A_{t, *}^{H}} - R_{t, A_{t}}| H_{t-1} \right]\right]^2 \\
 \stackrel{(a)}{\leq} &\ \E\left[\E \left[R_{t, A_{t, *}^{H}} - R_{t, A_{t}}| H_{t-1} \right]^2 \right] \\
  \stackrel{(b)}{\leq} &\ 
    \E\left[2 |\actions| \sigma^2_{\mathrm{SG}} \I\left(A_{t, *}^{H}; A_t, R_{t, A_{t}} | H_{t-1} = H_{t-1}\right)\right] \\
= &\ 2 |\actions| \sigma^2_{\mathrm{SG}}
    \I\left(A_{t, *}^{H}; A_t, R_{t, A_{t}} | H_{t-1}\right), \numberthis
\label{eq:proof_info_ratio_1}
\end{align*}
where $(a)$ follows from Jensen's Inequality and $(b)$ follows from \eqref{eq:info_ratio_proof_2}. 

In addition, for all $t \in \mathbb{Z}_+$, 
\begin{align*}
    \E[R_{t, \overline{*}}] 
    = &\ \E\left[\max_{a \in \actions} \E\left[R_{t, a} | R_{t+1:\infty}\right]\right] \\
    = &\ \E\left[\max_{a \in \actions} \E\left[\E\left[R_{t, a} | H_{t-1},  R_{t+1:\infty}\right] | R_{t+1:\infty} \right]\right] \\
    \stackrel{(a)}{\leq} &\ \E\left[ \E\left[\max_{a \in \actions} \E\left[R_{t, a} | H_{t-1},  R_{t+1:\infty}\right] \bigg| R_{t+1:\infty} \right]\right] \\
    \stackrel{}{=} &\ \E\left[\max_{a \in \actions} \E\left[R_{t, a} | H_{t-1},  R_{t+1:\infty}\right] \right] \\
    = &\ \E\left[R_{t, A_{t, *}^{H}}\right]. \numberthis
\label{eq:proof_info_ratio_2}
\end{align*}
By the data processing inequality of mutual information (Lemma~\ref{lemma:dp_MI} of Appendix~\ref{appendix:information}), we have for all $t \in \mathbb{Z}_+$, 
\begin{align}
   \I\left(R_{t+1:\infty}; A_t, R_{t, A_t} | H_{t-1}\right)
   \geq \I\left(A_{t, *}^{H}; A_t, R_{t, A_t} | H_{t-1}\right). 
\label{eq:proof_info_ratio_3}
\end{align}
Then it follows from \eqref{eq:proof_info_ratio_1},  \eqref{eq:proof_info_ratio_2} and 
\eqref{eq:proof_info_ratio_3} that for all $t \in \mathbb{Z}_+$, 
\begin{align*}
    \Gamma_t 
    = \frac{\E\left[R_{t, \overline{*}} - R_{t, A_t} \right]^2}{\I\left(R_{t+1:\infty}; A_t, R_{t, A_t} | H_{t-1}\right)}
    \leq \frac{\E\left[R_{t, A_{t, *}^{H}} - R_{t, A_t} \right]^2}{\I\left(R_{t+1:\infty}; A_t, R_{t, A_t} | H_{t-1}\right)}
    \leq \frac{\E\left[R_{t, A_{t, *}^{H}} - R_{t, A_t} \right]^2}{\I\left(A_{t, *}^{H}; A_t, R_{t, A_t} | H_{t-1}\right)}
    \leq 2 |\actions| \sigma_{\mathrm{SG}}^2. 
\end{align*}
\end{proof}

\section{PS Regret Upper Bound in a Modulated Bernoulli Bandit: Proof of Theorem~\ref{theorem:regret_bernoulli}}
\label{appendix:modulated_bernoulli}
\regretbernoulli*
{Note that in a modulated Bernoulli bandit, $R_{t,a} \in [0,1]$ for each $t \in \mathbb{Z}_+$ and $a \in \actions$. So the rewards are sub-Gaussian with parameter $\sigma_{\mathrm{SG}} = \frac{1}{2}$. Then}
Theorem~\ref{theorem:regret_bernoulli} follows directly from Lemma~\ref{lemma:bernoulli_pred}
{ and Corollary~\ref{cor:regret}}. Below we present this lemma followed by its proof. 
\begin{lemma}
\label{lemma:bernoulli_pred}
In a modulated Bernoulli bandit, for all $T \in \mathbb{Z}_+$, the cumulative predictive information satisfies 
{
\begin{align*}
\sum_{t = 0}^{T-1}\Delta_t 
\leq &\ \I(\theta_2; \theta_1) + (T - 1) \I(\theta_3; \theta_2 | \theta_1) \\
\leq  &\ \sum_{a \in \actions} (1 - q_a) \H(\theta_{1,a})
+ (T-1) \sum_{a \in \actions} \left[2 \H(q_a) + q_a(1 - q_a) \H(\theta_{1,a})\right]
\end{align*}
}
\remove{\begin{align*}
\sum_{t = 0}^{T-1}\Delta_t 
\leq  \sum_{a \in \actions} (1 - q_a) \H(\theta_{0,a})
+ (T-1) \sum_{a \in \actions} \left[2 \H(q_a) + q_a(1 - q_a) \H(\theta_{0,a})\right]
\end{align*}}
\end{lemma}
\begin{proof}
Applying Lemma~\ref{lemma:markov_info} by letting $S_t = \theta_{t}$, we can bound the cumulative predictive information as follows
\begin{align*}
\sum_{t = 1}^{T} \Delta_t \leq 
\I(S_2; S_1) + (T - 1) \I(S_3; S_2 | S_1) 
= 
\I(\theta_2; \theta_1) + (T - 1) \I(\theta_3; \theta_2 | \theta_1).
\end{align*}

We can further upper-bound the total predictive information by the entropy and conditional entropy of mean rewards, i.e., by $\H(\theta_1)$, $\H(\theta_2 | \theta_1)$, and $\H(\theta_3 | \theta_1)$: 
\begin{align*}
\sum_{t = 1}^{T} \Delta_t \leq &\ 
\I(\theta_2; \theta_1) + (T - 1) \I(\theta_3; \theta_2 | \theta_1)  \\
= &\ \H(\theta_2) - \H(\theta_2 | \theta_1)  + (T-1)  \left[\H(\theta_3 | \theta_1) - \H(\theta_3 | \theta_2, \theta_1)\right]  \\
\stackrel{}{=} &\ \H(\theta_1) - \H(\theta_2 | \theta_1) + (T-1 )\left[\H(\theta_3 | \theta_1) - \H(\theta_2 | \theta_1)\right], \numberthis
\label{eq:bernoulli_info_bound}
\end{align*}
where the last equality follows from $\theta_3 \perp \theta_1 | \theta_2$.

To upper-bound $\H(\theta_3 | \theta_1)$ and to lower-bound $\H(\theta_2 | \theta_1)$ in \eqref{eq:bernoulli_info_bound}, 
it is helpful to consider an alternative formulation of the modulated Bernoulli bandit. For all $a \in \actions$, let $\{B_{t, a}\}_{t=1}^{\infty}$ be an i.i.d.~$\mathrm{Bernoulli}(q_a)$ process and $\{X_{t,a}\}_{t=1}^{\infty}$ be an i.i.d.~process with discrete range. We let $\theta_{1,a} = X_{1,a}$ and, for all $a \in \actions$ and $t \in \mathbb{Z}_+$,
\begin{align*}
\theta_{t+1, a} = \left\{\begin{array}{ll}
X_{t+1, a} \qquad & \text{if } B_{t+1,a} = 1 \\
\theta_{t, a} \qquad & \text{if } B_{t+1,a} = 0.
\end{array}\right.
\end{align*}
It is clear that each $\theta_{t,a}$ is ``redrawn" from its initial distribution with probability $q_a$ at each timestep. With this alternative formulation, for all $t \in \mathbb{Z}_+$ and $a \in \actions$, 
\begin{align}
    \theta_{t + 1, a} = (1 - B_{t+1, a}) \theta_{t, a}
    + B_{t+1, a} X_{t+1, a}. 
\label{eq:theta_recurrence}
\end{align}
This recursive formula \eqref{eq:theta_recurrence} is helpful in deriving an lower bound for $\H(\theta_2 | \theta_1)$ and in deriving an upper bound for $\H(\theta_3 | \theta_1)$. 
We first derive a lower bound for $\H(\theta_2 | \theta_1)$: 
\begin{align*}
\H(\theta_2 | \theta_1) 
    = &\ \sum_{a \in \actions} \H(\theta_{2, a} | \theta_{1, a}) \\
    \geq &\ \sum_{a \in \actions} \H(\theta_{2, a} | \theta_{1, a}, B_{2,a}) \\
    \stackrel{(a)}{=} &\ \sum_{a \in \actions} \H((1 - B_{2, a}) \theta_{1, a} + B_{2, a} X_{2, a} | \theta_{1, a}, B_{2, a}) \\
    = &\ \sum_{a \in \actions} \H(B_{2,a} X_{2, a} | \theta_{1, a}, B_{2, a}) \\
    \stackrel{(b)}{=} &\ \sum_{a \in \actions} \H(B_{2,a} X_{2, a} | B_{2, a}) \\
    = &\ \sum_{a \in \actions} q_a \H(X_{2, a})\\
    = &\ \sum_{a \in \actions} q_a \H(\theta_{1, a}),
   \numberthis    \label{eq:proof_lemma_bound_eq1}
\end{align*}
where $(a)$ follows from \eqref{eq:theta_recurrence}, and $(b)$ follows from $B_{2,a} X_{2,a} \perp \theta_{1,a} | B_{2,a}$.  

Now we derive an upper bound for $\H(\theta_3 | \theta_1)$: 
\begin{align*}
\H(\theta_3 | \theta_1)
= &\ \sum_{a \in \actions} \H(\theta_{3,a} | \theta_{1,a}) \\
= &\ \sum_{a \in \actions} \H(B_{3,a} X_{3,a} + (1 - B_{3,a})B_{2,a} X_{2,a} + (1 - B_{3, a})(1 - B_{2, a}) \theta_{1, a} | \theta_{2,a}) \\
\leq &\ \sum_{a \in \actions} \H(\theta_{1, a}, B_{2, a}, B_{3, a}, B_{3,a} X_{3,a}, (1 - B_{3,a})B_{2,a} X_{2,a} | \theta_{1,a}) \\
= &\ \sum_{a \in \actions}\left[ \H(B_{2, a}) + \H(B_{3, a}) + \H(B_{3,a} X_{3,a} | B_{3,a}) + \H((1 - B_{3,a})B_{2,a} X_{2,a} | B_{2,a}, B_{3,a})\right] \\
= &\ \sum_{a \in \actions} \left[2\H(q_a) + q_a \H(\theta_{1,a}) + q_a (1 - q_a) \H(\theta_{1,a})\right],  \numberthis
\label{eq:proof_lemma_bound_eq2}
\end{align*}
where the second-to-last equality follows from the independence of $B_{2,a}$, $X_{2,a}, B_{3,a}, X_{3,a}$ and $\theta_{1,a}$. 

Plugging \eqref{eq:proof_lemma_bound_eq1} and \eqref{eq:proof_lemma_bound_eq2} into \eqref{eq:bernoulli_info_bound}, we complete the proof. 
\end{proof}

\section{Regret Lower Bound: Proof of Theorem~\ref{theorem:lower_bound}}
\label{appendix:lower_bound}
\regretlowerbound*

\begin{proof}
We outline the proof below. We first construct a modulated Bernoulli bandit where $\actions = \{1, 2\}$, and $\theta_{1, a} \sim \mathrm{unif}\{0, 1\}$ for each $a \in \actions$; we let $q_1 = 1/2$ and $q_2 = 1$. In this bandit, $q_2 = 1$, so selecting action $2$ provides information that immediately becomes irrelevant in the next timestep; in contrast, selecting action $1$ provides information of better durability. Recall that the regret is defined with respect to an oracle that acts optimally with full knowledge of all past rewards $R_{1:t-1}$. So if an agent selects action $2$ with a large probability, the agent is short in information compared to this oracle, and thus incurs a large regret in the \textit{next timestep}.  On the other hand, if an agent selects action $2$ with a small probability, then the agent collects an expected reward that is close to $\E[R_{t, 1}] = \frac{1}{2}$, and thus incurs a large regret in the \textit{current timestep}. Combining these two arguments, we lower-bound the regret.
 
We introduce a modulated Bernoulli bandit 
with a set of two actions $\actions = \{1, 2\}$, and for each $a \in \actions$, 
\begin{align*}
    \theta_{1, a} = \begin{cases}
    0 & \text{with probability }\ 1/2\\
    1 & \text{with probability }\ 1/2. 
    \end{cases}
\end{align*}
We let $q = [1/2, 1]$. Then for all $t \in \mathbb{Z}_+$, $t\geq 2$, the baseline at time $t$ is 
\begin{align*}
\label{eq:baseline_bernoulli}
    \E[R_{t, \overline{*}}] 
    = &\ \E\left[\max_{a \in \actions} \E[R_{t, a} | R_{t+1:\infty}] \right]\\
     \stackrel{(a)}{=}&\ \E\left[\max_{a \in \actions} \E[R_{t, a} | R_{t+1}] \right]\\
    \stackrel{(b)}{=} &\ \E\left[\max_{a \in \actions} \E[\theta_{t, a} | \theta_{t-1}] \right]\\
     \stackrel{(c)}{=} &\ \E\left[\max_{a \in \actions} \E[\theta_{t, a} | \theta_{t-1, 1}] \right]\\
     = &\ \E\left[\E\left[\max_{a \in \actions} \E[\theta_{t, a} | \theta_{t-1, 1}] \bigg| H_{t-2}^{\pi} \right]\right]\\
     \stackrel{(d)}{=} &\ \E\left[\sum_{a' \in \actions} \E\left[\max_{a \in \actions} \E[\theta_{t, a} | \theta_{t-1, 1}] \bigg| H_{t-2}^{\pi} \right] \Pr(A_{t-1}^{\pi} = a' |H_{t-2}^{\pi}) \right], \numberthis
\end{align*}
where $(a)$ follows from that $ \{\theta_t\}_{t=1}^{\infty}$  follows a Markov process, and that $R_{t} = \theta_t$, $(b)$ follows from the reversibility, $(c)$ follows from $q_2 = 1$, and $(d)$ from that $A_{t-1}^{\pi}$ is independent of $\theta_{t-1}$ conditioned on $H_{t-2}^{\pi}$. 

For any policy $\pi$ and all $t \in \mathbb{Z}_+$, $t\geq 2$, the reward collected at time $t$ is upper-bounded by 
\begin{align*}
\label{eq:reward_bernoulli}
 \E \left[R_{t, A_t^{\pi}}\right]
= &\ \E \left[\E\left[R_{t, A_t^{\pi}} | H_{t-1}^{\pi}\right]\right] \\
= &\ \E \left[\sum_{a \in \actions}\E\left[R_{t, a} | H_{t-1}^{\pi}\right] \Pr(A_t^{\pi} = a | H_{t-1}^{\pi})\right] \\
\leq &\ \E \left[\max_{a \in \actions} \E\left[R_{t, a} | H_{t-1}^{\pi}\right] \right]\\
    = &\ \E\left[\E\left[\max_{a \in \actions} \E[R_{t, a} | H_{t-1}^{\pi}] \bigg| H_{t-2}^{\pi} \right]\right] \\
         = &\ \E\left[\E\left[\max_{a \in \actions} \E[R_{t, a} | H_{t-2}^{\pi}, A_{t-1}^{\pi}, R_{{t-1}, A_{t-1}^{\pi}}] \bigg| H_{t-2}^{\pi} \right]\right] \\
        \stackrel{(a)}{=} &\ \E\left[\sum_{a' \in \actions}\E\left[\max_{a \in \actions}  \E[R_{t, a} | H_{t-2}^{\pi}, R_{t-1, a'}] \bigg| H_{t-2}^{\pi} \right] \Pr(A_{t-1}^{\pi} = a' | H_{t-2}^{\pi})\right] \\
    \stackrel{(b)}{=} &\ \E\left[\sum_{a' \in \actions}\E\left[\max_{a \in \actions} \E[\theta_{t, a} | H_{t-2}^{\pi}, \theta_{t-1, a'}] \bigg| H_{t-2}^{\pi} \right] \Pr(A_{t-1}^{\pi} = a' | H_{t-2}^{\pi})\right],  \numberthis
\end{align*}
where $(a)$ follows from that $A_{t-1}^{\pi}$ is independent of $R_{t-1}$ conditioned on $H_{t-2}^{\pi}$, and $(b)$ from $R_{t} = \theta_t$. Observe that for all $t \in \mathbb{Z}_+$, $t\geq 2$, the term $\E\left[\max_{a \in \actions} \E[\theta_{t, a} | H_{t-2}^{\pi}, \theta_{t-1, a'}] \bigg| H_{t-2}^{\pi} \right]$ in \eqref{eq:reward_bernoulli} for each of $a' \in \actions = \{1, 2\}$ can be derived or upper-bounded as follows: 
\begin{align}
\label{eq:reward_bernoulli_1}
    \E\left[\max_{a \in \actions} \E[\theta_{t, a} | H_{t-2}^{\pi}, \theta_{t-1, 1}] \bigg| H_{t-2}^{\pi} \right]
    = \E\left[\max_{a \in \actions} \E[\theta_{t, a} | \theta_{t-1, 1}] \bigg| H_{t-2}^{\pi} \right], 
\end{align}
and 
\begin{align*}
\label{eq:reward_bernoulli_2}
    \E\left[\max_{a \in \actions} \E[\theta_{t, a} | H_{t-2}^{\pi}, \theta_{t-1, 2}] \bigg| H_{t-2}^{\pi} \right]
 \stackrel{(a)}{=} &\ \E\left[\max_{a \in \actions} \E[\theta_{t, a} | H_{t-2}^{\pi}] \bigg| H_{t-2}^{\pi} \right]\\
  = &\ \max_{a \in \actions} \E[\theta_{t, a} | H_{t-2}^{\pi}]  \\
   \stackrel{(b)}{=} &\ \max_{a \in \actions} \E\left[\E\left[\theta_{t, a} | \theta_{t-2}\right]| H_{t-2}^{\pi}\right]  \\
  \stackrel{(c)}{\leq} &\ \E\left[\max_{a \in \actions} \E\left[\theta_{t, a} | \theta_{t-2}\right] \bigg| H_{t-2}^{\pi}\right] \\
    \stackrel{(d)}{=} &\ \E\left[\max_{a \in \actions} \E\left[\theta_{t, a} | \theta_{t-2, 1}\right] \bigg| H_{t-2}^{\pi}\right],  \numberthis
\end{align*}
where $(a)$ follows from $q_2 = 1$, $(b)$ follows from that $\theta_t$ is independent of $H_{t-2}^{\pi}$ conditioned on $\theta_{t-2}$ (recall that $H_{t-2}^{\pi} = (A_1^{\pi}, R_{1, A_1^{\pi}}, ... A_{t-2}^{\pi}, R_{t-2, A_{t-2}^{\pi}}) = (A_1^{\pi}, \theta_{1, A_0^{\pi}}, ... A_{t-2}^{\pi}, \theta_{t-2, A_{t-2}^{\pi}})$), $(c)$ from Jensen's inequality, and $(d)$ again from $q_2 = 1$. 
Subtracting \eqref{eq:reward_bernoulli} from \eqref{eq:baseline_bernoulli}, we establish a lower bound on the instantaneous regret:  
\begin{align*}
     \E[R_{t, \overline{*}} - R_{t, A_t^{\pi}}] 
 \geq &\ \E\left[\sum_{a' \in \actions}\E\left[\max_{a \in \actions} \E[\theta_{t,a} | \theta_{t-1, 1}] - \max_{a \in \actions} \E[\theta_{t, a} | H_{t-2}^{\pi}, \theta_{t-1, a'}] \bigg| H_{t-2}^{\pi} \right] \Pr(A_{t-1}^{\pi} = a' | H_{t-2}^{\pi})\right] \\
  \stackrel{(a)}{\geq} &\ \E\left[\E\left[\max_{a \in \actions} \E[\theta_{t,a} | \theta_{t-1, 1}] - \max_{a \in \actions} \E[\theta_{t, a} | \theta_{t-2, 1}] \bigg| H_{t-2}^{\pi} \right] \Pr(A_{t-1}^{\pi} = 2 | H_{t-2}^{\pi})\right]\\
    \stackrel{(b)}{=} &\ \E\left[\frac{1}{16} \Pr(A_{t-1}^{\pi} = 2 | H_{t-2}^{\pi})\right] \\
     \stackrel{}{=} &\ \frac{1}{16} \Pr(A_{t-1}^{\pi} = 2), \numberthis
\label{eq:lower_bound_1}
\end{align*}
where $(a)$ follows from \eqref{eq:reward_bernoulli_1} and \eqref{eq:reward_bernoulli_2}, and $(b)$ from computing the conditional expectation, which turns out to be independent of $H_{t-2}^{\pi}$. 

Below we derive another lower bound on the instantaneous regret. First, observe that for any policy $\pi$ and all $t \in \mathbb{Z}_+$, the reward collected at time $t$ is upper-bounded by:
\begin{align*}
    \E\left[R_{t, A_{t}^{\pi}}\right] 
    = &\ 
    \E\left[R_{t, A_{t}^{\pi}} | A_{t}^{\pi}= 1\right] \Pr\left(A_t^{\pi} = 1\right) + \E\left[R_{t, A_{t}^{\pi}} | A_{t}^{\pi}= 2\right] \Pr\left(A_t^{\pi} = 2\right) \\
     \stackrel{}{\leq} &\  \E\left[R_{t, 1} | A_{t}^{\pi}= 1\right] \Pr\left(A_t^{\pi} = 1\right) + \Pr\left(A_t^{\pi} = 2\right) \\
    \stackrel{}{\leq} &\ \E\left[R_{t, 1}\right] + \Pr\left(A_t^{\pi} = 2\right), 
\end{align*}
where both inequalities follow from that rewards are bounded in $[0, 1]$. 
Therefore, for any policy $\pi$ and all $t \in \mathbb{Z}_+$, 
the instantaneous regret can be lower-bounded as follows:
\begin{align*}
    \E\left[R_{t, \overline{*}} - R_{t, A_{t}^{\pi}}\right] 
    \geq &\ \E\left[R_{t, \overline{*}}\right] - \E\left[R_{t, 1}\right] - \Pr\left(A_t^{\pi} = 2\right)\\
    = &\ \frac{5}{8} - \frac{1}{2} - \Pr(A_t^{\pi} = 2) = \frac{1}{8} - \Pr(A_t^{\pi} = 2). \numberthis
\label{eq:lower_bound_2}
\end{align*}
Incorporating the two lower bounds on instantaneous regret established in \eqref{eq:lower_bound_1} and \eqref{eq:lower_bound_2}, respectively, we derive a lower bound on the cumulative regret: for any policy $\pi$, and $T \in \mathbb{Z}_{+}$, $T \geq 2$,  
\begin{align*}
    {\mathrm{Regret}}_{\mathrm{F}}(T; \pi)
    \geq &\ \max\left\{\sum_{t = 1}^{T-1} \frac{1}{16} \Pr(A_t^{\pi} = 2), \sum_{t = 1}^{T-2} \left[\frac{1}{8} - \Pr(A_t^{\pi} = 2)\right] \right\} \\
    \geq &\ \frac{16}{17}\sum_{t = 1}^{T-1} \frac{1}{16} \Pr(A_t^{\pi} = 2) + \frac{1}{17}\sum_{t = 1}^{T-1} \left[\frac{1}{8} - \Pr(A_t^{\pi} = 2)\right] \\
    = &\ \frac{1}{136} (T - 1) \\
    \geq &\ \frac{1}{272} T. \numberthis
\label{eq:lower_bound_final_1}
\end{align*}
For any policy $\pi$, and $T = 1$, 
\begin{align*}
    {\mathrm{Regret}}_{\mathrm{F}}(T; \pi) = &\ \E[R_{1, *}] - \E[R_{1, A_1^{\pi}}] \\
    \geq &\ \E[R_{1, *}] - \E\left[\max_{a \in \actions}\E[R_{1, a}]\right] \\
    = &\ \frac{5}{8} - \frac{1}{2} = \frac{1}{8} \geq \frac{1}{272} T.  \numberthis 
\label{eq:lower_bound_final_2}
\end{align*}
Combining \eqref{eq:lower_bound_final_1} and \eqref{eq:lower_bound_final_2}, we complete the proof.
\end{proof}

{
\section{Application of an Existing Theoretical Result to Modulated Bernoulli Bandits: Illustration of a Discussion in Section~\ref{sec:modulated_bernoulli}}
\label{appendix:existing_bounds}

Theorem~2 of \cite{besbes2019optimal} establishes a regret upper bound of Rexp3, under a different notion of regret which we denote by ${\mathrm{Regret}}_{\mathrm{D}}$. The result suggests that for all $T \in \mathbb{Z}_+$, 
\begin{align*}
{\mathrm{Regret}}_{\mathrm{D}}(T; \pi^{\mathrm{Rexp3}}) \leq \overline{C} (|\actions| \log |\actions|)^{1/3} V_T^{1/3} T^{2/3},
\end{align*}
where 
 $\overline{C}$ is a constant that does not depend on $|\actions|$, $V_T$, or $T$, and 
 $V_T$ measures the temporal variation of the mean reward sequence and is defined as 
 $V_T = \sum_{t = 1}^{T-1} \sup_{a \in \actions} |\theta_{t,a} - \theta_{t+1,a}|$.  
 
Applying this frequentist result to the modulated Bernoulli bandits introduced by Example~\ref{ex:modulated_bernoulli} yields a bound of 
 $\overline{C} (|\actions| \log |\actions|)^{1/3} \mathbb{E}[V_T^{1/3}] T^{2/3}$. We can lower-bound 
 $\mathbb{E}[V_T^{1/3}]$ as follows: for all $T \in \mathbb{Z}_+$, we have 
\begin{align*}
\E\left[V_T^{1/3}\right] 
= &\ \E\left[\left(\sum_{t = 1}^{T-1} \sup_{a \in \actions} |\theta_{t,a} - \theta_{t+1,a}|\right)^{1/3}\right]\\
= &\ \E\left[T^{1/3}\left(\frac{1}{T}\sum_{t = 1}^{T-1} \sup_{a \in \actions} |\theta_{t,a} - \theta_{t+1,a}|\right)^{1/3}\right]\\
\geq &\ \E\left[T^{1/3}\ \frac{1}{T}\sum_{t = 1}^{T-1} \left( \sup_{a \in \actions} |\theta_{t,a} - \theta_{t+1,a}|\right)^{1/3}\right] \\
= &\ T^{-2/3} \sum_{t = 1}^{T-1} \E\left[\left( \sup_{a \in \actions} |\theta_{t,a} - \theta_{t+1,a}|\right)^{1/3}\right] \\
= &\ T^{1/3} \E\left[\left( \sup_{a \in \actions} |\theta_{1,a} - \theta_{2,a}|\right)^{1/3}\right]. 
\end{align*}
This implies that applying the result established by \cite{besbes2019optimal} 
to the modulated Bernoulli bandits yields a bound of at least 
$\overline{C} (|\actions| \log |\actions|)^{1/3} \E\left[\left( \sup_{a \in \actions} |\theta_{1,a} - \theta_{2,a}|\right)^{1/3}\right] T$. This is linear in $T$ and increases as $q_a$ increases. 
}

{
\section{Analysis of a Class of Modulated Bernoulli Bandits: Example~\ref{ex:modulated_bernoulli_example}}
\label{appendix:example_random_TS}
This section presents simple derivations for a regret lower bound for a random policy that uniformly selects actions and a regret lower bound for TS in Example~\ref{ex:modulated_bernoulli_example}. 
We first restate the example below, which describes a class of modulated Bernoulli bandits parameterized by $q_2$.
\modulatedBernoulliBandit*

In each step, an agent that selects each action uniformly at random would collect an expected reward of 
$\frac{1}{2}\E[\theta_{t,1}] + \frac{1}{2}\E[\theta_{t,2}] = 0.7$. An optimal agent would achieve least $0.9$ in expected reward by selecting only action $1$. Therefore, a random policy has a regret lower bound of $0.2$. If PS attains regret of less than $0.2$, it would suggest that it is performing arm selection in a manner more intelligent than a random policy. 

Next, we consider a TS agent, who, at each timestep $t$, 
\begin{enumerate}
    \item \textbf{samples} $\hat{\theta}_{t,1}$ from posterior $\Pr(\theta_{t,1} \in \cdot | H_{t-1}^{\pi_{\mathrm{TS}}})$ and $\hat{\theta}_{t,2}$ from posterior $\Pr(\theta_{t,2} \in \cdot | H_{t-1}^{\pi_{\mathrm{TS}}})$, 
    \item \textbf{selects} the action that maximizes $\hat{\theta}_{t,a}$ for $a \in \{1, 2\}$. 
\end{enumerate}

If $q_2 > 0.2$, then in each step $t \in \mathbb{Z}_+$, TS collects an expected reward of 
\begin{align*}
\E\left[R_{t, A_t^{\pi_{\mathrm{TS}}}}\right]
= &\ \E\left[\E\left[R_{t, A_t^{\pi_{\mathrm{TS}}}}|H_{t-1}^{\pi_{\mathrm{TS}}}\right]\right]\\
= &\ \E\left[\E\left[R_{t,1}|H_{t-1}^{\pi_{\mathrm{TS}}}\right]\Pr(A_t^{\pi_{\mathrm{TS}}} = 1 |H_{t-1}^{\pi_{\mathrm{TS}}}) + \E\left[R_{t,2}|H_{t-1}^{\pi_{\mathrm{TS}}}\right]\Pr(A_t^{\pi_{\mathrm{TS}}} = 2 |H_{t-1}^{\pi_{\mathrm{TS}}}) 
 \right]\\
 = &\ \E\left[0.9 \Pr(\theta_{t,2} = 0 |H_{t-1}^{\pi_{\mathrm{TS}}}) + \Pr\left(\theta_{t,2} = 1|H_{t-1}^{\pi_{\mathrm{TS}}}\right)\Pr(\theta_{t,2} = 1 |H_{t-1}^{\pi_{\mathrm{TS}}}) 
 \right] \\
 \leq &\ \E\left[0.9 \Pr(\theta_{t,2} = 0 |H_{t-1}^{\pi_{\mathrm{TS}}}) + \left(1 - \frac{q_2}{2}\right)\Pr(\theta_{t,2} = 1 |H_{t-1}^{\pi_{\mathrm{TS}}}) 
 \right] \\
\leq &\  \E\left[0.9 \left(1 - \frac{q_2}{2}\right) + \left(1 - \frac{q_2}{2}\right) \frac{q_2}{2}) 
 \right] \\
 = &\ \left(0.9 + \frac{q_2}{2}\right) \left(1 - \frac{q_2}{2}\right), 
\end{align*}
where the last inequality follows from 
$0.9 > 1 - \frac{q_2}{2}$ when 
$q_2 > 0.2$. 
Consequently, a regret lower bound for a TS agent is $0.9 - (0.9 + q_2/2)(1 - q_2 / 2) = q_2^2/4 - q_2/20$, when $q_2 \geq 0.2$. If PS achieves a regret lower than this bound, it implies that PS outperforms TS.
}

\section{Efficient Implementation of PS in AR(1) Bandits: Proof of Proposition~\ref{proposition:ar1}}
\label{appendix:ar1}
\arderivation*
\begin{proof}
The analysis is done for PS and 
an arbitrary arm $a \in \actions$. We drop
the subscript $a$ from most of the random variables. 

For all $t \in \mathbb{Z}_+$, and $n \in \mathbb{Z}_{+}$, $n \geq 2$, let
\begin{align*}
    \overline{\theta}_{t}^{H}(n) =  \E[R_{t}|H_{t-1}, R_{t+1:t + n}].
\end{align*}
We define 
\begin{align*}\tilde{R}_{t+1} = R_{t+1}, \text{ and }
\tilde{R}_{t+i} = R_{t+i} - \gamma R_{t+i - 1} - (1 - \gamma) c
\text{ for all } i \in \{2, ... , n\}.
\end{align*}
Then we can rewrite $\overline{\theta}^{H}_{t}(n)$ as follows: 
\begin{align*}
    \overline{\theta}^{H}_{t}(n) =  \E\left[R_{t}|H_{t-1}, R_{t+1:t + n }\right]
    = &\ \E\left[R_{t}|H_{t-1}, \tilde{R}_{t+1:t + n }\right]. 
\end{align*}
Conditioned on $H_{t-1}$, for all $n \geq 2$,  $R_{t:t+n}$ is Gaussian, so the vector constructed by stacking $R_{t} $ and $\tilde{R}_{t+1:t+n}$ is also Gaussian. We use $\mu_n$ and $\Sigma_n$ to denote its mean and variance. In particular, we view $\mu_n$ as a block matrix with blocks $\mu_{n1} \in \mathbb{R}$ and $\mu_{n2} \in \mathbb{R}^{n}$ and $\Sigma_n$ a block matrix with blocks $\Sigma_{n11} \in \mathbb{R}$, $\Sigma_{n12} \in \mathbb{R}^{1 \times n}$, $\Sigma_{n21} \in \mathbb{R}^{n \times 1}$ and $\Sigma_{n22} \in \mathbb{R}^{n \times n}$. Observe that for all $n \geq 2$, conditioned on $H_{t-1}$, 
\begin{align*}
    \overline{\theta}_{t}^{H}(n)  
    =  \E\left[R_{t}|H_{t-1}, \tilde{R}_{t+1:t + n}\right] 
    =  \mu_{n1} + \Sigma_{n12} \Sigma_{n22}^{-1} \left(\tilde{R}_{t+1:t+n} - \mu_{n2}\right).
\end{align*}
Then, conditioned on $H_{t-1}$, $\overline{\theta}^{H}_{t}(n)$ is Gaussian with mean 
\begin{align*}
    \E\left[\overline{\theta}^{H}_{t}(n) \Big| H_{t-1}\right] = \mu_{n1} = \mu_{t, a}
\end{align*}
and variance
\begin{align*}
\V\left(\overline{\theta}^{H}_{t}(n) \Big| H_{t-1}\right) = \Sigma_{n12}\Sigma_{n22}^{-1} \Sigma_{n22} \Sigma_{n22}^{-1}\Sigma_{n21} = \Sigma_{n12}\Sigma_{n22}^{-1}\Sigma_{n21}.
\end{align*}
Observe that for all
$t \in \mathbb{Z}_+$ and 
$k \in \mathbb{Z}_{+}$, $k \geq 2$:
\begin{align*}
    \tilde{R}_{t+ k + 1} = R_{t + k + 1} - \gamma R_{t+ k} - (1 - \gamma)c =  W_{t + k+1} + Z_{t + k + 1} - \gamma Z_{t+k}.
\end{align*}
Then we have for all $t \in \mathbb{Z}_+$:
\begin{enumerate} [(i)]
    \item $\V(R_{t}, \tilde{R}_{t+1} | H_{t-1}) = \V(R_{t}, {R}_{t+1} | H_{t-1}) = \V(\theta_{t} + Z_{t}, \gamma \theta_{t} + W_{t+1} + Z_{t+1} | H_{t-1}) = \gamma \V(\theta_{t} | H_{t-1}) = \gamma \sigma_{t}^2$;
    \item $\V(R_{t}, \tilde{R}_{t+k} | H_{t-1|}) = \V(\theta_{t} + Z_{t}, W_{t+k} + Z_{t+k } - \gamma Z_{t+k-1} | H_{t-1}) = 0$ for all $k \geq 2$, $k \in \mathbb{Z}_+$;  
    \item $\V(\tilde{R}_{t+1} | H_{t-1}) = \V(R_{t+1} | H_{t-1}) = \V(\gamma \theta_{t} + W_{t+1} + Z_{t+1} | H_{t-1}|) = \gamma^2 \sigma_t^2 + \delta^2 + \sigma^2$; 
    \item $\V(\tilde{R}_{t + k } | H_{t-1}) = \V(W_{t + k} + Z_{t + k } - \gamma Z_{t + k-1} | H_{t-1}) = \delta^2 + \sigma^2 + \gamma^2 \sigma^2$, for all $k \geq 2$, $k \in \mathbb{Z}_+$;
    \item $\V(\tilde{R}_{t + i }, \tilde{R}_{t + k } | H_{t-1}) = \V(W_{t+i} + Z_{t+i} - \gamma Z_{t+i-1}, W_{t + k} + Z_{t + k } - \gamma Z_{t+k-1} | H_{t-1})$ for all $k > i \geq 1$. Hence, $\V(\tilde{R}_{t + i }, \tilde{R}_{t + k } | H_{t-1}) = \gamma \sigma^2$ for $k = i + 1$, $i \geq 1$, and $\V(\tilde{R}_{t + i }, \tilde{R}_{t + k } | H_{t-1}) = 0$ for $k \geq i + 2$, $i \geq 1$. 
\end{enumerate}
Based on these derivations, 
\begin{align*}
  \Sigma_{n12} =  \gamma \sigma_{t}^2 
    \begin{bmatrix}
    1 & 0 & 0 & ... & 0
    \end{bmatrix},
\end{align*}
and 
\begin{align*}
  \Sigma_{n22} = \begin{bmatrix}
    \gamma^2 \sigma_{t}^2 + \delta^2 + \sigma^2 & \gamma \sigma^2 & 0 & ... & 0\\
    \gamma \sigma^2 & \delta^2 + (1 + \gamma^2) \sigma^2 & \gamma \sigma^2 & ... & 0 \\
    0 & \gamma \sigma^2 & \delta^2 + (1 + \gamma^2) \sigma^2 & ... & 0 \\
    \vdots & \vdots & \vdots &  & \vdots \\
    0 & 0 & 0 & ...  & \delta^2 + (1 + \gamma^2) \sigma^2
    \end{bmatrix}. 
\end{align*}
We use Gaussian elimination to compute the inverse of $\Sigma_{n22}$. In particular, for $k = 1, 2, ... , n - 1$, we perform row operations on the $(k+1)$-th to last row: Subtract $r_k$ times $k$-th to the last row from the $(k+1)$-th to last row. The sequence $\{r_k\}$ are  
such that 
the matrix becomes lower-triangular after the $n - 1$ row operations.
If we use $d_k$ to denote the diagonal entry of the matrix on the $k$-th to last row after these row operations. Then the sequence $\{d_k\}$ satisfies the following recurrence:
\begin{align*}
    d_1 = &\ \delta^2 + (1 + \gamma^2) \sigma^2, \\
    d_{k} = &\ \delta^2 + (1 + \gamma^2) \sigma^2 - \frac{\gamma^2 \sigma^4}{d_{k-1}},\ k = 2, ..., n-1, \\
    d_n = &\  \gamma^2(\sigma_t^2 - \sigma^2)+ \delta^2 + (1 + \gamma^2) \sigma^2 - \frac{\gamma^2 \sigma^4}{d_{n-1}}.
\end{align*}
Note that the recurrence induces the following fixed-point equation:
\begin{align*}
    d_* = \delta^2 + (1 + \gamma^2) \sigma^2 - \frac{\gamma^2 \sigma^4}{d_*}. 
\end{align*}
Solving for $d_*$, we have 
\begin{align*}
    d_* = &\ \frac{1}{2} \left(\gamma^2 \sigma^2 + \sigma^2 + \delta^2 \pm \sqrt{(\gamma^2 \sigma^2 + \sigma^2 + \delta^2)^2 - 4\gamma^2 \sigma^4} \right) \\
    = &\ \frac{1}{2} \left(\gamma^2 \sigma^2 + \sigma^2 + \delta^2 \pm \sqrt{(\delta^2 + \sigma^2 - \gamma^2 \sigma^2)^2 + 4\gamma^2 \delta^2 \sigma^2} \right).
\end{align*}
Then for all $t \in \mathbb{Z}_+$, the variance
\begin{align*}
\V\left(\hat{\theta}_{t, a} | H_{t-1}\right) = &\ 
    \V\left(\overline{\theta}^{H}_{t, a} | H_{t-1}\right) = \lim_{n \rightarrow \infty} \V\left(\overline{\theta}^{H}_{t, a}(n) | H_{t-1}\right)\\
    = &\ \lim_{n \rightarrow \infty} \Sigma_{n12}\Sigma_{n22}^{-1}\Sigma_{n21} 
    = \frac{\gamma_a^2 \sigma_{t, a}^4}{d_* + \gamma_a^2 (\sigma_{t, a}^2 - \sigma^2)} 
    = \frac{\gamma_a^2 \sigma_{t, a}^4}{\gamma_a^2 \sigma_{t, a}^2 + x_a^*}, 
\end{align*}
where 
\begin{align*}x_a^* = \frac{1}{2} \left(\delta_a^2 + \sigma_a^2 - \gamma_a^2 \sigma_a^2 + \sqrt{(\delta_a^2 + \sigma_a^2 - \gamma_a^2 \sigma_a^2)^2 + 4 \gamma_a^2 \delta_a^2 \sigma_a^2} \right).
\end{align*}
\end{proof}

\section{Tractable Variant of PS in AR(1) Logistic Bandits}
\label{appendix:ar_logistic}
We next consider a class of bandits which we refer to as AR(1) logistic bandits. Just like the AR(1) bandits, the reward distributions of these bandits are governed by a sequence $\{\alpha_t\}_{t \in \mathbb{Z}_+}$ that evolves according to an AR(1) process, but in this case the rewards will be Bernoulli. This class extends bandits modulated by AR(1) processes to logistic bandits. We study AR(1) logistic bandits because it serves as a first step towards implementing PS in contextual bandits and further in practical problems.

\begin{example} [\bf{AR(1) Logistic Bandit}]
\label{ex:logistic_ar1} 
In an AR(1) logistic bandit, each reward $R_{t, a}$ is Bernoulli distributed with mean $\frac{\exp(\alpha_t^{\top}\phi_a)}{1 + \exp(\alpha_t^{\top} \phi_a)}$, where $\alpha_t \in \R^d$ is unknown with a known dimension $d \in \mathbb{Z}_{+}$ and 
$\phi_a \in \R^d$ denotes a known feature vector associated with action $a \in \actions$. The variable $\alpha_{t,a}$ is defined exactly as in Example~\ref{ex:autoregressive}, transitioning following an AR(1) process.  
\end{example}

Similarly, the formulation of AR(1) logistic bandits accommodates stationary logistic bandits as a special case. In particular,
following a similar argument to that of AR(1) bandits, 
we can model any stationary logistic bandit using an AR(1) logistic bandit with $\gamma_a = 1$ and $\delta_a = 0$ for $a \in \actions$ and suitably-chosen $\alpha_{1, a}$ for each $a \in \actions$. As was the case with the AR(1) bandits, we assume that an agent knows a priori 
$c_a$, $\gamma_a$, $\delta_a$, $q_a$, $\Pr(\theta_{1,a} \in \cdot)$, and $\sigma_a$ for all $a \in \actions$.

\subsection{Approximate PS in AR(1) Logistic Bandits}  
\label{sec:logistic_implementation}
We next discuss techniques to approximate PS and apply those techniques to AR(1) logistic bandits to develop efficient implementations.  

\paragraph{Incremental Laplace Approximation And Approximate TS} 
We first introduce a technique that is useful in developing tractable procedues to approximate PS in AR(1) logistic bandits. 
For better exposition, we will first demonstrate this technique in the context of implementing an approximation of TS in AR(1) logistic bandits. 
In an AR(1) logistic bandit, a natural learning target is $\alpha_t$. We thus focus on TS that takes $\alpha_t$ as its learning target. 
At each timestep, TS samples $\hat{\alpha}_t^{\pi_{\mathrm{TS}}}$ from the posterior of $\alpha_t$, and selects an action that maximizes the expected reward conditioned on the sample.

Because deriving the posterior of $\alpha_t$ is usually intractable, we can instead apply Laplace approximation \citep{laplace1986memoir}.
That is, we approximate the posterior using a Gaussian distribution centered at the maximum a posteriori (MAP) of $\alpha_t$, 
with a variance that is equal to the inverse of the Hessian of the log-posterior. This is a standard practice and has been popular with stationary logistic bandits and contextual bandits \citep{NIPS2011_e53a0a29}.

However, when the environment is non-stationary, applying the method to approximate the posterior of $\alpha_t$ can be computationally onerous. To address this, we propose what we call incremental Laplace approximation---the practice of approximating the posterior distribution incrementally at each timestep using Laplace approximation. 
Incremental Laplace approximation is comparable with the standard Laplace approximation in stationary logistic bandits, but can be efficiently carried out in non-stationary ones.

\begin{centering}
\begin{algorithm}
\caption{Approximate TS in an AR(1) logistic bandit}\label{alg:thompson-sampling-logistic}
\For{$t = 1, 2, \ldots, T$}{
\textbf{sample}: $\hat{\alpha}_t$ from 
$
\mathcal{N}(\mu_t^{ }, \Sigma_t^{ })
$\\
\textbf{estimate}: $\hat{\theta}_t^{ } = \phi^{\top} \hat{\alpha}_t^{ }$ \\
\textbf{select}: $A_t^{ } \in \argmax_{a \in \actions} \hat{\theta}_{t, a}^{ }$ \\
\textbf{observe}: $R_{t,A_t^{ }}$ \\
\textbf{derive}: 
    $\mu_{t+1}^{ } \leftarrow \min_{\alpha} \left\{ \frac{1}{2}(\alpha - \mu_{t}^{ })^{\top} \Sigma_{t}^{ -1} (\alpha - \mu_{t}^{ }) - R_{t, A_t^{ }} \phi_{A^{ }_t}^{} \alpha + \log\left(1 + \exp\left(\phi_{A_t^{ }}^{} \alpha\right)\right)\right\}$, 
and 
  $\Sigma_{t+1}^{} \leftarrow \left\{\Sigma_{t}^{} + {\exp(\phi_{A_t^{}}^{} \mu_{t+1}^{})}
    {\left[1 + \exp\left(\phi_{A_t^{}}^{} \mu_{t+1}^{}\right)\right]^{-2}} \phi_{A_t^{}}^{\top} \phi_{A^{}_t}^{}\right\}^{-1}$
\\
\textbf{update}: $\mu_{t+1} \leftarrow A \mu_{t+1} \text{ and }
    \Sigma_{t+1} \leftarrow A \Sigma_{t+1} A^{\top} + V$ \\

}
\end{algorithm}
\end{centering}

Applying incremental Laplace approximation, we develop an efficient implementation of an approximation of TS in an AR(1) logistic bandit. It is detailed in Algorithm~\ref{alg:thompson-sampling-logistic}; 
for the sake of simplicity, we drop the superscript $\pi_{\mathrm{TS}}$. 
Step~6 carries out incremental Laplace approximation. In Step~7, $A$ denotes a diagonal matrix with $\gamma_a$ at its $a$-th position along its diagonal, and $V$ a diagonal matrix with $\delta_a^2$ at its $a$-th position along its diagonal.

\paragraph{Finite-sample Approximation} Another technique that is useful in constructing computationally tractable approximations of PS is to sample a finite number of rewards instead of an infinite sequence of rewards in executing the algorithm. In Algorithm~\ref{alg:predictive-sampling}, this corresponds to changing Steps~2 and~3 to:
\begin{enumerate}
    \item \textbf{sample} 
    $\hat{R}^{(t)}_{t+1:t+n} \sim \Pr(R_{t+1:t+n} \in \cdot | H_{t-1})$, 
    \item \textbf{estimate} 
    $\hat{\theta}_{t} = \E[R_{t} | H_{t-1}, R_{t+1:t+n} \leftarrow \hat{R}^{(t+1)}_{t+1:t+n}]$. 
\end{enumerate}

The new Steps~2 and~3 are equivalent to sampling $\hat{\theta}_t$ from the distribution $\Pr(\E[R_{t} | H_{t-1}, R_{t+1:t+n}] \in \cdot | H_{t-1})$. To see why, first observe that for all $t \in \mathbb{Z}_+$,  $\Pr(\hat{R}^{(t)}_{t+1:t+n} \in \cdot | H_{t-1}) = \Pr(R_{t+1:t+n} \in \cdot | H_{t-1})$. Therefore, for all $t \in \mathbb{Z}_+$, 
\begin{align*}
\Pr(\hat{\theta}_t \in \cdot | H_{t-1}) = \Pr(\E[R_{t} | H_{t-1}, R_{t+1:t+n} \leftarrow \hat{R}^{(t)}_{t+1:t+n}] \in \cdot | H_{t-1})
= \Pr(\E[R_{t} | H_{t-1}, R_{t+1:t+n}] \in \cdot | H_{t-1}).
\end{align*}

\paragraph{Gaussian Imagination}
If the sampling distribution $\Pr(\E[R_{t} | H_{t-1}, R_{t+1:t+n}] \in \cdot | H_{t-1})$  of an approximation of PS have closed-form solutions, then we can 
design efficient procedures to execute the algorithm. Regardless of whether this is computationally tractable, an agent can approximate this distribution by pretending that the rewards are Gaussian. This practice is called Gaussian imagination \citep{liu2022gaussian}. Specifically, an agent can construct imaginary Gaussian rewards $\tilde{R}_{t}$ for all $t \in \mathbb{Z}_+$; the imaginary rewards live in the agent's imagination and are designed to approximate real rewards. The agent then estimates 
$\Pr(\E[R_{t} | H_{t-1}, R_{t+1:t+n}] \in \cdot | H_{t-1})$ using $\Pr(\E[\tilde{R}_{t} | H_{t-1}, \tilde{R}_{t+1:t+n}] \in \cdot | H_{t-1})$.

\paragraph{Approximate PS} 
We next apply finite-sampling approximation, Gaussian imagination, and incremental Laplace approximation to approximate PS in AR(1) logistic bandits. Specifically, this approximation algorithm, which we refer to as approximate PS, samples $\hat{\theta}_t$ from the an approximate posterior of 
$\E[\tilde{R}_{t} | H_{t-1}, \tilde{R}_{t+1:t+n}]$, and selects an action that maximizes $\hat{\theta}_{t,a}$. In applying Gaussian imagination, 
we let 
$
\tilde{R}_{t} \sim \mathcal{N}\left(\frac{1}{2} \mathbf{1} + \frac{1}{4}\phi^{\top} \alpha_t, \frac{1}{16}I\right),
$
where $\mathbf{1}$ is an all-one vector and $\phi$ is the matrix whose $a$-th column is $\phi_a$. With this definition of $\tilde{R}_{t}$, $\E[\tilde{R}_{t} | H_{t-1}, \tilde{R}_{t+1:t+n}] = \frac{1}{2} \mathbf{1} + \frac{1}{4} \phi^{\top} \E[\alpha_{t} | H_{t-1}, \tilde{R}_{t+1:t+n}]$. Therefore, equivalently, approximate PS can sample $\hat{\alpha}_t$ from an approximate posterior distribution of $\E[\alpha_t | H_{t-1}, \tilde{R}_{t+1:t+n}]$ and select an action that maximizes $\phi^{\top} \hat{\alpha}_{t,a}$.

Applying incremental Laplace approximation, we construct the approximate posterior using a Gaussian distribution $\mathcal{N}(\mu'_t, \Sigma'_t)$. 
More precisely, we first approximate the posterior of $\alpha_t$ using a Gaussian distribution $\mathcal{N}(\mu_t, \Sigma_t)$ according to incremental Laplace approximation, and then derive $\mu_t'$ and $\Sigma_t'$ based on $\mu_t$ and $\Sigma_t$. 
Detailed steps about this derivation are provided in  Appendix~\ref{appendix:laplace_examples}.  

The implementation of approximate PS in an AR(1) logistic bandit is presented in  Algorithm~\ref{alg:predictive-sampling-logistic}. 
It is worth highlighting that the only difference between approximate PS and approximate TS in an AR(1) logistic bandit is that between the sampling distribution of $\hat{\alpha}_t$ and $\hat{\alpha}_t^{\pi_{\mathrm{TS}}}$. In Algorithm~\ref{alg:predictive-sampling-logistic}, Step~6 again carries out incremental Laplace approximation to incrementally derive $\mu_t$ and $\Sigma_t$. As we have discussed, Appendix~\ref{appendix:laplace_examples} presents details on how we can derive $\mu'_t$ and $\Sigma'_t$ based on $\mu_t$ and $\Sigma_t$.

\begin{centering}
\begin{algorithm}
\caption{Approximate PS in an AR(1) logistic bandit}\label{alg:predictive-sampling-logistic}
\For{$t = 1, 2, \ldots, T$}{
\textbf{sample}: $\hat{\alpha}_t$ from 
$
\mathcal{N}(\mu'_t, \Sigma'_t)
$ \\
\textbf{estimate}: $\hat{\theta}_t = \phi^{\top} \hat{\alpha}_t$ \\
\textbf{select}: $A_t \in \argmax_{a \in \actions} \hat{\theta}_{t, a}$ \\
\textbf{observe}: $R_{t,A_t}$ \\
\textbf{derive}: 
    $\mu_{t+1} \leftarrow \min_{\alpha} \left\{ \frac{1}{2}(\alpha - \mu_{t})^{\top} \Sigma_{t}^{-1} (\alpha - \mu_{t}) - R_{t, A_t} \phi_{A_t}^{} \alpha + \log\left(1 + \exp\left(\phi_{A_t}^{} \alpha\right)\right)\right\}$, 
   and 
    $\Sigma_{t+1} \leftarrow \left\{\Sigma_{t}^{-1} + {\exp(\phi_{A_t}^{} \mu_{t+1})}
    {\left[1 + \exp\left(\phi_{A_t}^{} \mu_{t+1}\right)\right]^{-2}} \phi_{A_t}^{\top} \phi_{A_t}^{}\right\}^{-1}$
\\
\textbf{update}: $\mu_{t+1} \leftarrow A \mu_{t+1}, 
    \Sigma_{t+1} \leftarrow A \Sigma_{t+1} A^{\top} + V$, and update $\mu'_t$ and $\Sigma'_t$ according to incremental Laplace approximation
}
\end{algorithm}
\end{centering}

\subsection{Experiments}
We now conduct experiments to examine the performance of approximate PS. 
Since the experiment design and the analyses are similar to that concerning PS in Section~\ref{sec:ar1_experiments}, some repetition is expected but is kept to a minimum. 
In particular, we compare approximate PS with approximate TS in a sequence of AR(1) logistic bandits where $\alpha_{t,1}$ is associated with information of varying durability. We also examine both the case where the actions are independent and the case where the actions are dependent. Specifically, we let $\actions = \{1, 2, 3\}$, the stationary distribution of each $\alpha_{t,k}$ be $\mathcal{N}(0, 1)$, and $\gamma = [\gamma_1, 0.9, 0.9]$, 
where $\gamma_1 \in \{0.1, 0.9, 0.99\}$. 
In addition, we let $\phi \in \{\phi^{\text{ind}}, \phi^{\mathrm{dep}}\}$, where
\begin{align*}
    \phi^{\mathrm{ind}} = 
    \begin{bmatrix}
    1 & 0 & 0 \\
    0 & 1 & 0 \\
    0 & 0 & 1
    \end{bmatrix} \text{ and } 
    \phi^{\mathrm{dep}} = 
   \begin{bmatrix}
    0.9 & 0.1 & 0 \\
    0 & 0.9 & 0.1 \\
    0.1 & 0 & 0.9
    \end{bmatrix}. 
\end{align*}
Note that $\gamma_1$ determines the durability of the information associated with $\alpha_{t,1}$; 
the feature matrix $\phi$ determines weather the actions are independent, with $\phi^{\mathrm{ind}}$ indicating that the actions are independent, and 
$\phi^{\mathrm{dep}}$ indicating that the actions are dependent. 
Therefore, the set of parameters define a range of three-armed AR(1) logistic bandits with varying durability of information associated with $\alpha_{t,1}$, and accommodates both the case where the actions are independent and the case where the actions are dependent.

\paragraph{Approximate PS Outperforms Approximate TS Across Bandits}
Figure~\ref{fig:logistic_scatter} plots the average rewards collected by approximate PS and that by approximate TS over a long duration of $T = 1000$ timesteps, to study the long-run performance of the algorithms. The plot shows that approximate PS consistently outperforms approximate TS. These results support the efficacy of approximate PS and the usefulness of the techniques we introduced to approximate PS.

\paragraph{Approximate PS Outperforms Approximate TS Across Time}
Similar to how we investigate PS, we also investigate if approximate PS sacrifices its short-term performance for long-term benefits. We focus on an example in the aforementioned AR(1) logistic bandits, with $\gamma_1 = 0.1$ and $\phi = \phi^{\mathrm{ind}}$. Figure~\ref{fig:logistic_example} plots
the average reward collected by approximate PS and that collected by approximate TS over $t \in \{1, 2, ... , 200\}$ timesteps, with the error bars representing $95\%$ confidence intervals. 
The results suggest that approximate PS outperforms approximate TS across time. 

\begin{figure} [!ht]
\centering
\begin{subfigure}[b]{0.47\textwidth}
\centering
\includegraphics[width=\textwidth]{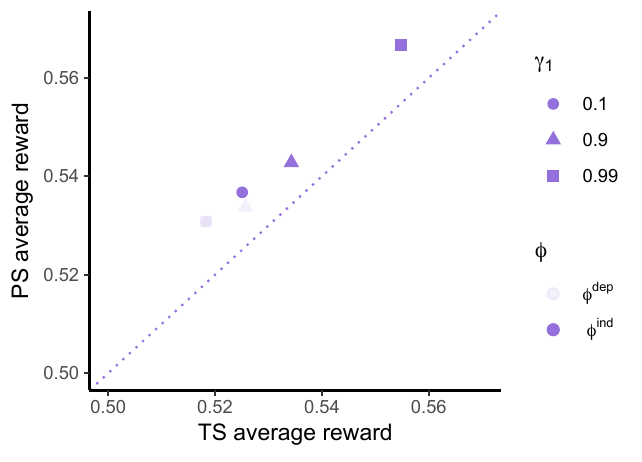}
\caption{The average rewards collected over $T = 1000$ timesteps in the environments where $\phi \in \{\phi^{\mathrm{ind}}, \phi^{\mathrm{dep}}\}$, $\gamma_1 \in \{0.1, 0.9, 0.99\}$}
\label{fig:logistic_scatter}
\end{subfigure}
\hfill
\begin{subfigure}[b]{0.47\textwidth}
\centering
\includegraphics[width=\textwidth]{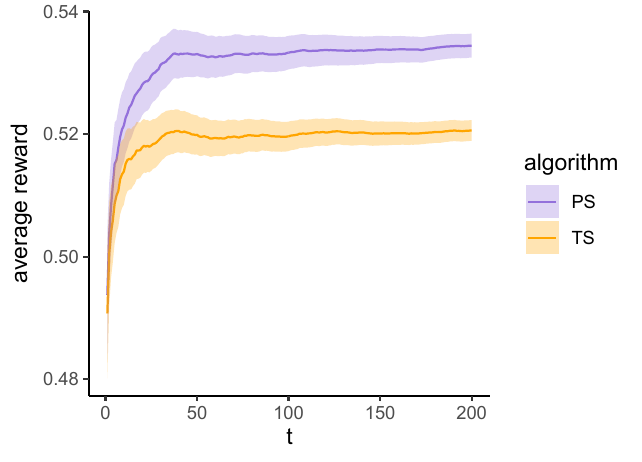}
\caption{An example: the average rewards collected over $t \in \{1, ... , 200\}$ timesteps in an  environment where $\gamma_1 = 0.1$, $\phi = \phi^{\text{ind}}$}
\label{fig:logistic_example}
\end{subfigure}
\caption{The average rewards collected by approximate PS and that collected by approximate TS in AR(1) logistic bandits}
\label{figure:ar1_laplace_gi}
\end{figure}

\subsection{More Details on Implementation of Algorithm~\ref{alg:predictive-sampling-logistic}}
\label{appendix:laplace_examples}
In implementing Algorithm~\ref{alg:predictive-sampling-logistic}, we present detailed steps to derive $\mu_t^{'}$ and $\Sigma_t^{'}$ based on $\mu_t$ and $\Sigma_t$ as follows: 
$\mu_t^{'} = \mu_t$, $\Sigma_t^{'} = \Sigma_{12} \Sigma_{22}^{-1} \Sigma_{21}$, and 

\begin{align*}
    \Sigma_{21} = \frac{1}{4}
    \begin{bmatrix}
    \phi A \Sigma_t \\
    \phi A^2 \Sigma_t \\
    ... \\
    \phi A^n \Sigma_t
    \end{bmatrix}, 
    \Sigma_{12} = \Sigma_{21}^{\top}, 
    \Sigma_{22} 
    = \frac{1}{16}
    \begin{bmatrix}
    \phi \tilde{\Sigma}_{t+1}^{(t)} \phi^{\top} + I & \phi \tilde{\Sigma}_{t+1}^{(t)} A^{\top} \phi^{\top} & \phi \tilde{\Sigma}_{t+1}^{(t)} A^{2 \top}\phi^{\top} & \ldots & \phi \tilde{\Sigma}_{t+1}^{(t)} A^{n-1 \top}\phi^{\top}  \\
    \phi A \tilde{\Sigma}_{t+1}^{(t)} \phi^{\top} & \phi \tilde{\Sigma}_{t+2}^{(t)} \phi^{\top} + I & \phi \tilde{\Sigma}_{t+2}^{(t)} A^{\top} \phi^{\top} & \ldots
    & \phi \tilde{\Sigma}_{t+2}^{(t)} A^{n - 2 \top} \phi^{\top} \\
    \phi A^2 \tilde{\Sigma}_{t+1}^{(t)} \phi^{\top} & 
    \phi A \tilde{\Sigma}_{t+2}^{(t)} \phi^{\top} & 
    \phi \tilde{\Sigma}_{t+3}^{(t)} \phi^{\top} + I & \ldots & \phi \tilde{\Sigma}_{t+3}^{(t)} A^{n-3 \top} \phi^{\top} \\
    \vdots & \vdots & \vdots & \ddots & \vdots \\
    \phi A^{n-1} \tilde{\Sigma}_{t+1}^{(t)} \phi  
    & \phi A^{n-2} \tilde{\Sigma}_{t+2}^{(t)} \phi
    & \phi A^{n-3} \tilde{\Sigma}_{t+3}^{(t)} \phi
    & \ldots 
    & \phi \tilde{\Sigma}_{t+n}^{(t)} \phi^{\top} + I 
    \end{bmatrix}.
\end{align*}

The matrices $\tilde{\Sigma}_{t}^{(t)}, \tilde{\Sigma}_{t+1}^{(t)}, ... , \tilde{\Sigma}_{t+n}^{(t)}$ can be derived from the following recursion:
\begin{align*}
    \tilde{\Sigma}_{t}^{(t)}  = &\ \Sigma_t, \\
    \tilde{\Sigma}_{t+i + 1}^{(t)} = &\ A \tilde{\Sigma}_{t + i}^{(t)} A^{\top} + V,\ i = 0, 1, 2, ... , n - 1. 
\end{align*}

\section{Additional Experiments}
\label{appendix:ar1_experiments}
We conduct additional experiments in AR(1) bandits with $\gamma_1 \in \{0.1, 0.3\}$ and $\gamma_2 \in \{0.1, 0.5, 0.9\}$. The results suggest that PS outperforms TS across time consistently in all bandits. 
\begin{figure} [!ht]
\centering
\begin{subfigure}[b]{0.45\textwidth}
\centering
\includegraphics[width=\textwidth]{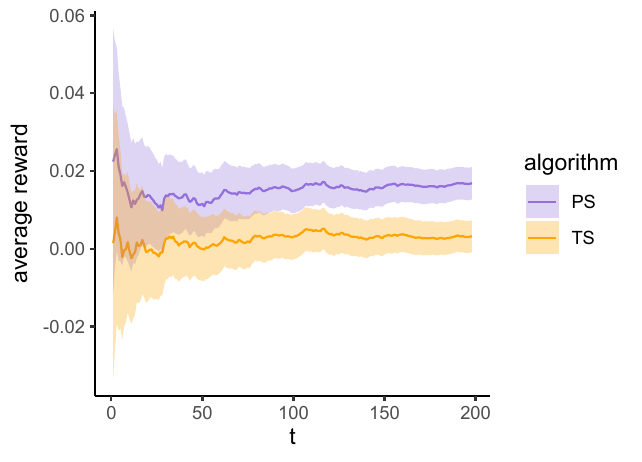}
\caption{$\gamma_1 = 0.1, \gamma_2 = 0.1$}
\label{fig:ar1_plot_1_1}
\end{subfigure}
\hfill
\begin{subfigure}[b]{0.45\textwidth}
\centering
\includegraphics[width=\textwidth]{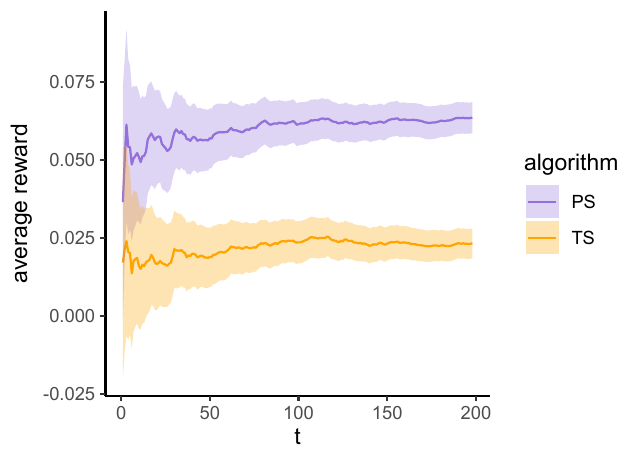}
\caption{$\gamma_1 = 0.1, \gamma_2 = 0.5$}
\label{fig:ar1_plot_1_2}
\end{subfigure}
\hfill
\begin{subfigure}[b]{0.45\textwidth}
\centering
\includegraphics[width=\textwidth]{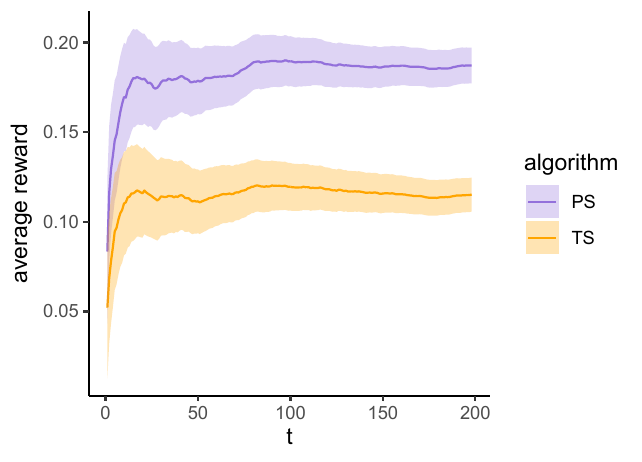}
\caption{$\gamma_1 = 0.1, \gamma_2 = 0.9$}
\label{fig:ar1_plot_1_3}
\end{subfigure}
\hfill
\begin{subfigure}[b]{0.45\textwidth}
\centering
\includegraphics[width=\textwidth]{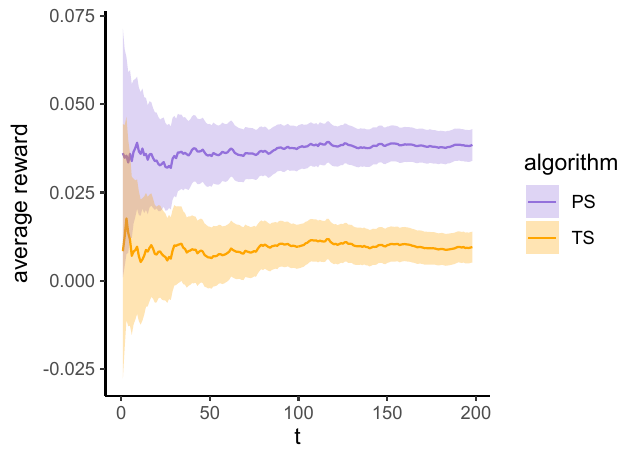}
\caption{$\gamma_1 = 0.3, \gamma_2 = 0.1$}
\label{fig:ar1_plot_1_4}
\end{subfigure}
\hfill
\begin{subfigure}[b]{0.45\textwidth}
\centering
\includegraphics[width=\textwidth]{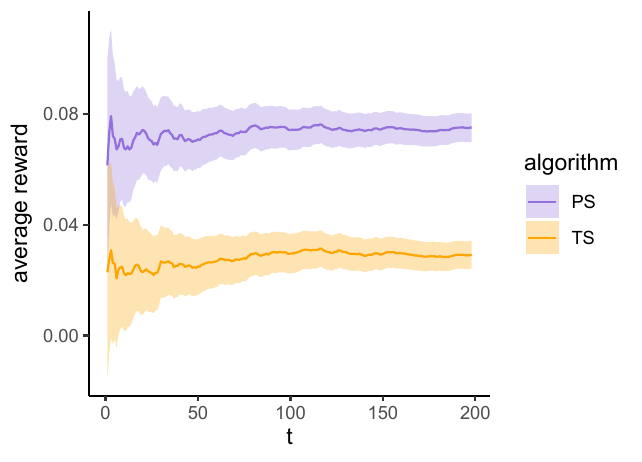}
\caption{$\gamma_1 = 0.3, \gamma_2 = 0.5$}
\label{fig:ar1_plot_1_5}
\end{subfigure}
\hfill
\begin{subfigure}[b]{0.45\textwidth}
\centering
\includegraphics[width=\textwidth]{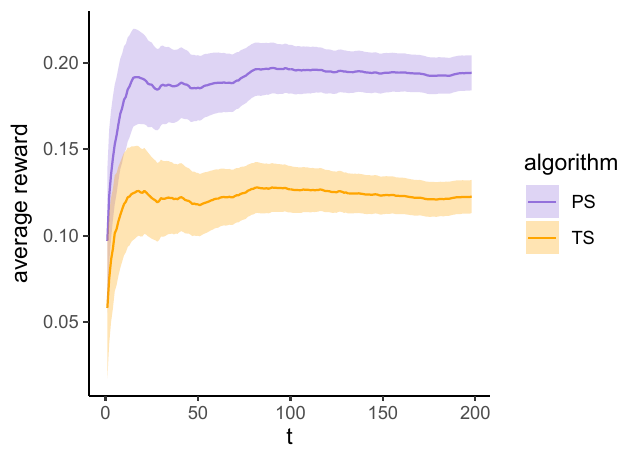}
\caption{$\gamma_1 = 0.3, \gamma_2 = 0.9$}
\label{fig:ar1_plot_1_6}
\end{subfigure}
\caption{Average reward collected by PS and that collected by TS in AR(1) bandits}
\label{fig:ar1_plot_2}
\end{figure}

%% file: arxiv/references.bib
@article{varaiya1985extensions,
  title={Extensions of the multiarmed bandit problem: The discounted case},
  author={Varaiya, Pravin and Walrand, Jean and Buyukkoc, Cagatay},
  journal={IEEE transactions on automatic control},
  volume={30},
  number={5},
  pages={426--439},
  year={1985},
  publisher={IEEE}
}

@inproceedings{min2023information,
  title={An information-theoretic analysis of nonstationary bandit learning},
  author={Min, Seungki and Russo, Daniel},
  booktitle={International Conference on Machine Learning},
  pages={24831--24849},
  year={2023},
  organization={PMLR}
}

@article{levine2017rotting,
  title={Rotting bandits},
  author={Levine, Nir and Crammer, Koby and Mannor, Shie},
  journal={Advances in neural information processing systems},
  volume={30},
  year={2017}
}

@article{mandelbaum1987continuous,
  title={Continuous multi-armed bandits and multiparameter processes},
  author={Mandelbaum, Avi},
  journal={The Annals of Probability},
  volume={15},
  number={4},
  pages={1527--1556},
  year={1987},
  publisher={Institute of Mathematical Statistics}
}

@article{mandelbaum1986discrete,
  title={Discrete multi-armed bandits and multi-parameter processes},
  author={Mandelbaum, Avi},
  journal={Probability Theory and Related Fields},
  volume={71},
  number={1},
  pages={129--147},
  year={1986},
  publisher={Springer}
}

@article{kaspi1998multi,
  title={Multi-armed bandits in discrete and continuous time},
  author={Kaspi, Haya and Mandelbaum, Avishai},
  journal={The Annals of Applied Probability},
  volume={8},
  number={4},
  pages={1270--1290},
  year={1998},
  publisher={Institute of Mathematical Statistics}
}

@book{cover1999elements,
  title={Elements of information theory},
  author={Cover, Thomas M},
  year={1999},
  publisher={John Wiley \& Sons}
}

@article{mintz2020nonstationary,
  title={Nonstationary bandits with habituation and recovery dynamics},
  author={Mintz, Yonatan and Aswani, Anil and Kaminsky, Philip and Flowers, Elena and Fukuoka, Yoshimi},
  journal={Operations Research},
  volume={68},
  number={5},
  pages={1493--1516},
  year={2020},
  publisher={INFORMS}
}

@article{chen2024non,
  title={Non-stationary bandits with auto-regressive temporal dependency},
  author={Chen, Qinyi and Golrezaei, Negin and Bouneffouf, Djallel},
  journal={Advances in Neural Information Processing Systems},
  volume={36},
  year={2024}
}

@article{bacchiocchi2022autoregressive,
  title={Autoregressive bandits},
  author={Bacchiocchi, Francesco and Genalti, Gianmarco and Maran, Davide and Mussi, Marco and Restelli, Marcello and Gatti, Nicola and Metelli, Alberto Maria},
  journal={arXiv preprint arXiv:2212.06251},
  year={2022}
}

@article{abbasi2022new,
  title={A new look at dynamic regret for non-stationary stochastic bandits},
  author={Abbasi-Yadkori, Yasin and Gyorgy, Andras and Lazic, Nevena},
  journal={arXiv preprint arXiv:2201.06532},
  year={2022}
}

@article{ghatak2021kolmogorov,
  title={Kolmogorov--{S}mirnov Test-Based Actively-Adaptive Thompson Sampling for Non-Stationary Bandits},
  author={Ghatak, Gourab and Mohanty, Hardhik and Rahman, Aniq Ur},
  journal={IEEE Transactions on Artificial Intelligence},
  volume={3},
  number={1},
  pages={11--19},
  year={2021},
  publisher={IEEE}
}

@inproceedings{bogunovic2016time,
  title={Time-varying {G}aussian process bandit optimization},
  author={Bogunovic, Ilija and Scarlett, Jonathan and Cevher, Volkan},
  booktitle={Artificial Intelligence and Statistics},
  pages={314--323},
  year={2016},
  organization={PMLR}
}

@article{russac2020algorithms,
  title={Algorithms for non-stationary generalized linear bandits},
  author={Russac, Yoan and Capp{\'e}, Olivier and Garivier, Aur{\'e}lien},
  journal={arXiv preprint arXiv:2003.10113},
  year={2020}
}

@article{liu2023definition,
  title={A definition of non-stationary bandits},
  author={Liu, Yueyang and Kuang, Xu and Van Roy, Benjamin},
  journal={arXiv preprint arXiv:2302.12202},
  year={2023}
}

@inproceedings{mellor2013thompson,
  title={Thompson sampling in switching environments with Bayesian online change detection},
  author={Mellor, Joseph and Shapiro, Jonathan},
  booktitle={Artificial Intelligence and Statistics},
  pages={442--450},
  year={2013},
  organization={PMLR}
}

@ARTICLE{gaussianar1,
    author={Julia Kuhn and Michel Mandjes and Yoni Nazarathy},
    title={Exploration vs Exploitation with Partially Observable Gaussian Autoregressive Arms},
    journal={EAI Endorsed Transactions on Self-Adaptive Systems},
    volume={1},
    number={4},
    publisher={EAI},
    journal_a={SAS},
    year={2015},
    month={2},
    doi={10.4108/icst.valuetools.2014.258207}
}

@article{kuhn2015wireless,
  title={Wireless channel selection with reward-observing restless multi-armed bandits},
  author={Kuhn, Julia and Nazarathy, Yoni},
  journal={Chapter to appear in “Markov Decision Processes in Practice”, Editors: R. Boucherie and N. van Dijk},
  year={2015},
  publisher={Citeseer}
}

@article{brown2010information,
  title={Information relaxations and duality in stochastic dynamic programs},
  author={Brown, David B and Smith, James E and Sun, Peng},
  journal={Operations research},
  volume={58},
  number={4-part-1},
  pages={785--801},
  year={2010},
  publisher={INFORMS}
}

@article{dong2018information,
  title={An information-theoretic analysis for thompson sampling with many actions},
  author={Dong, Shi and Van Roy, Benjamin},
  journal={Advances in Neural Information Processing Systems},
  volume={31},
  year={2018}
}

@article{hao2022contextual,
  title={Contextual Information-Directed Sampling},
  author={Hao, Botao and Lattimore, Tor and Qin, Chao},
  journal={arXiv preprint arXiv:2205.10895},
  year={2022}
}

@article{russo2018learning,
  title={Learning to optimize via information-directed sampling},
  author={Russo, Daniel and Van Roy, Benjamin},
  journal={Operations Research},
  volume={66},
  number={1},
  pages={230--252},
  year={2018},
  publisher={INFORMS}
}

@article{mesbah2018stochastic,
  title={Stochastic model predictive control with active uncertainty learning: A survey on dual control},
  author={Mesbah, Ali},
  journal={Annual Reviews in Control},
  volume={45},
  pages={107--117},
  year={2018},
  publisher={Elsevier}
}

@article{neu2022lifting,
  title={Lifting the Information Ratio: An Information-Theoretic Analysis of Thompson Sampling for Contextual Bandits},
  author={Neu, Gergely and Olkhovskaya, Julia and Papini, Matteo and Schwartz, Ludovic},
  journal={arXiv preprint arXiv:2205.13924},
  year={2022}
}

@article{russo2022satisficing,
  title={Satisficing in time-sensitive bandit learning},
  author={Russo, Daniel and Van Roy, Benjamin},
  journal={Mathematics of Operations Research},
  year={2022},
  publisher={INFORMS}
}

@article{freund2019good,
  title={Good prophets know when the end is near},
  author={Freund, Daniel and Banerjee, Siddhartha},
  journal={Available at SSRN 3479189},
  year={2019}
}

@article{spencer2014queuing,
  title={Queuing with future information},
  author={Spencer, Joel and Sudan, Madhu and Xu, Kuang},
  journal={Annals of Applied Probability},
  volume={24},
  number={5},
  pages={2091--2142},
  year={2014},
  publisher={Institute of Mathematical Statistics}
}

@article{xu2016using,
  title={Using future information to reduce waiting times in the emergency department via diversion},
  author={Xu, Kuang and Chan, Carri W},
  journal={Manufacturing \& Service Operations Management},
  volume={18},
  number={3},
  pages={314--331},
  year={2016},
  publisher={INFORMS}
}

@inproceedings{arumugam2021deciding,
  title={Deciding what to learn: A rate-distortion approach},
  author={Arumugam, Dilip and Van Roy, Benjamin},
  booktitle={International Conference on Machine Learning},
  pages={373--382},
  year={2021},
  organization={PMLR}
}

@article{arumugam2021value,
  title={The value of information when deciding what to learn},
  author={Arumugam, Dilip and Van Roy, Benjamin},
  journal={Advances in Neural Information Processing Systems},
  volume={34},
  pages={9816--9827},
  year={2021}
}

@article{ferreira2018online,
  title={Online network revenue management using {T}hompson sampling},
  author={Ferreira, Kris Johnson and Simchi-Levi, David and Wang, He},
  journal={Operations research},
  volume={66},
  number={6},
  pages={1586--1602},
  year={2018},
  publisher={INFORMS}
}

@article{schwartz2017customer,
  title={Customer acquisition via display advertising using multi-armed bandit experiments},
  author={Schwartz, Eric M and Bradlow, Eric T and Fader, Peter S},
  journal={Marketing Science},
  volume={36},
  number={4},
  pages={500--522},
  year={2017},
  publisher={INFORMS}
}

@inproceedings{hill2017efficient,
  title={An efficient bandit algorithm for realtime multivariate optimization},
  author={Hill, Daniel N and Nassif, Houssam and Liu, Yi and Iyer, Anand and Vishwanathan, SVN},
  booktitle={Proceedings of the 23rd ACM SIGKDD International Conference on Knowledge Discovery and Data Mining},
  pages={1813--1821},
  year={2017}
}

@article{bai2013bayesian,
  title={Bayesian mixture modelling and inference based {T}hompson sampling in Monte-Carlo tree search},
  author={Bai, Aijun and Wu, Feng and Chen, Xiaoping},
  journal={Advances in neural information processing systems},
  volume={26},
  year={2013}
}

@inproceedings{graepel2010web,
  title={Web-scale {B}ayesian click-through rate prediction for sponsored search advertising in microsoft's bing search engine},
  author={Graepel, Thore and Candela, Joaquin Quinonero and Borchert, Thomas and Herbrich, Ralf},
  year={2010},
  organization={Omnipress}
}

@inproceedings{agarwal2013computational,
  title={Computational advertising: the {L}inked{I}n way},
  author={Agarwal, Deepak},
  booktitle={Proceedings of the 22nd ACM international conference on Information \& Knowledge Management},
  pages={1585--1586},
  year={2013}
}

@inproceedings{agarwal2014laser,
  title={Laser: A scalable response prediction platform for online advertising},
  author={Agarwal, Deepak and Long, Bo and Traupman, Jonathan and Xin, Doris and Zhang, Liang},
  booktitle={Proceedings of the 7th ACM international conference on Web search and data mining},
  pages={173--182},
  year={2014}
}

@article{kawale2015efficient,
  title={Efficient {T}hompson Sampling for Online Matrix-Factorization Recommendation},
  author={Kawale, Jaya and Bui, Hung H and Kveton, Branislav and Tran-Thanh, Long and Chawla, Sanjay},
  journal={Advances in neural information processing systems},
  volume={28},
  year={2015}
}

@inproceedings{turner2009adaptive,
  title={Adaptive sequential {B}ayesian change point detection},
  author={Turner, Ryan and Saatci, Yunus and Rasmussen, Carl Edward},
  booktitle={Temporal Segmentation Workshop at NIPS},
  pages={1--4},
  year={2009},
  organization={Citeseer}
}

@article{wilson2010bayesian,
  title={Bayesian online learning of the hazard rate in change-point problems},
  author={Wilson, Robert C and Nassar, Matthew R and Gold, Joshua I},
  journal={Neural computation},
  volume={22},
  number={9},
  pages={2452--2476},
  year={2010},
  publisher={MIT Press One Rogers Street, Cambridge, MA 02142-1209, USA journals-info~…}
}

@inproceedings{wei2018abruptly,
  title={On abruptly-changing and slowly-varying multiarmed bandit problems},
  author={Wei, Lai and Srivatsva, Vaibhav},
  booktitle={2018 Annual American Control Conference (ACC)},
  pages={6291--6296},
  year={2018},
  organization={IEEE}
}

@article{auer2002nonstochastic,
  title={The nonstochastic multiarmed bandit problem},
  author={Auer, Peter and Cesa-Bianchi, Nicolo and Freund, Yoav and Schapire, Robert E},
  journal={SIAM journal on computing},
  volume={32},
  number={1},
  pages={48--77},
  year={2002},
  publisher={SIAM}
}

@article{hartland2006multi,
  title={Multi-armed bandit, dynamic environments and meta-bandits},
  author={Hartland, C{\'e}dric and Gelly, Sylvain and Baskiotis, Nicolas and Teytaud, Olivier and Sebag, Michele},
  year={2006}
}

@inproceedings{viappiani2013thompson,
  title={Thompson sampling for {B}ayesian bandits with resets},
  author={Viappiani, Paolo},
  booktitle={International Conference on Algorithmic Decision Theory},
  pages={399--410},
  year={2013},
  organization={Springer}
}

@ARTICLE{9194367,
  author={Ghatak, Gourab},
  journal={IEEE Transactions on Computers}, 
  title={A change-detection-based {T}hompson sampling framework for non-Stationary bandits}, 
  year={2021},
  volume={70},
  number={10},
  pages={1670-1676},
  doi={10.1109/TC.2020.3022634}}

@article{freund1997decision,
  title={A decision-theoretic generalization of on-line learning and an application to boosting},
  author={Freund, Yoav and Schapire, Robert E},
  journal={Journal of computer and system sciences},
  volume={55},
  number={1},
  pages={119--139},
  year={1997},
  publisher={Elsevier}
}

@inproceedings{kocsis2006discounted,
  title={Discounted {UCB}},
  author={Kocsis, Levente and Szepesv{\'a}ri, Csaba},
  booktitle={2nd PASCAL Challenges Workshop},
  volume={2},
  pages={51--134},
  year={2006}
}

@inproceedings{bubeck2015bandit,
  title={Bandit convex optimization:$\sqrt{T}$ regret in one dimension},
  author={Bubeck, S{\'e}bastien and Dekel, Ofer and Koren, Tomer and Peres, Yuval},
  booktitle={Conference on Learning Theory},
  pages={266--278},
  year={2015},
  organization={PMLR}
}

@article{besson2019generalized,
  title={The generalized likelihood ratio test meets {KLUCB}: an improved algorithm for piece-wise non-stationary bandits},
  author={Besson, Lilian and Kaufmann, Emilie},
  journal={Proceedings of Machine Learning Research vol XX},
  volume={1},
  pages={35},
  year={2019}
}

@article{lai1985asymptotically,
  title={Asymptotically efficient adaptive allocation rules},
  author={Lai, TL and Robbins, Herbert},
  journal={Advances in Applied Mathematics},
  volume={6},
  number={1},
  pages={4--22},
  year={1985},
  publisher={Elsevier}
}

@inproceedings{slivkins2008adapting,
  title={Adapting to a Changing Environment: the {B}rownian Restless Bandits.},
  author={Slivkins, Aleksandrs and Upfal, Eli},
  booktitle={COLT},
  pages={343--354},
  year={2008}
}

@article{trovo2020sliding,
  title={Sliding-window {T}hompson sampling for non-stationary settings},
  author={Trovo, Francesco and Paladino, Stefano and Restelli, Marcello and Gatti, Nicola},
  journal={Journal of Artificial Intelligence Research},
  volume={68},
  pages={311--364},
  year={2020}
}

@article{laplace1986memoir,
  title={Memoir on the probability of the causes of events},
  author={Laplace, Pierre Simon},
  journal={Statistical science},
  volume={1},
  number={3},
  pages={364--378},
  year={1986},
  publisher={JSTOR}
}

@article{liu2022gaussian,
  title={Gaussian Imagination in Bandit Learning},
  author={Liu, Yueyang and Devraj, Adithya M and Van Roy, Benjamin and Xu, Kuang},
  journal={arXiv preprint arXiv:2201.01902},
  year={2022}
}

@inproceedings{gupta2011thompson,
  title={Thompson sampling for dynamic multi-armed bandits},
  author={Gupta, Neha and Granmo, Ole-Christoffer and Agrawala, Ashok},
  booktitle={2011 10th International Conference on Machine Learning and Applications and Workshops},
  volume={1},
  pages={484--489},
  year={2011},
  organization={IEEE}
}

@inproceedings{lattimore2019information,
  title={An Information-Theoretic Approach to Minimax Regret in Partial Monitoring},
  author={Lattimore, Tor and Szepesv{\'a}ri, Csaba},
  booktitle={Conference on Learning Theory},
  pages={2111--2139},
  year={2019},
  organization={PMLR}
}

@InProceedings{pmlr-v99-auer19a,
  title = 	 {Adaptively Tracking the Best Bandit Arm with an Unknown Number of Distribution Changes},
  author =       {Auer, Peter and Gajane, Pratik and Ortner, Ronald},
  booktitle = 	 {Proceedings of the Thirty-Second Conference on Learning Theory},
  pages = 	 {138--158},
  year = 	 {2019},
  editor = 	 {Beygelzimer, Alina and Hsu, Daniel},
  volume = 	 {99},
  series = 	 {Proceedings of Machine Learning Research},
  month = 	 {25--28 Jun},
  publisher =    {PMLR}
}

@inproceedings{cheung2019learning,
  title={Learning to optimize under non-stationarity},
  author={Cheung, Wang Chi and Simchi-Levi, David and Zhu, Ruihao},
  booktitle={The 22nd International Conference on Artificial Intelligence and Statistics},
  pages={1079--1087},
  year={2019},
  organization={PMLR}
}

@inproceedings{NIPS2011_e53a0a29,
 author = {Chapelle, Olivier and Li, Lihong},
 booktitle = {Advances in Neural Information Processing Systems},
 editor = {J. Shawe-Taylor and R. Zemel and P. Bartlett and F. Pereira and K.Q. Weinberger},
 pages = {},
 title = {An Empirical Evaluation of {T}hompson Sampling},
 volume = {24},
 year = {2011}
}

@InProceedings{pmlr-v23-agrawal12,
  title = 	 {Analysis of {T}hompson Sampling for the Multi-armed Bandit Problem},
  author = 	 {Agrawal, Shipra and Goyal, Navin},
  booktitle = 	 {Proceedings of the 25th Annual Conference on Learning Theory},
  pages = 	 {39.1--39.26},
  year = 	 {2012},
  editor = 	 {Mannor, Shie and Srebro, Nathan and Williamson, Robert C.},
  volume = 	 {23},
  series = 	 {Proceedings of Machine Learning Research},
  address = 	 {Edinburgh, Scotland},
  month = 	 {25--27 Jun},
  publisher =    {PMLR}
}

@article{RussoMOR2014,
  author    = {Daniel Russo and
               Van Roy, Benjamin},
  title     = {Learning to Optimize Via Posterior Sampling},
  journal   = {Mathematics of Operations Research},
  volume    = {39},
  number    = {4},
  year      = {2014},
  pages.    = {1221-1243},
}

@InProceedings{pmlr-v162-kim22j,
  title = 	 {Rotting Infinitely Many-Armed Bandits},
  author =       {Kim, Jung-Hun and Vojnovic, Milan and Yun, Se-Young},
  booktitle = 	 {Proceedings of the 39th International Conference on Machine Learning},
  pages = 	 {11229--11254},
  year = 	 {2022},
  editor = 	 {Chaudhuri, Kamalika and Jegelka, Stefanie and Song, Le and Szepesvari, Csaba and Niu, Gang and Sabato, Sivan},
  volume = 	 {162},
  series = 	 {Proceedings of Machine Learning Research},
  month = 	 {17--23 Jul},
  publisher =    {PMLR}
}

@article{JMLR:v17:14-087,
  author  = {Daniel Russo and Van Roy, Benjamin},
  title   = {An Information-Theoretic Analysis of {T}hompson Sampling},
  journal = {Journal of Machine Learning Research},
  year    = {2016},
  volume  = {17},
  number  = {68},
  pages   = {1-30},
}

@book{gray2011entropy,
  title={Entropy and information theory},
  author={Gray, Robert M},
  year={2011},
  publisher={Springer Science \& Business Media}
}

@InProceedings{pmlr-v31-mellor13a,
  title = 	 {Thompson Sampling in Switching Environments with {B}ayesian Online Change Detection},
  author = 	 {Mellor, Joseph and Shapiro, Jonathan},
  booktitle = 	 {Proceedings of the Sixteenth International Conference on Artificial Intelligence and Statistics},
  pages = 	 {442--450},
  year = 	 {2013},
  editor = 	 {Carvalho, Carlos M. and Ravikumar, Pradeep},
  volume = 	 {31},
  series = 	 {Proceedings of Machine Learning Research},
  address = 	 {Scottsdale, Arizona, USA},
  month = 	 {29 Apr--01 May},
  publisher =    {PMLR}
}

@article{besbes2019optimal,
  title={Optimal Exploration-Exploitation in a Multi-Armed-Bandit Problem with Non-Stationary Rewards},
  author={Besbes, Omar and Gur, Yonatan and Zeevi, Assaf},
  journal={Stochastic Systems},
  volume={9},
  number={4},
  pages={319--337},
  year={2019}
}

@article{Thompson1933,
	Author = {Thompson, William R},
	Journal = {Biometrika},
	Number = {3/4},
	Pages = {285--294},
	Publisher = {JSTOR},
	Title = {On the likelihood that one unknown probability exceeds another in view of the evidence of two samples},
	Volume = {25},
	Year = {1933}}

@article{lu2021reinforcement,
  title={Reinforcement Learning, Bit by Bit},
  author={Lu, Xiuyuan and Van Roy, Benjamin and Dwaracherla, Vikranth and Ibrahimi, Morteza and Osband, Ian and Wen, Zheng},
  journal={arXiv preprint arXiv:2103.04047},
  year={2021}
}

@article{wen2022predictions2decisions,
  title={From Predictions to Decisions: The Importance of Joint Predictive Distributions},
  author={Wen, Zheng and Osband, Ian Osband and Qin, Chao and Lu, Xiuyuan and Ibrahimi, Morteza and Dwaracherla, Vikranth and Asghari, Mohammad and Van Roy, Benjamin},
  journal={arXiv preprint arXiv:2107.09224},
  year={2022}
}

@article{garivier2008upper,
  title={On upper-confidence bound policies for non-stationary bandit problems},
  author={Garivier, Aur{\'e}lien and Moulines, Eric},
  journal={arXiv preprint arXiv:0805.3415},
  year={2008}
}

@article{raj2017taming,
  title={Taming non-stationary bandits: A {B}ayesian approach},
  author={Raj, Vishnu and Kalyani, Sheetal},
  journal={arXiv preprint arXiv:1707.09727},
  year={2017}
}

@inproceedings{zhao2020simple,
  title={A simple approach for non-stationary linear bandits},
  author={Zhao, Peng and Zhang, Lijun and Jiang, Yuan and Zhou, Zhi-Hua},
  booktitle={International Conference on Artificial Intelligence and Statistics},
  pages={746--755},
  year={2020},
  organization={PMLR}
}

@inproceedings{luo2018efficient,
  title={Efficient contextual bandits in non-stationary worlds},
  author={Luo, Haipeng and Wei, Chen-Yu and Agarwal, Alekh and Langford, John},
  booktitle={Conference On Learning Theory},
  pages={1739--1776},
  year={2018},
  organization={PMLR}
}

@article{min2019thompson,
  title={Thompson sampling with information relaxation penalties},
  author={Min, Seungki and Maglaras, Costis and Moallemi, Ciamac C},
  journal={Advances in Neural Information Processing Systems},
  volume={32},
  year={2019}
}
